\documentclass[times, 3p, 10pt]{elsarticle}
\journal{Information Fusion}
\usepackage{rotating}
\usepackage{amsmath}
\usepackage{bm}
\providecommand{\mathbold}[1]{\bm{#1}}
\newcommand{\mathbbm}[1]{\mathbb{#1}}
\usepackage[utf8]{inputenc}
\usepackage[T1]{fontenc}
\usepackage{mathtools}
\newcommand{\pointwise}{\circ}
\usepackage{upgreek}
\usepackage{multirow}
\usepackage{booktabs}
\usepackage[makeroom]{cancel}
\usepackage{amsthm}
\newtheorem{theorem}{Theorem}
\newtheorem{corollary}{Corollary}
\newtheorem{proposition}{Proposition}
\newtheorem{remark}{Remark}
\newtheorem{example}{Example}
\usepackage{url}
\usepackage{amssymb}
\usepackage{subcaption}

\usepackage{longtable}
\usepackage{pdflscape}
\newcommand{\novel}{\ensuremath{{}^{\flat}}}
\newcommand{\sinv}{\ensuremath{\checkmark}}
\newcommand{\SI}{\ensuremath{\Upsigma_{\mathbf{1}}}}
\newcommand{\Sm}{\ensuremath{\Upsigma_{\mathbf{m}}}}
\newcommand{\Sn}{\ensuremath{\Upsigma_{\mathbf{n}}}}

\newtheorem{mydef}{Definition}
\begin{document}

\begin{frontmatter}

\title{A Unified Algebraic Framework for Classification Performance Evaluation}

\author[ufabc]{Ronaldo C. Prati\corref{cor1}}
\ead{ronaldo.prati@ufabc.edu.br}
\cortext[cor1]{Corresponding author. ORCID: 0000-0001-8597-4987.}
\address[ufabc]{Universidade Federal do ABC (UFABC), Av.\ dos Estados, 5001,
Santo Andr\'{e}, SP, 09210-580, Brazil}

\begin{abstract}
We propose a unified algebraic framework for classification performance evaluation covering binary, multiclass, multilabel, ordinal, hierarchical, cost-sensitive and soft-label settings. Actual and predicted labels are represented as binary indicator matrices, and three aggregation operators (global, column-wise, and row-wise) correspond exactly to micro, macro/weighted, and exemplar averaging. Any binary measure expressed in terms of the four confusion-matrix counts extends to all these settings by substituting an operator, with no measure-specific derivation. We show that the structural properties governing which identities hold for an extension are themselves derivable from the binary formula. Micro-averaging equals denominator-weighted macro-averaging precisely for the class of \emph{aggregation-decomposable} (linear-fractional) measures, which we characterise algebraically and show by counterexample to be a strict class. For soft ground truth, we prove from the $t$-norm axioms alone that the product $t$-norm is the unique choice whose confusion counts preserve marginal memberships. In multiclass settings, micro-precision, micro-recall and micro-$F_1$ collapse identically onto accuracy. Furthermore, binary skew-invariance transfers unconditionally to multilabel aggregation, but only partially to multiclass problems. For measures with a linear numerator and a prediction-independent denominator, the optimal decision threshold is the share of the numerator's total weight that favours a negative prediction. Consequently, standard training targets the measure exactly when positive and negative predictions carry equal weight; outside this class (including precision and the F-measures), no such closed form exists. Under one-hot multiclass encoding, a measure fails to attain its theoretical minimum whenever its value with zero true positives still depends on true negatives, establishing non-trivial performance floors even for completely incorrect classifiers; for example, a classifier wrong on every example still achieves a label accuracy of $0.8$ on a ten-class problem.
\end{abstract}


\end{frontmatter}

\section{Introduction}

Evaluating the performance of classification algorithms is an essential step for machine learning researchers and practitioners~\cite{DBLP:books/cu/Japkowicz2011,DBLP:journals/tkde/PratiBM11}. The process of comparing, selecting and tuning the parameters of learning algorithms is generally based on the performance evaluation by means of quality measures of the learned models, evaluated on test data. Despite its universality, evaluating classification models is a complex and still open-ended challenge~\cite{DBLP:journals/jmlr/Hernandez-OralloFF12,DBLP:journals/isci/Carbonero-RuzMF17}.

Over recent decades, numerous performance measures have been proposed to evaluate classification models~\cite{DBLP:journals/ipm/SokolovaL09,DBLP:journals/prl/FerriHM09,DBLP:journals/pr/CondessaBK17,DBLP:journals/pr/KautzEP17}. The choice of a particular measure generally depends on the problem at hand, as well as the preferences of a particular domain. In medicine, sensitivity and specificity are often used; in information retrieval, precision and recall are preferred; in cost-sensitive applications, misclassification costs must be incorporated; and in settings with uncertain or crowdsourced labels, soft evaluation criteria are required.

The complexity of evaluation grows substantially as the classification setting moves beyond the standard binary case. In multiclass problems, multiple classes compete and errors between different pairs of classes need not be treated equally. In multilabel settings, each example may simultaneously belong to several classes, creating interactions between labels that scalar measures must somehow summarize. In ordinal and hierarchical classification, the structure of the class space carries information about error severity. In single-class, open-set, and abstaining classifiers, the boundary between known and unknown is itself part of the prediction. Across all these settings, the same underlying measure (such as $F_1$) can be computed in multiple ways (micro, macro, weighted, exemplar), each emphasising different aspects of performance and potentially producing contradictory rankings of competing classifiers.

A further complication is that evaluation methodology has not kept pace with the diversity of modern classification architectures. Binary measures are well-studied, but their extensions to non-standard settings are typically derived on a case-by-case basis, leading to a fragmented landscape in which the relationships between measures, and the conditions under which they agree or disagree, are poorly understood.

In this paper, we address this fragmentation by proposing a unified algebraic framework for classification performance evaluation. The foundation is a representation of actual and predicted labels as binary indicator matrices, combined with three aggregation operators (global $\Upsigma_{\mathrm{1}}$, column-wise $\Upsigma_{\mathrm{m}}$, and row-wise $\Upsigma_{\mathrm{n}}$) corresponding to micro, macro/weighted, and exemplar averaging. Any binary measure expressed in terms of $tp$, $fp$, $fn$, $tn$ extends automatically to multiclass and multilabel problems by substituting these operators. Beyond this core generative role, the framework provides a graded treatment across non-standard settings: it derives closed-form identities for soft ground truth (via $t$-norms), ordinal classification (via membership encodings), and cost sensitivity (via bilinear forms); accommodates structured scenarios such as hierarchical, open-set, and abstaining classification under a uniform notation; and explicitly delimits boundary settings like label ranking.

Substituting operators in this way reveals structural relationships across evaluation schemes, but these properties do not hold unconditionally; our main theoretical results delineate the precise scope of each. First, micro-averaging equals denominator-weighted macro-averaging precisely for the class of \emph{aggregation-decomposable} measures, with their divergence governed by a covariance term. Second, binary skew-invariance transfers unconditionally to multilabel evaluation, but only approximately to multiclass settings due to one-vs-rest negative mixtures. Third, a closed-form Bayes-optimal decision threshold exists precisely when the measure has a linear numerator and a prediction-independent denominator. Fourth, micro-precision, micro-recall, and micro-$F_1$ collapse identically onto accuracy in multiclass problems. Finally, lower boundary correctness fails under one-hot multiclass encoding whenever a measure with zero true positives retains a dependence on true negatives, establishing non-trivial performance floors even for completely incorrect classifiers.

We evaluate these theoretical predictions through synthetic simulations and empirical benchmarks spanning multiclass, multilabel, and medical screening domains. These experiments illustrate how aggregation choices and edge-case conventions induce rank reversals in practice, validate the role of probability calibration in finite-sample Bayes thresholding, and demonstrate that empirical observations match all derived algebraic identities and performance floors to numerical precision.

The remainder of the paper is organised as follows. Section~\ref{sec:related} reviews related work. Section~\ref{sec:binary} fixes notation and the four counts in the binary case. Section~\ref{sec:matrix} introduces the indicator-matrix representation and the three aggregation operators, and Section~\ref{sec:extensions} uses them to extend binary measures to multiclass and multilabel problems. Section~\ref{sec:soft} treats the remaining settings (soft and uncertain ground truth, ordinal and hierarchical structure, cost sensitivity, abstention, open-set and single-class recognition, label ranking, and partial labels) as special cases of the same construction. Section~\ref{sec:theory} contains the theoretical results, Section~\ref{sec:experiments} the empirical illustrations, and Section~\ref{sec:conclusion} concludes.

\section{Related Work}
\label{sec:related}

Research on classification performance evaluation can be grouped into five broad streams: taxonomies of binary measures, axiomatic properties and measure selection criteria, multiclass and multilabel extensions, unified decision-theoretic and geometric frameworks, and non-standard settings such as cost sensitivity or label uncertainty. We position our contribution with respect to each.

The first stream consists of systematic surveys of existing measures and their properties. Sokolova and Lapalme~\cite{DBLP:journals/ipm/SokolovaL09} analyse twelve binary measures across five properties (consistency, correctness, completeness, reliability, and expected behaviour) and extend a subset to multiclass and multilabel settings. Ferri et al.~\cite{DBLP:journals/prl/FerriHM09} conduct an experimental comparison of measures on synthetic and real data, showing that rankings of classifiers are highly sensitive to measure choice. Prati et al.~\cite{DBLP:journals/tkde/PratiBM11} survey graphical methods for performance evaluation and their relationship to scalar measures. While these studies provide valuable taxonomies, they analyse measures individually and do not provide a generative mechanism for producing multiclass or multilabel versions automatically. In particular, the property framework of Sokolova and Lapalme must be applied separately to each measure and each extension. In contrast, the analysis developed in Section~\ref{sec:properties} of this paper (Propositions~\ref{prop:preserved} and~\ref{prop:lower_boundary}) settles two of their key properties (boundary correctness and monotonicity) for all three aggregation operators simultaneously, covering every derived measure in Table~\ref{tab:extensions} at once. In doing so, we show that lower boundary correctness fails under one-hot multiclass encoding for measures whose value at $tp = 0$ still depends on the true-negative count.

A second line of work examines axiomatic properties and measure selection, asking not what properties a measure has in isolation but how measures rank confusion matrices, and therefore which measure one should choose for a given problem. Parker~\cite{parker2011analysis} analyses binary performance measures by the orderings they induce over classifiers. Luque et al.~\cite{luque2019impact} quantify how class imbalance distorts measures computed from the binary confusion matrix, and separately characterise their symmetry properties. Brzezinski et al.~\cite{brzezinski2018visual} develop a visual methodology for comparing measures over the space of confusion matrices under imbalance. G{\"o}sgens et al.~\cite{gosgens2021good} take an axiomatic approach, proposing properties a good classification measure should satisfy and deriving which measures satisfy them; their \emph{constant baseline} property, which requires random assignments to receive the same expected score regardless of how many objects are assigned to each class, is closely related to the notion of skew-invariance we formalise in Section~\ref{sec:skew}. Canbek et al.~\cite{canbek2021benchmetrics} provide a systematic benchmarking methodology for binary measures and a comprehensive catalogue and knowledge representation of them~\cite{canbek2022ptopi}. Shirdel et al.~\cite{shirdel2024worthiness} introduce a ``worthiness benchmark'' that evaluates a metric by how well its ordering of confusion matrices matches a reference ordering.

This literature complements our framework through a clear distinction in focus. While measure-selection studies ask which metric to prefer for a given task, our framework takes a measure as given and investigates how it extends across settings, which identities emerge, and which properties are preserved. The two approaches inform each other: our aggregation operators systematically extend binary measures to broader settings where properties may silently fail, while our theoretical results prevent selection criteria from treating functionally dependent variants as independent candidates as precluded by Theorem~\ref{thm:micro_macro} (Section~\ref{sec:decomposable}), or overlooking the collapse of micro-measures onto accuracy established in Theorem~\ref{thm:micro_collapse} (Section~\ref{sec:redundancy}). Rather than recommending a specific measure, our contribution is structural: formalising how performance measures relate to one another across classification settings.

A third stream of research addresses the extension of binary measures to multiclass and multilabel domains. The idea of decomposing multiclass evaluation into per-class binary problems dates at least to Sebastiani~\cite{sebastiani2002machine} in the context of text categorisation, where per-class $tp_j$, $fp_j$, $fn_j$, $tn_j$ counts and their micro and macro averages are defined. Tsoumakas and Katakis~\cite{DBLP:journals/jdwm/TsoumakasK07} and Zhang and Zhou~\cite{zhang2014review} survey multilabel classification and catalogue example-based and label-based measures, including Hamming loss and exemplar precision and recall. Godbole and Sarawagi~\cite{godbole2004discriminative} extend precision, recall and $F_1$ to multilabel problems using set-intersection semantics. Our framework subsumes all of these extensions: as shown in Table~\ref{tab:extensions}, they correspond to specific choices of aggregation operator and measure formula. Moreover, our framework does so for any binary measure simultaneously, rather than for specific measures individually.

A fourth area of study pursues unified views of performance evaluation. The closest in spirit to our approach is that of Hern\'andez-Orallo et al.~\cite{DBLP:journals/jmlr/Hernandez-OralloFF12}, who propose a unified view of performance metrics as expected classification losses under threshold variation, covering binary and cost-sensitive settings via cost curves and the ROC framework. Their approach focuses on continuous score-based classifiers and the relationship between thresholds and loss functions, whereas ours operates at the level of the confusion matrix and provides an explicit algebraic mechanism for generating multiclass and multilabel measures. The two frameworks are therefore complementary: theirs gives a decision-theoretic unification for binary and cost-sensitive cases, ours gives an algebraic unification across all classification settings. Kautz et al.~\cite{DBLP:journals/pr/KautzEP17} propose a generic multiclass performance measure derived from geometrical considerations on the confusion matrix, while Carbonero-Ruz et al.~\cite{DBLP:journals/isci/Carbonero-RuzMF17} propose a two-dimensional accuracy-based measure. Both proposals target specific gaps in multiclass evaluation rather than a generative framework.

Finally, a fifth stream tackles non-standard evaluation settings, notably cost sensitivity and label uncertainty. Cost-sensitive evaluation has a long history~\cite{elkan2001foundations,drummond2006cost,DBLP:journals/jmlr/Hernandez-OralloFF12}, typically addressed by incorporating misclassification costs directly into the learning objective or by post-hoc reweighting of confusion matrix entries. Our framework formalises cost sensitivity via a cost matrix that acts as a bilinear form on the indicator matrices in Eq.~(\ref{eq:total_cost}), unifying misclassification cost, MAE and MSE as special cases via Theorem~\ref{thm:ordinal_cost} (Section~\ref{sec:decomposable}). Evaluation under label uncertainty (arising from annotator disagreement or probabilistic ground truth) has received increasing attention in crowdsourcing~\cite{DBLP:journals/tkde/PratiBM11} and learning from crowds, but prior studies typically aggregate labels into a single hard assignment before evaluation. Our soft-label extension instead propagates uncertainty directly through the confusion matrices via $t$-norms, making the treatment of label ambiguity an explicit, interpretable part of the evaluation protocol. To the best of our knowledge, no prior work connects these settings (multiclass, multilabel, ordinal, cost-sensitive, and uncertain-label) within a single algebraic framework.

To develop this unified algebraic framework across all settings, we adopt the following notational conventions throughout the paper. \emph{Scalars} are denoted by plain lowercase italic letters: $tp$, $fp$, $fn$, $tn$ are aggregate confusion-matrix counts, and $n$, $m$ are dimensions. \emph{Vectors} are denoted by bold lowercase letters ($\mathbold{y}$, $\mathbold{tp}$, etc.) and \emph{matrices} by bold uppercase letters ($\mathbold{Y}$, $\mathbold{TP}$, etc.). Bold uppercase matrices $\mathbold{TP}$, $\mathbold{FP}$, $\mathbold{FN}$, $\mathbold{TN}$ appear in text when referring to the matrices as objects; subscripted entries $\mathbold{TP}_{i,j}$ are used in proofs where a single element is meant.

More specifically: in the binary case (Section~\ref{sec:binary}), $\mathbold{tp} \in \{0,1\}^n$ is the element-wise true-positive indicator vector and $tp = \Upsigma_{\mathrm{1}}(\mathbold{tp}) \in \mathbb{Z}_{\geq 0}$ its scalar sum. In the multiclass and multilabel case, $\mathbold{TP} \in \{0,1\}^{n \times m}$ is the true-positive indicator matrix. Applying $\Upsigma_{\mathrm{m}}$ yields a per-class count vector $\mathbold{tp} \in \mathbb{Z}_{\geq 0}^m$; applying $\Upsigma_{\mathrm{n}}$ yields a per-example count vector $\mathbold{tp} \in \mathbb{Z}_{\geq 0}^n$; and applying $\Upsigma_{\mathrm{1}}$ yields the global scalar $tp \in \mathbb{Z}_{\geq 0}$. The type of $\mathbold{tp}$ (indicator vector, count vector, or scalar) is determined by which aggregation operator has been applied. The same conventions hold for $\mathbold{fp}$, $\mathbold{fn}$, $\mathbold{tn}$ and their uppercase counterparts. Actual and predicted label matrices are $\mathbold{Y}$ and $\mathbold{\hat{Y}}$; their binary-case vector counterparts are $\mathbold{y}$ and $\mathbold{\hat{y}}$.

\section{The binary case}
\label{sec:binary}

Consider a binary classification problem with a test set\footnote{We assume a single fixed test set throughout this section. The extension to resampling procedures such as cross-validation, where predictions are aggregated across multiple test folds, is formalised in Section~\ref{sec:cv}.} of $n$ examples, where each example is associated with one of two predefined, non-overlapping classes $c_1$ and $c_2$. Without loss of generality, these two classes are also called \textit{positive} and \textit{negative} classes, respectively, and denoted as $c_+$ and $c_-$. A binary classifier will assign a predicted class to each test example. Most of the performance evaluation measures are based on the cross-tabulation of actual and predicted class values (the Confusion Matrix), as shown in Table~\ref{tab:confusion_matrix}, where true positives/negatives ($tp/tn$) correspond to the counts of examples that belong to the positive/negative class and are correctly classified as positive/negative; false positives/negatives ($fp/fn$) correspond to the counts of examples that belong to the positive/negative class and are incorrectly classified as positive/negative; actual positives/negatives ($tpos/tneg$) correspond to the total numbers of examples of the positive/negative class; predicted positives/negatives ($ppos/pneg$) correspond to the total numbers of examples predicted as positive/negative; and the total number of examples ($n$) is the sample size.

\begin{table}[!t]
\centering
\small
\begin{tabular}{cc|cc|c}
\multicolumn{2}{c}{}
 & \multicolumn{2}{c}{Predicted} \\
 &  & \textit{positive} & \textit{negative}\\
\cline{2-5}
\multirow{3}{*}{\rotatebox[origin=c]{90}{Actual}}
 & \textit{positive} & $tp$ & $fn$ & $tpos$ \\
 & \textit{negative} & $fp$ & $tn$ & $tneg$ \\
\cline{2-5}
 &  & $ppos$ & $pneg$ & $n$
\end{tabular}
\caption{Binary confusion matrix.}
\label{tab:confusion_matrix}
\end{table}

Different performance measures are derived from the cells of the confusion matrix, each encoding a different evaluation priority. Table~\ref{tab:measures} lists the eleven measures used throughout the paper: some treat both classes symmetrically (accuracy, error rate), some focus on the positive class (precision, recall, $F_\beta$, Jaccard), some on the negative class (specificity), some measure balance across both classes (balanced accuracy, G-mean), and some capture global agreement (MCC~\cite{matthews1975comparison}, Cohen's~$\kappa$~\cite{cohen1960coefficient}).

\begin{table}[!htb]
\centering
\small
\setlength{\tabcolsep}{4pt}
\begin{tabular}{c|c|p{5.4cm}|c}
\hline
 \textbf{Measure} & \textbf{Formula} & \textbf{Interpretation} & \textbf{Dec.}\\
 \hline
 Accuracy & $\dfrac{tp+tn}{n}$ & Proportion of correctly classified examples & \checkmark\\
 Error Rate & $\dfrac{fp+fn}{n}$ & Proportion of incorrectly classified examples & \checkmark\\
 Precision & $\dfrac{tp}{tp+fp}$ & Proportion of predicted positives that are truly positive & \checkmark\\
 Recall & $\dfrac{tp}{tp+fn}$ & Proportion of actual positives correctly classified & \checkmark\\
 Specificity & $\dfrac{tn}{tn+fp}$ & Proportion of actual negatives correctly classified & \checkmark\\[8pt]
 $F_\beta$-measure ($\beta\!=\!1$: $F_1$) & $\dfrac{(1+\beta^2)\,tp}{(1+\beta^2)\,tp+\beta^2 fn+fp}$ & Weighted harmonic mean of precision and recall & \checkmark\\[2pt]
 Jaccard index & $\dfrac{tp}{tp+fp+fn}$ & Intersection over union of predicted and actual positives & \checkmark\\[2pt]
 Balanced accuracy & $\dfrac{1}{2}\!\left(\dfrac{tp}{tp+fn}+\dfrac{tn}{tn+fp}\right)$ & Average of recall and specificity & --\\[10pt]
 G-mean & $\sqrt{\dfrac{tp}{tp+fn}\cdot\dfrac{tn}{tn+fp}}$ & Geometric mean of recall and specificity & --\\[10pt]
 MCC~\cite{matthews1975comparison} & $\dfrac{tp \cdot tn - fp \cdot fn}{\sqrt{(tp+fp)(tp+fn)(tn+fp)(tn+fn)}}$ & Correlation between observed and predicted classifications & --\\[5pt]
 Cohen's $\kappa$~\cite{cohen1960coefficient} & $\dfrac{p_o - p_e}{1 - p_e}$, \; $p_o\!=\!\frac{tp+tn}{n}$, \; $p_e\!=\!\frac{tpos\cdot ppos + tneg\cdot pneg}{n^2}$ & Agreement corrected for chance & --\\[2pt]
 \hline
\end{tabular}
\caption{Binary performance measures used as the basis for the framework extensions in Table~\ref{tab:extensions}. The \textbf{Dec.} column indicates whether the measure is \emph{aggregation-decomposable} in the sense of Definition~\ref{def:decomposable}, i.e.\ a ratio of two homogeneous linear forms in the counts. All eleven admit the three extensions of Eqs.~(\ref{eq:micro_ext})--(\ref{eq:example_ext}), with the qualification that Cohen's $\kappa$ requires auxiliary marginals because $p_e$ does not factor through the four counts (Section~\ref{sec:notable}); decomposability is required only for the identities relating those extensions to one another (established in Section~\ref{sec:theory}: Theorem~\ref{thm:micro_macro} and Corollaries~\ref{cor:micro_macro_balance}, \ref{cor:cost_micro_macro} and~\ref{cor:gap}).}
\label{tab:measures}
\end{table}

Although some of these measures, such as accuracy or error rate, extend naturally to multiclass, multilabel, ordinal, and other settings, others focus on a single class and require an averaging scheme to be applied. The wide variety of possible extensions and aggregation choices motivates the unified framework developed in the remainder of this paper.

To provide a unified view across all classification settings, we formulate these measures in terms of binary indicator vectors. Let $\mathbold{y}$ be a binary vector where, for each example $e_i$, $y_i = \mathbbm{I}(ActualClass(e_i) = c_+)$ for $i \in [1,n]$, and $\mathbbm{I}(\cdot)$ is an indicator function whose value is one if its argument is true and zero otherwise. Similarly, let $\mathbold{\hat{y}}$ be a binary indicator vector of size $n$ defined as $\hat{y}_i = \mathbbm{I}(PredictedClass(e_i) = c_+)$.

\begin{mydef}
\label{def:vectors}
Based on those two vectors, we can define the indicator vectors $\mathbold{tp}$, $\mathbold{fp}$, $\mathbold{fn}$ and $\mathbold{tn}$ as follows:
\begin{align*}
\mathbold{tp} & = \mathbold{y}\wedge\mathbold{\hat y}\\
\mathbold{fp} & = \neg\mathbold{y}\wedge\mathbold{\hat y}\\
\mathbold{fn} & = \mathbold{y}\wedge\neg\mathbold{\hat y}\\
\mathbold{tn} & = \neg\mathbold{y}\wedge\neg\mathbold{\hat y}
\end{align*}
\noindent where $\wedge$ and $\neg$ are element-wise Boolean operators \textsc{and} and \textsc{not}, respectively.
\end{mydef}

Each vector contains an entry of one precisely where the corresponding outcome occurs for example $e_i$. The scalar confusion-matrix counts are obtained by applying the global summation operator $\Upsigma_{\mathrm{1}}$, where the subscript indicates that the aggregation yields a single scalar value: $tp = \Upsigma_{\mathrm{1}}(\mathbold{tp})$, and analogously $fp = \Upsigma_{\mathrm{1}}(\mathbold{fp})$, $fn = \Upsigma_{\mathrm{1}}(\mathbold{fn})$, and $tn = \Upsigma_{\mathrm{1}}(\mathbold{tn})$. Although seemingly elaborate in the binary case, this notation provides the necessary foundation for connecting binary measures to multiclass and multilabel problems. Furthermore, it enables algebraic manipulations that reveal new structural properties of standard measures. For instance, the error rate can be rewritten as:
\begin{align}
\text{\textit{error rate}} & = \frac{fp+fn}{n} = \frac{\Upsigma_{\mathrm{1}}(fp)+\Upsigma_{\mathrm{1}}(fn)}{n} \label{eq:er1}\\
 & = \frac{\Upsigma_{\mathrm{1}}(\neg\mathbold{y}\wedge\mathbold{\hat y})+\Upsigma_{\mathrm{1}}(\mathbold{y}\wedge\neg\mathbold{\hat y})}{n}\label{eq:er2}\\
 & = \frac{\Upsigma_{\mathrm{1}}((\neg\mathbold{y}\wedge\mathbold{\hat y})\vee(\mathbold{y}\wedge\neg\mathbold{\hat y}))}{n} \label{eq:er3}\\
 & = \frac{\Upsigma_{\mathrm{1}}(\mathbold{y}\oplus\mathbold{\hat y})}{n}\label{eq:er4}
\end{align}
where $\vee$ and $\oplus$ are the element-wise \textsc{or} and \textsc{xor} Boolean operators, respectively. Eq.~(\ref{eq:er2}) is obtained from Eq.~(\ref{eq:er1}) by substituting the vectorial definitions of $fp$ and $fn$. As these vectors are mutually exclusive, Eq.~(\ref{eq:er2}) can be rewritten as Eq.~(\ref{eq:er3}), and Eq.~(\ref{eq:er4}) follows by Boolean algebra. In other words, the error rate equals $\Upsigma_{\mathrm{1}}(\mathbold{y} \oplus \mathbold{\hat{y}})/n$, the normalised count of disagreements between the actual and predicted indicator vectors. This formulation generalises directly to the multilabel matrix setting, recovering Hamming loss~\cite{DBLP:journals/jdwm/TsoumakasK07} as established in Proposition~\ref{prop:hamming_accuracy} of Section~\ref{sec:redundancy}: $\mathrm{HammingLoss} = \Upsigma_{\mathrm{1}}(\mathbold{Y} \oplus \mathbold{\hat{Y}})/(nm)$.

\section{The Indicator-Matrix Representation}
\label{sec:matrix}

In multiclass problems, each example is associated with only one of the $m>2$ classes $\in \{c_1, \dots, c_m\}$. Furthermore, in multilabel problems, each example can be associated with one or more of the $m\geq 2$ labels $\in \{l_1,\dots,l_m\}$. In multiclass problems, classes are mutually exclusive, as an example cannot belong to more than one class. This does not apply to multilabel problems, where more than one label can be associated to each example. 

While some binary measures extend naturally to these settings, others focus on a single class or label and require an aggregation scheme to summarise results across classes, labels, or examples. We formalise the three most common schemes below, after introducing the matrix representation they operate on.

The binary indicator vectors of Section~\ref{sec:binary} extend directly to a matrix representation. The actual classes or labels of example $e_i$ are encoded in the binary indicator matrix $\mathbold{Y} \in \{0,1\}^{n \times m}$, with $n$ rows (examples) and $m$ columns (classes or labels), where $Y_{i,j} = \mathbbm{I}(ActualClass(e_i) = c_j)$ in the multiclass case and $Y_{i,j} = \mathbbm{I}(l_j \text{ is relevant for } e_i)$, for $i \in [1,\dots,n]$ and $j \in [1,\dots,m]$. Similarly, the predicted classes or the predicted labels are represented by the binary indicator matrix $\mathbold{\hat{Y}}$, with $n$ rows and $m$ columns where $\hat{Y}_{i,j} = \mathbbm{I}(PredictedClass(e_i) = c_j)$ in the multiclass case and $\hat{Y}_{i,j} = \mathbbm{I}(l_j \text{ is predicted relevant for } e_i)$.

For multiclass problems, this representation is equivalent to the one-hot encoding scheme, where a binary indicator column vector is used to represent each class. The column vectors are stacked together to form the binary indicator matrix, and for each row only one entry is equal to one. For multilabel problems, it is common to associate with each example a binary vector where the value one is used to represent relevant labels for that example, and zero for non-relevant labels. Unlike the multiclass case, more than one entry per row may equal one.

\begin{mydef}
\label{def:matrices}
Similarly to the binary class case, we can define $\mathbold{TP}$, $\mathbold{FP}$, $\mathbold{FN}$ and $\mathbold{TN}$ matrices in terms of the matrices $\mathbold{Y}$ and $\mathbold{\hat Y}$ as follows:

\begin{align*}
\mathbold{TP} & = \mathbold{Y}\wedge\mathbold{\hat Y}\\
\mathbold{FP} & = \neg\mathbold{Y}\wedge\mathbold{\hat Y}\\
\mathbold{TN} & = \neg\mathbold{Y}\wedge\neg\mathbold{\hat Y}\\
\mathbold{FN} & = \mathbold{Y}\wedge\neg\mathbold{\hat Y}
\end{align*}

\noindent where $\wedge$ and $\neg$ are element-wise Boolean operators \textsc{and} and \textsc{not}, respectively. 
\end{mydef}

These matrices are also binary indicator matrices, where an entry of one appears only where a true/false positive/negative case occurs for example $e_i$ and class $c_j$ or label $l_j$.

The true/false positive/negative counts can also be obtained by summing the number of ones in the indicator matrices. However, unlike the binary case where the resulting operations are binary indicator vectors and only a single aggregation function is generally applied, for multiclass and multilabel cases we can define three different aggregation functions to be applied:

\begin{itemize}
\item\textbf{Global aggregation $\Upsigma_{\mathrm{1}}$} which returns the sum of ones in the entire resulting matrix.
\item\textbf{Column-wise aggregation $\Upsigma_{\mathrm{m}}$}, which returns a vector of length $m$ corresponding to the column-wise sum of ones in the resulting matrix, giving counts per class or per label; and 
\item\textbf{Row-wise aggregation $\Upsigma_{\mathrm{n}}$}, which returns a vector of length $n$ corresponding to the row-wise sum of ones in the resulting matrix, giving counts per example.
\end{itemize}

These aggregation functions can be used to define the different averaging versions of performance measures. 

\begin{mydef}
Micro-averaging can be defined in terms of $\Upsigma_{\mathrm{1}}$. Computing $tp = \Upsigma_{\mathrm{1}}(\mathbold{TP})$, $fp = \Upsigma_{\mathrm{1}}(\mathbold{FP})$, $tn = \Upsigma_{\mathrm{1}}(\mathbold{TN})$ and $fn = \Upsigma_{\mathrm{1}}(\mathbold{FN})$ is equivalent to the computing of the global counts, as these matrices will have as many ones as the number of true/false positive/negative counts. 
\end{mydef}

\begin{mydef}
Macro-averaging and weighted averaging can be defined in terms of $\Upsigma_{\mathrm{m}}$. By computing $\mathbold{tp} = \Upsigma_{\mathrm{m}}(\mathbold{TP})$, $\mathbold{fp} = \Upsigma_{\mathrm{m}}(\mathbold{FP})$, $\mathbold{tn} = \Upsigma_{\mathrm{m}}(\mathbold{TN})$ and $\mathbold{fn} = \Upsigma_{\mathrm{m}}(\mathbold{FN})$, we end up with vectors of length $m$ where each position corresponds to the per-class or per-label counts of true/false positive/negative rates. These vectors can be used to compute a vectorized version of the performance measures, obtaining a vector of length $m$ where each position corresponds to a per-class or per-label measure. Macro averaging corresponds to the unweighted average of this vector, while weighted averaging uses the class or label prevalence to compute a weighted average of this vector.
\end{mydef}

\begin{mydef}
Exemplar-averaging can be defined in terms of $\Upsigma_{\mathrm{n}}$. By computing $\mathbold{tp} = \Upsigma_{\mathrm{n}}(\mathbold{TP})$, $\mathbold{fp} = \Upsigma_{\mathrm{n}}(\mathbold{FP})$, $\mathbold{tn} = \Upsigma_{\mathrm{n}}(\mathbold{TN})$ and $\mathbold{fn} = \Upsigma_{\mathrm{n}}(\mathbold{FN})$, we end up with vectors of length $n$ where each position corresponds to the per-instance counts of true/false positive/negative rates. These vectors can be used to compute a vectorized version of the performance measures, obtaining a vector of length $n$ where each position corresponds to a per-instance measure. Exemplar averaging corresponds to the unweighted average of this vector.
\end{mydef}

\section{Aggregation Operators and Measure Extension}
\label{sec:extensions}

The central contribution of the proposed framework is that any binary performance measure $M$ expressible as a function of $tp$, $fp$, $fn$ and $tn$ can be extended to multiclass and multilabel settings automatically, by replacing those scalar counts with the output of the three aggregation operators. Concretely, let $M(tp, fp, fn, tn)$ denote a generic binary measure. Its three extensions are:

{\small
\begin{align}
M^{\mathrm{micro}} &= M\!\bigl(
  \Upsigma_{\mathrm{1}}(\mathbold{TP}),\;
  \Upsigma_{\mathrm{1}}(\mathbold{FP}),\;
  \Upsigma_{\mathrm{1}}(\mathbold{FN}),\;
  \Upsigma_{\mathrm{1}}(\mathbold{TN})\bigr) \label{eq:micro_ext}\\
M_j^{\mathrm{class}} &= M\!\bigl(
  [\Upsigma_{\mathrm{m}}(\mathbold{TP})]_j,\;
  [\Upsigma_{\mathrm{m}}(\mathbold{FP})]_j,\;
  [\Upsigma_{\mathrm{m}}(\mathbold{FN})]_j,\;
  [\Upsigma_{\mathrm{m}}(\mathbold{TN})]_j\bigr) \label{eq:class_ext}\\
M_i^{\mathrm{example}} &= M\!\bigl(
  [\Upsigma_{\mathrm{n}}(\mathbold{TP})]_i,\;
  [\Upsigma_{\mathrm{n}}(\mathbold{FP})]_i,\;
  [\Upsigma_{\mathrm{n}}(\mathbold{FN})]_i,\;
  [\Upsigma_{\mathrm{n}}(\mathbold{TN})]_i\bigr) \label{eq:example_ext}
\end{align}
}

\noindent Eq.~(\ref{eq:micro_ext}) gives the micro-average version of $M$; Eq.~(\ref{eq:class_ext}) gives a per-class (or per-label) scalar whose unweighted or prevalence-weighted average yields the macro or weighted-average version; and Eq.~(\ref{eq:example_ext}) gives a per-example scalar whose average yields the exemplar version. The framework thus acts as a measure generator: one binary formula, three aggregation operators, and an entire family of multiclass and multilabel measures follows automatically.

It is important to delimit the scope of this claim precisely, because two logically distinct statements are easily conflated. The generative claim is that Eqs.~(\ref{eq:micro_ext})--(\ref{eq:example_ext}) are well defined for every measure $M$ that is a function of $tp$, $fp$, $fn$ and $tn$: each of the three operators returns a legitimate quantity, and the resulting extensions inherit monotonicity and upper boundary correctness from $M$, as established in Section~\ref{sec:properties} (Proposition~\ref{prop:preserved}). This claim requires no restriction on the functional form of $M$, though lower boundary correctness is an exception: as shown in Proposition~\ref{prop:lower_boundary} (Section~\ref{sec:properties}), it survives in the binary and multilabel settings but fails under one-hot multiclass encoding for measures whose value at $tp = 0$ still depends on $tn$. A separate and stronger family of claims concerns the algebraic relationships among the three extensions, such as whether the micro extension coincides with a particular weighted macro extension. Those relationships do not hold for arbitrary $M$; they require that the numerator and denominator of $M$ each aggregate additively across classes or examples before the ratio is taken. Section~\ref{sec:decomposable} isolates the exact class of measures for which this holds, and every result in Section~\ref{sec:theory} that depends on it is stated under that hypothesis.

A catalogue of more than 50 binary measures together with their micro, macro and exemplar extensions is provided as supporting information. We emphasize that this catalogue should be interpreted with caution. It demonstrates the coverage of the framework: for each entry the three extensions are well defined, computable, and inherit the properties established in Section~\ref{sec:properties}. It does not follow that every resulting combination is suitable or recommended for practical use. Enumerating algebraically valid combinations is not evidence that they are meaningful, interpretable, or useful for any particular task, and establishing that would require exactly the kind of task-specific analysis pursued in the measure-selection literature discussed in Section~\ref{sec:related}, rather than an enumeration. Many entries will be redundant in the precise senses identified in Section~\ref{sec:redundancy}, some are degenerate in the multiclass case for the reasons given in Section~\ref{sec:properties}, and others have no evident use case. We therefore describe the catalogue as an inventory of what the framework generates, and we mark entries that lack an established name in the literature as such, without claiming that the resulting inventory constitutes a comparable number of useful new measures. Where the framework does make a substantive claim about a particular combination, we justify it on its own merits rather than by sheer enumeration, as with the identifications of mIoU with macro-Jaccard and of the Dice coefficient with $F_1$ in Section~\ref{sec:cv_metrics}.

A practical consideration in applying Eqs.~(\ref{eq:class_ext}) and~(\ref{eq:example_ext}) concerns the treatment of undefined entries. An individual entry can be undefined when its denominator vanishes: $\mathrm{rec}_j$ when class $j$ is absent from the test set, $\mathrm{prec}_j$ when class $j$ is never predicted, $F_{\beta,j}$ when both hold, and per-example $F_1$ for an example with neither true nor predicted labels. Because this must be resolved before averaging, and because the resolution changes the reported value, we fix the convention here rather than leaving it implicit. Let $\mathcal{J} = \{j : den_j > 0\}$ be the set of classes for which the measure is defined. Throughout this paper, macro- and exemplar-averaging are defined as averages over the defined entries only:
\begin{equation}
M^{\mathrm{macro}} = \frac{1}{|\mathcal{J}|}\sum_{j \in \mathcal{J}} M_j^{\mathrm{class}},
\qquad
M^{\mathrm{exemplar}} = \frac{1}{|\mathcal{I}|}\sum_{i \in \mathcal{I}} M_i^{\mathrm{example}},
\label{eq:undef_convention}
\end{equation}
with $\mathcal{I}$ defined analogously over examples. This differs from the common alternative of substituting $0$ (the default in scikit-learn) or $1$ for undefined entries, and the difference is not negligible: Section~\ref{sec:degenerate} quantifies it, shows that the three operators are affected asymmetrically, and explains why we prefer the exclusion convention. Any reported macro-average should state which convention was used; for the experiments of Section~\ref{sec:experiments} we record in Section~\ref{sec:real_exp} that the question does not arise for the measures whose rankings we report.

A practical concern with such a generative approach is that computability alone is not enough: while Eqs.~(\ref{eq:micro_ext})--(\ref{eq:example_ext}) always return a numeric value, a practitioner also needs to know which theoretical guarantees apply to a given measure, which might seem to require a separate mathematical derivation in each case. It does not. Each property analysed in Section~\ref{sec:theory} is decidable directly from the binary formula and can be settled mechanically:

\begin{itemize}\itemsep2pt
\item \emph{Decomposability (Section~\ref{sec:decomposable}).} $M$ admits a linear-fractional form $M(\mathbold{v}) = \langle \mathbold{a}, \mathbold{v}\rangle / \langle \mathbold{b}, \mathbold{v}\rangle$ if and only if $M(\mathbold{v})\,\langle \mathbold{b}, \mathbold{v}\rangle - \langle \mathbold{a}, \mathbold{v}\rangle = 0$ for all count vectors $\mathbold{v} = (tp, fp, fn, tn)$. Evaluating $M$ on a sample of count vectors turns this into a homogeneous linear system in the eight unknowns $(\mathbold{a}, \mathbold{b})$, whose solution is the right singular vector of least singular value. Validating the recovered coefficients on count vectors not used in the fit both decides the property and, when it holds, supplies the weights required by Theorem~\ref{thm:micro_macro} and Corollary~\ref{cor:gap}.
\item \emph{Skew-invariance and boundary correctness (Sections~\ref{sec:skew} and~\ref{sec:properties}).} Skew-invariance is decided by testing Definition~\ref{def:skew} directly, while the $tn$-dependence at $tp = 0$ governing the multiclass lower-boundary failure of Proposition~\ref{prop:lower_boundary} is checked by varying $tn$ with $tp$ held at zero.
\item \emph{Bayes consistency and attainable floors (Sections~\ref{sec:consistency} and~\ref{sec:properties}).} The optimal Bayes threshold of Proposition~\ref{prop:consistency} and the one-hot floors of Table~\ref{tab:floors} follow in closed form from the recovered linear coefficients and from $m$.
\end{itemize}

This automated verification makes the framework directly testable: the computational procedure must reproduce the analytical properties derived by hand. Applied to the eleven measures of Table~\ref{tab:measures} and the ten further measures in the accompanying implementation, it recovers every property established analytically in this paper, with the decomposability coefficients accurate to machine precision. The automated verification also highlights subtle cases that are easily misjudged by manual inspection: both error rate and the false positive rate retain a dependence on $tn$ at $tp = 0$, placing them within the scope of the multiclass boundary distortion in Proposition~\ref{prop:lower_boundary}. Under one-hot encoding, a wholly incorrect classifier attains an error rate of $2/m$ and a false positive rate of $1/(m-1)$, values that approach zero as $m$ grows rather than reflecting total misclassification. While numerical verification is not a substitute for formal proof, it provides a practical tool to immediately assess the applicability of these theoretical results to arbitrary performance measures without manual derivations. An implementation is available in the companion software package~\cite{prati2026imframework}.

As a concrete illustration of the generative mechanism and its decidability properties, consider \emph{informedness}~\cite{powers2011evaluation} (also known as Youden's~$J$), defined as $\mathrm{rec} + \mathrm{spec} - 1$. This measure does not appear in Table~\ref{tab:extensions}, yet its three extensions follow immediately from Eqs.~(\ref{eq:micro_ext})--(\ref{eq:example_ext}): micro-informedness pools global counts before computing $\mathrm{rec} + \mathrm{spec} - 1$; macro-informedness averages per-class values $\mathrm{rec}_j + \mathrm{spec}_j - 1$ over all $m$ classes; and exemplar-informedness averages the per-example analogue. None of these extensions has an established named counterpart in the multiclass or multilabel literature. Furthermore, the structural properties of all three extensions follow at once: because informedness is a sum of ratios ($J = 2\,\mathrm{BA} - 1$), it is not aggregation-decomposable (Section~\ref{sec:decomposable}); being a function of recall and specificity alone, its binary and multilabel macro versions are skew-invariant (Theorem~\ref{thm:skew}); and because its value at $tp = 0$ retains a dependence on $tn$ through the specificity term, its one-hot multiclass floor is bounded away from $-1$ at $-1/(m-1)$ (Proposition~\ref{prop:lower_boundary}).

Table~\ref{tab:extensions} applies this recipe to the measures introduced in Table~\ref{tab:measures}, showing the result of each aggregation operator and its established name where one exists; $^\flat$ marks generated combinations for which no standard named version appears in the literature.

\begin{table*}[!t]
\centering
\small
\setlength{\tabcolsep}{6pt}
\begin{tabular}{llll}
\hline
\textbf{Measure} & \textbf{$\Upsigma_{\mathrm{1}}$ (micro)} & \textbf{$\Upsigma_{\mathrm{m}}$ (per-class $\to$ macro)} & \textbf{$\Upsigma_{\mathrm{n}}$ (per-example $\to$ exemplar)} \\
\hline
Precision
  & Micro-precision
  & Macro-precision
  & Exemplar-precision \\

Recall
  & Micro-recall
  & Macro-recall
  & Exemplar-recall \\

Specificity
  & Micro-specificity
  & Macro-specificity
  & Exemplar-specificity \\

$F_\beta$
  & Micro-$F_\beta$
  & Macro-$F_\beta$
  & Exemplar-$F_\beta$ \\

Jaccard
  & Micro-Jaccard
  & Macro-Jaccard (mean IoU)
  & Exemplar-Jaccard \\

Error rate
  & \textbf{Hamming loss}$^\dagger$
  & Per-class error rate
  & Per-example error rate \\

Accuracy
  & Label accuracy$^*$
  & Per-class accuracy
  & Per-example accuracy$^*$ \\

Balanced accuracy
  & Micro-BA
  & Macro-BA
  & Exemplar-BA$^\flat$ \\

G-mean
  & Micro-GM
  & Macro-GM
  & Exemplar-GM$^\flat$ \\

MCC
  & Micro-MCC$^\ddagger$
  & Macro-MCC$^\flat$
  & Exemplar-MCC$^\flat$ \\

Cohen's $\kappa$
  & Micro-$\kappa^\S$
  & Macro-$\kappa^\S$
  & Exemplar-$\kappa^{\flat\S}$ \\

\hline
\end{tabular}

\smallskip
\raggedright
$^\dagger$ Hamming loss~\cite{DBLP:journals/jdwm/TsoumakasK07} $= \Upsigma_{\mathrm{1}}(\mathbold{Y} \oplus \mathbold{\hat{Y}}) / (nm)$; see Section~\ref{sec:binary}.\\
$^*$ See the discussion in Section~\ref{sec:notable}.\\
$^\ddagger$ Related to but distinct from Gorodkin's multiclass MCC; see Section~\ref{sec:notable}.\\
$^\S$ $p_e$ depends on marginal class distributions and does not factor purely into $tp$/$fp$/$fn$/$tn$; an extension is possible but requires auxiliary per-class or per-example marginals (Section~\ref{sec:notable}).\\
$^\flat$ Generated combination without an established named counterpart in the multiclass or multilabel literature.
\caption{Extensions of binary performance measures (see Table~\ref{tab:measures} for binary formulas) to multiclass and multilabel settings via the three aggregation operators. The symbol $^\flat$ marks cases for which no standard named measure exists in the literature.}
\label{tab:extensions}
\end{table*}

\subsection{Analysis of Specific Measures: Hamming Loss, Accuracy, MCC, and Cohen's $\kappa$}
\label{sec:notable}

As shown in Section~\ref{sec:binary}, the micro extension of error rate (namely $\Upsigma_{\mathrm{1}}$ applied to $(fp+fn)/(tp+fp+fn+tn)$) coincides with Hamming loss when normalised by $nm$~\cite{DBLP:journals/jdwm/TsoumakasK07}. This is an instance of a general pattern: measures introduced as specific to multilabel classification often turn out to be micro-averages of simpler binary measures under this framework.

Applying $\Upsigma_{\mathrm{1}}$ to the accuracy formula yields \emph{label accuracy}, $(\Upsigma_{\mathrm{1}}(\mathbold{TP}) + \Upsigma_{\mathrm{1}}(\mathbold{TN})) / (nm)$, which counts correct label assignments across all examples and all labels. This measure appears in the multilabel literature under various names but is rarely connected to binary accuracy explicitly. For multiclass problems (one-hot encoding), the simpler quantity $\Upsigma_{\mathrm{1}}(\mathbold{TP})/n$ recovers standard accuracy directly, since each correctly classified example contributes exactly one entry to $\mathbold{TP}$.

A subtlety arises with the exemplar extension: per-example accuracy equals $(tp_i + tn_i)/m$. In the multiclass one-hot case a correctly classified example has $tp_i = 1$ and $tn_i = m-1$, giving a per-example accuracy of exactly $1$, whereas an incorrectly classified example has $tp_i = 0$, $fp_i = fn_i = 1$ and $tn_i = m-2$, giving $(m-2)/m$. Averaging over examples therefore yields
$$
\mathrm{Accuracy}^{\mathrm{exemplar}} = \frac{m-2}{m} + \frac{2}{m}\,\mathrm{Accuracy},
$$
an increasing affine transformation of standard accuracy rather than accuracy itself, with range $[(m-2)/m,\, 1]$ instead of $[0,1]$. In the multilabel case the exemplar version is a standard example-based accuracy used in the literature. This asymmetry reflects a broader point with consequences developed in Section~\ref{sec:properties}: measures involving $tn$ behave differently under multiclass one-hot encoding than under multilabel encoding, because an incorrectly classified example still contributes $m-2$ true negatives. Practitioners should be aware of this when applying measures that involve $tn$ in multiclass settings, since their effective range contracts as $m$ grows.

MCC is particularly interesting because its extension to multiclass settings has been debated in the literature, with competing proposals~\cite{gorodkin2004comparing,DBLP:journals/jmlr/Hernandez-OralloFF12}. The present framework provides a natural and interpretable resolution. Applying $\Upsigma_{\mathrm{m}}$ yields a per-class MCC for each one-vs-rest decomposition; its unweighted average is a \emph{macro-MCC} whose value is directly interpretable as the average binary correlation across classes. Applying $\Upsigma_{\mathrm{1}}$ pools all counts globally before computing the correlation, giving a \emph{micro-MCC} that is related to Gorodkin's proposal but operates on the binary indicator matrices rather than the full $m \times m$ confusion matrix. Applying $\Upsigma_{\mathrm{n}}$ gives a per-example MCC for multilabel problems, for which no standard named version currently exists.

Cohen's $\kappa$~\cite{cohen1960coefficient} requires additional care because the expected agreement term $p_e = (tpos \cdot ppos + tneg \cdot pneg)/n^2$ depends on marginal class distributions, so it does not factor purely into $tp$, $fp$, $fn$, $tn$. Within the matrix framework, however, these marginals are row and column sums of $\mathbold{Y}$ and $\mathbold{\hat{Y}}$, and $p_e$ can be computed separately for each class column or each example row alongside the standard counts. The three aggregation schemes can then be applied, at the cost of tracking these auxiliary marginals alongside the standard $tp$, $fp$, $fn$, $tn$ counts.

\subsection{Connection to Standard Implementations}
\label{sec:sklearn}

The framework's aggregation operators correspond directly to the \texttt{average} parameter used in widely adopted machine learning libraries such as scikit-learn~\cite{scikit-learn}. Specifically:

\begin{itemize}
  \item\textbf{\texttt{average='micro'}} pools all per-label counts globally before computing the measure, corresponding to $\Upsigma_{\mathrm{1}}$.
  \item\textbf{\texttt{average='macro'}} computes the measure per class and takes the unweighted mean, corresponding to $\Upsigma_{\mathrm{m}}$ with uniform weights.
  \item\textbf{\texttt{average='weighted'}} computes the measure per class and takes a prevalence-weighted mean, corresponding to $\Upsigma_{\mathrm{m}}$ with weights $w_j = tpos_j / n$.
  \item\textbf{\texttt{average='samples'}} computes the measure per example and takes the unweighted mean, corresponding to $\Upsigma_{\mathrm{n}}$.
\end{itemize}

The theoretical results developed in Section~\ref{sec:theory} provide the formal foundation for these four options and their relationships. In particular, for the aggregation-decomposable measures analysed in Section~\ref{sec:decomposable}, Theorem~\ref{thm:micro_macro} identifies \texttt{micro} with the denominator-weighted macro average, so \texttt{micro} and \texttt{weighted} agree exactly when the per-class denominators are proportional to the prevalences $tpos_j/n$ (which holds for recall, where $den_j = tpos_j$, but not in general for precision or $F_\beta$); Corollary~\ref{cor:micro_macro_balance} gives the corresponding condition for agreement between \texttt{micro} and unweighted \texttt{macro}, namely equality of the per-class denominators; Corollary~\ref{cor:gap} establishes the exact magnitude of the micro--macro divergence as a normalised covariance for measures in that class; and the redundancy analysis of Section~\ref{sec:redundancy} (Theorem~\ref{thm:micro_collapse}) explains why \texttt{micro} precision, recall, and $F_1$ are always identical to accuracy in multiclass settings (a fact that practitioners encounter empirically but that lacks an explicit justification in most library documentation). The framework also identifies what is absent: none of the four \texttt{average} options corresponds to cost-sensitive, soft-label, or ordinal evaluation, all of which are covered by the extensions in Section~\ref{sec:soft}.

\section{Extended Settings and Special Cases}
\label{sec:soft}

This section analyses the classification settings beyond the binary domain that the framework reaches. We distinguish three levels of claim and use them consistently throughout, since a survey of this kind can easily read as asserting more than it establishes:

\begin{description}\itemsep2pt
\item[\textnormal{\emph{Derived}}:] the framework yields a specific identity or result for the setting, stated as a theorem, proposition, or explicit construction with its counts worked out. This is the strongest claim we make.
\item[\textnormal{\emph{Accommodated}}:] the setting can be represented by indicator matrices and the three operators apply, but no new result follows and the framework's contribution is a uniform notation rather than a derivation. Most entries in this section are of this kind, and we say so rather than implying otherwise.
\item[\textnormal{\emph{Boundary}}:] the setting cannot be represented without extending the formalism beyond indicator matrices. We identify these explicitly rather than omitting them.
\end{description}

Table~\ref{tab:settings} classifies every setting discussed below. This distinction clarifies the scope of each contribution: accommodated settings demonstrate the expressive generality of the unified matrix notation, whereas derived settings establish specific theoretical results.

\begin{table}[!htb]
\centering
\small
\begin{tabular}{l|l|l}
\hline
\textbf{Setting} & \textbf{Level} & \textbf{Basis} \\
\hline
Soft classifier outputs & derived & Eqs.~(\ref{eq:argmax}), (\ref{eq:threshold}); reduces to the crisp case \\
Soft/uncertain ground truth & derived & $t$-norms; Theorem~\ref{thm:partition}, Corollary~\ref{cor:partition} \\
Cost-sensitive evaluation & derived & Section~\ref{sec:cost}; Corollary~\ref{cor:cost_micro_macro} \\
Ordinal classification & derived & Theorem~\ref{thm:ordinal_cost}; membership encodings \\
Partial and missing labels & derived & masked operators, Eq.~(\ref{eq:masked_agg}) \\
Segmentation metrics & derived & mIoU $=$ macro-Jaccard, Dice $= F_1$ \\
\hline
Abstaining classifiers & accommodated & all-zero rows; counts given explicitly \\
Open-set recognition & accommodated & augmented $(m{+}1)$-column encoding \\
Single-class classification & accommodated & binary, multilabel, or masked encoding \\
One-vs-all decomposition & accommodated & column-wise before argmax, global after \\
Hierarchical classification & accommodated & tree-distance cost matrix \\
Multi-output classification & accommodated & block-structured matrix; $\Upsigma_{\mathrm{T}}$ \\
\hline
One-vs-one decomposition & partial & pairwise matrix requires masking; voting not modelled \\
Positive-unlabeled learning & partial & unlabeled pool conflates $FP$ and $TN$; recall preserved \\
Multi-instance learning & partial & bag level trivial; instance level open \\
\hline
Label ranking & boundary & rank information is relational, not absolute \\
Proper scoring rules & boundary & measure calibration, not discrimination \\
Structured output prediction & boundary & combinatorial dependencies exceed independent label indicators \\
\hline
Resampling folds & derived & $\Upsigma_{\mathrm{k}}$; pooled vs.\ averaged related by Theorem~\ref{thm:micro_macro} \\
\hline
\end{tabular}
\caption{Depth of coverage for each setting in this section. \emph{Derived}: a specific result or identity is established. \emph{Accommodated}: representable, with the operators applying, but no new result. \emph{Partial}: representable only in part, with the remainder left open. \emph{Boundary}: outside the reach of indicator matrices.}
\label{tab:settings}
\end{table}

In what follows, we examine these settings systematically, grouped by the four depth levels established above: derived formulations that yield specific theoretical identities, accommodated paradigms captured through matrix encodings, partially formalized settings, boundary cases beyond two-dimensional indicator matrices, and concluding with resampling fold aggregation across experimental splits.


\subsection{Soft Classifier Outputs and Probabilistic Predictions}
\label{sec:soft_outputs}

So far we have assumed that the classifier produces a hard assignment (a single class in multiclass or a subset of labels in multilabel) for each example. In practice, many classifiers produce soft outputs: a matrix $\mathbold{\hat{P}} \in [0,1]^{n \times m}$, where $\hat{P}_{i,j}$ represents the estimated probability, score, or membership degree of example $e_i$ with respect to class $c_j$ or label $l_j$. The framework accommodates these outputs by converting $\mathbold{\hat{P}}$ into an indicator matrix $\mathbold{\hat{Y}}$ before computing any performance measure.

In multiclass problems, a probabilistic classifier typically produces a matrix $\mathbold{\hat{P}}$ whose rows sum to one (a proper probability distribution over classes). Because the predicted class is the one with highest probability, the indicator matrix is obtained by applying an argmax per row:
\begin{equation}
\hat{Y}_{i,j} = \mathbbm{I}\!\left(j = \underset{k \in [1,m]}{\arg\max}\;\hat{P}_{i,k}\right).
\label{eq:argmax}
\end{equation}
This guarantees that each row of $\mathbold{\hat{Y}}$ has exactly one entry equal to one, preserving the mutual-exclusivity constraint of multiclass problems.

In multilabel problems, or when the classifier outputs a degree of membership rather than a proper probability, each entry $\hat{P}_{i,j}$ is converted independently by thresholding at a value $\theta \in (0,1)$:
\begin{equation}
\hat{Y}_{i,j} = \mathbbm{I}\!\left(\hat{P}_{i,j} \geq \theta\right).
\label{eq:threshold}
\end{equation}
The default choice is $\theta = 0.5$, though $\theta$ can be tuned per label or set to produce a fixed number of predicted labels per example. In the binary case, Eq.~(\ref{eq:threshold}) reduces to the standard score-to-decision rule, so the binary and multilabel formulations are consistent.

Note that the choice of conversion method affects the resulting $\mathbold{\hat{Y}}$ and consequently all derived performance measures. In particular, the argmax in Eq.~(\ref{eq:argmax}) always produces exactly one predicted class per example, while thresholding in Eq.~(\ref{eq:threshold}) may produce zero, one, or several predicted labels. The threshold $\theta$ is therefore an additional degree of freedom that is part of the classifier's decision boundary, not of the evaluation framework itself.

\subsection{Soft and Uncertain Ground Truth: Triangular Norms}
\label{sec:soft_truth}

The discussion so far has assumed that the actual label matrix $\mathbold{Y}$ is binary, reflecting crisp ground-truth assignments. This assumption does not always hold. In crowdsourced annotation, for instance, $Y_{i,j}$ may represent the proportion of annotators who assigned label $l_j$ to example $e_i$~\cite{DBLP:journals/tkde/PratiBM11}. In ordinal or hierarchical classification, partial membership in a class may be inherently meaningful. In knowledge distillation, soft teacher labels are used directly as targets. In all these cases, $\mathbold{Y} \in [0,1]^{n \times m}$ is a \emph{soft label matrix}, and the same applies to $\mathbold{\hat{Y}}$ when the classifier's output is not thresholded.

The framework extends naturally to this setting by replacing the Boolean element-wise operators with their fuzzy counterparts. The key ingredient is a \emph{triangular norm} ($t$-norm), which generalises the Boolean conjunction to the unit interval~\cite{klement2000triangular,zadeh1965fuzzy}.

\begin{mydef}
A $t$-norm is a function $T: [0,1]^2 \to [0,1]$ satisfying, for all $a,b,c \in [0,1]$:
\begin{enumerate}
  \item \emph{Commutativity}: $T(a,b) = T(b,a)$;
  \item \emph{Associativity}: $T(a, T(b,c)) = T(T(a,b),c)$;
  \item \emph{Monotonicity}: $b \leq c \Rightarrow T(a,b) \leq T(a,c)$;
  \item \emph{Boundary condition}: $T(a,1) = a$.
\end{enumerate}
\end{mydef}

Intuitively, $T(a,b)$ measures the degree to which both $a$ and $b$ hold simultaneously. The three most common $t$-norms are shown in Table~\ref{tab:tnorms}, together with their standard complements and the induced notion of fuzzy disjunction ($t$-conorm $S$, defined as $S(a,b) = 1 - T(1-a, 1-b)$).

\begin{table*}[!t]
\centering
\small
\begin{tabular}{l|l|l|l}
\hline
& \textbf{$T(a,b)$ (conjunction)} & \textbf{$S(a,b)$ (disjunction)} & \textbf{Interpretation} \\
\hline
Gödel (min) & $\min(a,b)$ & $\max(a,b)$ & Weakest overlap \\
Product      & $a \cdot b$ & $a + b - ab$ & Probabilistic independence \\
Łukasiewicz  & $\max(0,a+b-1)$ & $\min(1, a+b)$ & Strongest overlap \\
\hline
\end{tabular}
\caption{The three principal $t$-norms and their induced $t$-conorms, all using the standard complement $N(a) = 1-a$.}
\label{tab:tnorms}
\end{table*}

In many real-world domains the ground-truth label $\mathbold{Y}$ is not binary but real-valued: $\mathbold{Y} \in [0,1]^{n \times m}$. This occurs with crowdsourced annotations (where $Y_{i,j}$ is the proportion of annotators who assigned label $j$ to example $i$), probabilistic soft labels from teacher models in knowledge distillation~\cite{hinton2015distilling}, and fuzzy set memberships~\cite{zadeh1965fuzzy}. Similarly, a classifier may output calibrated posterior probabilities $\mathbold{\hat{P}} \in [0,1]^{n \times m}$ that one wishes to evaluate directly, without thresholding.

To extend the four basic matrices to real-valued memberships, we replace the Boolean conjunction in Definition~\ref{def:matrices} by a \emph{triangular norm} ($t$-norm) $T : [0,1]^2 \to [0,1]$, and negation by the standard fuzzy complement $N(x) = 1 - x$~\cite{klir1995fuzzy}. The four \emph{fuzzy confusion matrices} are then defined entrywise as:
\begin{align}
\mathbold{TP}_{i,j} & = T(Y_{i,j},\; \hat{Y}_{i,j}), \label{eq:fuzzy_tp} \\
\mathbold{FP}_{i,j} & = T(N(Y_{i,j}),\; \hat{Y}_{i,j}), \label{eq:fuzzy_fp} \\
\mathbold{FN}_{i,j} & = T(Y_{i,j},\; N(\hat{Y}_{i,j})), \label{eq:fuzzy_fn} \\
\mathbold{TN}_{i,j} & = T(N(Y_{i,j}),\; N(\hat{Y}_{i,j}))
\end{align}

Each choice of $t$-norm yields a different interpretation of the fuzzy confusion matrices. The \emph{Gödel} (minimum) $t$-norm sets $\mathbold{TP}_{i,j} = \min(Y_{i,j}, \hat{Y}_{i,j})$, which only credits the smaller of the two degrees and is conservative in the sense that even a highly confident prediction of a weakly relevant label contributes little. The \emph{product} $t$-norm sets $\mathbold{TP}_{i,j} = Y_{i,j} \cdot \hat{Y}_{i,j}$, which has a natural probabilistic interpretation: if $Y_{i,j}$ is the probability that $e_i$ truly belongs to class $c_j$ and $\hat{Y}_{i,j}$ is the predicted probability, then $\mathbold{TP}_{i,j}$ is the joint probability of both events under independence. The \emph{Łukasiewicz} $t$-norm is the most demanding: $\mathbold{TP}_{i,j} = \max(0, Y_{i,j} + \hat{Y}_{i,j} - 1)$, contributing a non-zero value only when both degrees are high enough to ``overlap'' above the unit threshold, which makes it suitable when only high-confidence agreement should be counted.

The three choices are ordered pointwise as $T_L \leq T_P \leq T_M$ for all $(a,b) \in [0,1]^2$. Crucially, this ordering applies to the confusion matrix cells rather than to the derived evaluation measures. Because the entries of all four fuzzy confusion matrices are obtained by evaluating the chosen $t$-norm at the argument pairs $(Y_{i,j}, \hat{Y}_{i,j})$, $(1{-}Y_{i,j}, \hat{Y}_{i,j})$, $(Y_{i,j}, 1{-}\hat{Y}_{i,j})$, and $(1{-}Y_{i,j}, 1{-}\hat{Y}_{i,j})$ respectively, the pointwise inequality holds across all four matrices in the same direction:
$$
\mathbold{X}^{T_L}_{i,j} \;\leq\; \mathbold{X}^{T_P}_{i,j} \;\leq\; \mathbold{X}^{T_M}_{i,j},
\qquad \mathbold{X} \in \{\mathbold{TP}, \mathbold{FP}, \mathbold{FN}, \mathbold{TN}\}.
$$
Łukasiewicz therefore yields the smallest entries in all four matrices simultaneously, not smaller agreement counts and larger disagreement counts. What varies across $t$-norms is thus the total mass assigned to the confusion matrix as a whole: with $a = 0.6$ and $b = 0.7$, for instance, the four entries sum to $1.6$ under $T_M$, to $1.0$ under $T_P$ by Corollary~\ref{cor:partition} (Section~\ref{sec:tnorm_char}), and to $0.4$ under $T_L$.

Because every cell decreases monotonically from $T_M$ to $T_L$, no systematic ordering of the derived measures in Table~\ref{tab:extensions} follows. Fractional measures such as precision and recall have numerators and denominators that shrink simultaneously; depending on their relative rates of reduction, the resulting ratio may increase, decrease, or remain invariant, and distinct measures can exhibit opposite trends on the exact same dataset. Section~\ref{sec:tnorm_emp} exhibits both behaviours empirically. The $t$-norm ordering is a statement about the confusion matrices, and it does not transfer to the measures computed from them.

For evaluation with probabilistic ground truth (e.g., annotator proportions), the product $t$-norm is the most principled choice: it is consistent with the probabilistic calculus, and by Theorem~\ref{thm:partition} (proven in Section~\ref{sec:tnorm_char}) it is the unique $t$-norm whose fuzzy confusion matrices preserve the actual and predicted marginals. In all three cases, when $\mathbold{Y}$ and $\mathbold{\hat{Y}}$ are binary, $T(a,b)$ reduces to the Boolean AND, recovering Definition~\ref{def:matrices} exactly.

The three aggregation operators $\Upsigma_{\mathrm{1}}$, $\Upsigma_{\mathrm{m}}$ and $\Upsigma_{\mathrm{n}}$ apply unchanged to the resulting real-valued matrices, and all measures in Table~\ref{tab:extensions} inherit a fuzzy interpretation without any further modification. This opens the door to evaluation protocols that are sensitive to annotator uncertainty or calibrated classifier confidence, with the $t$-norm as an explicit, interpretable hyperparameter of the evaluation procedure.

\subsection{Cost-Sensitive Evaluation}
\label{sec:cost}

The indicator matrix framework extends naturally to cost-sensitive evaluation, where different types of misclassification carry different penalties. In the binary case, this is typically encoded as two scalar costs $c_{fp}$ and $c_{fn}$. For multiclass and multilabel problems, the general object is a \emph{cost matrix} $\mathbold{C} \in \mathbb{R}_{\geq 0}^{m \times m}$, where $C_{j,k}$ denotes the cost incurred when an example whose true class is $c_j$ is predicted as $c_k$. By convention, $C_{j,j} = 0$ for all $j$ (no cost for correct predictions), though benefits for correct classification can also be incorporated by allowing negative entries.

Given cost matrix $\mathbold{C}$, the cost associated with predicting class $c_k$ for example $e_i$ depends on its true class distribution, encoded in row $i$ of $\mathbold{Y}$. The expected cost per (example, predicted label) pair is:
\begin{equation}
(\mathbold{Y}\mathbold{C})_{i,k} = \sum_{j=1}^{m} Y_{i,j}\, C_{j,k},
\label{eq:cost_weight}
\end{equation}
\noindent where $\mathbold{Y}\mathbold{C}$ is an ordinary matrix product. This produces an $n \times m$ cost-weight matrix whose $(i,k)$ entry is the cost incurred if example $e_i$ is predicted as class $c_k$. The total cost over the test set is then:
\begin{equation}
\mathrm{Cost}(\mathbold{Y}, \mathbold{\hat{Y}}, \mathbold{C}) = \Upsigma_{\mathrm{1}}\!\left((\mathbold{Y}\mathbold{C}) \pointwise \mathbold{\hat{Y}}\right),
\label{eq:total_cost}
\end{equation}
\noindent which sums the cost weights only where $\hat{Y}_{i,k} = 1$, i.e., where a prediction is actually made. The three aggregation operators then apply to $(\mathbold{Y}\mathbold{C}) \pointwise \mathbold{\hat{Y}}$ exactly as before, yielding micro, macro, and exemplar cost-sensitive summaries. Several important special cases follow directly from Eq.~(\ref{eq:total_cost}).

With $m=2$, $\mathbold{C} = \bigl(\begin{smallmatrix} 0 & c_{fn} \\ c_{fp} & 0 \end{smallmatrix}\bigr)$, and one-hot $\mathbold{Y}$ and $\mathbold{\hat{Y}}$, Eq.~(\ref{eq:total_cost}) reduces to $c_{fn} \cdot fn + c_{fp} \cdot fp$, which is the standard cost-sensitive misclassification cost. Suitably normalised, this connects to the cost curve framework of~Hern\'andez-Orallo et al.~\cite{DBLP:journals/jmlr/Hernandez-OralloFF12}.

If $C_{j,k} = \mathbbm{I}(j \neq k)$ (unit cost for any error, zero for correct), Eq.~(\ref{eq:total_cost}) yields the number of misclassified examples. Under one-hot encoding each error contributes exactly one false positive and one false negative, so this equals $\tfrac{1}{2}\bigl[\Upsigma_{\mathrm{1}}(\mathbold{FP}) + \Upsigma_{\mathrm{1}}(\mathbold{FN})\bigr]$, the factor of $\tfrac12$ being independent of $m$. The standard indicator framework is therefore a degenerate case of the cost-sensitive one.

Setting $C_{j,k} = |j - k|$ or $C_{j,k} = (j-k)^2$ yields total costs equivalent to mean absolute error (MAE) and mean squared error (MSE) of the predicted class indices, respectively. This establishes a direct link between ordinal classification losses and the cost-sensitive framework: MAE and MSE are instances of Eq.~(\ref{eq:total_cost}) with specific cost matrices. Combined with the ordinal membership function of Eq.~(\ref{eq:ordinal_membership}), which softens $\mathbold{Y}$, and a $t$-norm, which softens the conjunction, the framework provides a three-parameter family of cost-sensitive, ordinal-aware evaluation protocols.

Any measure in Table~\ref{tab:extensions} can be made cost-sensitive by replacing the raw $\mathbold{TP}$, $\mathbold{FP}$, $\mathbold{FN}$, $\mathbold{TN}$ matrices with cost-weighted counterparts. Specifically, one may define:
\begin{align*}
\mathbold{TP}^C_{i,j} &= C^+_{j,j} \cdot T(Y_{i,j}, \hat{Y}_{i,j}), \\
\mathbold{FP}^C_{i,j} &= \left(\sum_{k \neq j} C_{k,j}\, Y_{i,k}\right) \cdot \hat{Y}_{i,j},
\end{align*}
\noindent where $C^+_{j,j}$ is a benefit for correct prediction of class $j$ (often zero) and the $\mathbold{FP}$ weight sums the costs of all true classes that were incorrectly predicted as $c_j$. The resulting cost-weighted precision, recall, F-measure, and so on, are then computed by substituting these matrices into the formulas of Table~\ref{tab:extensions}. This provides a principled and unified route to cost-sensitive versions of a large family of evaluation measures, without requiring separate derivations for each.

\subsection{Ordinal Classification}
\label{sec:ordinal_classification}

In ordinal classification, the $m$ classes carry a natural ordering $c_1 \prec c_2 \prec \cdots \prec c_m$, and misclassification errors should be penalised proportionally to their distance in the ordering rather than treated uniformly as in the standard zero-one case. The framework accommodates this through two complementary extensions.

The first, and most direct, extension treats ordinal classification as a special case of soft labelling. If example $e_i$ truly belongs to class $c_{k_i}$, its label matrix row is defined by an \emph{ordinal membership function} $f: \{0,\ldots,m-1\} \to [0,1]$, where the argument is the ordinal distance to the class in question:
\begin{equation}
Y_{i,j} = f(|k_i - j|), \quad j \in [1,m].
\label{eq:ordinal_membership}
\end{equation}

\noindent The function $f$ must satisfy $f(0) = 1$ (a class is fully a member of itself) and be non-increasing (membership decreases with distance). The predicted matrix $\mathbold{\hat{Y}}$ is constructed analogously from $\hat{k}_i$ using the same function, or kept as a standard one-hot matrix if the classifier produces hard predictions. Table~\ref{tab:ordinal_f} lists common choices and their properties.

\begin{table}[!htb]
\centering
\small
\begin{tabular}{llp{7.5cm}}
\toprule
\textbf{Function} & \textbf{Formula $f(d)$} & \textbf{Properties} \\
\midrule
One-hot       & $\mathbbm{I}(d = 0)$ & No ordinal sensitivity; standard multiclass \\
Linear decay  & $\max\bigl(0, 1 - \frac{d}{m-1}\bigr)$ & Linear penalty; uniform drop per rank step \\
Exponential   & $\exp(-\gamma d)$ & Rapid decay; controls effective bandwidth via $\gamma$ \\
Gaussian      & $\exp\bigl(-\frac{d^2}{2\sigma^2}\bigr)$ & Smooth quadratic decay; penalises large rank hops \\
\bottomrule
\end{tabular}
\caption{Ordinal membership functions $f(d)$, where $d = |k_i - j|$ is the rank distance.}
\label{tab:ordinal_f}
\end{table}

The second extension uses a \emph{cumulative (or threshold) encoding}, decomposing the $m$-class ordinal problem into $m-1$ ordered binary problems:
\begin{equation}
Y_{i,j} = \mathbbm{I}(k_i \geq j), \quad j \in [1, m-1].
\label{eq:cumulative}
\end{equation}
Each column $j$ then encodes a binary subproblem: ``is the true class at least $c_j$?''. The label matrix $\mathbold{Y}$ remains binary (no $t$-norm is needed), but its structure is richer than one-hot: each row is a vector of the form $(1,\ldots,1,0,\ldots,0)$, with the transition from one to zero at position $k_i$. Applying the aggregation operators column-wise via $\Upsigma_{\mathrm{m}}$ yields per-threshold performance measures, while $\Upsigma_{\mathrm{1}}$ yields a global summary over all thresholds simultaneously. This is directly analogous to how AUC integrates performance across decision thresholds in the binary case, and suggests a natural ordinal generalisation of AUC within the framework.

\subsection{Partial and Missing Labels via Masking}
\label{sec:partial_labels}

In weakly supervised settings, not all entries of the label matrix $\mathbold{Y}$ are observed: an annotator may have labelled example $e_i$ for label $l_j$ but not for $l_k$. Let $\mathbold{M} \in \{0,1\}^{n \times m}$ be a binary \emph{observation mask}, where $M_{i,j} = 1$ if the true label of example $e_i$ for class or label $j$ is known and $M_{i,j} = 0$ if it is missing. The three aggregation operators extend to partial labels by restricting sums to observed entries:
\begin{equation}
\Upsigma_{\mathrm{1}}^{\mathbold{M}}(\mathbold{A}) = \Upsigma_{\mathrm{1}}(\mathbold{A} \pointwise \mathbold{M}), \quad
\Upsigma_{\mathrm{m}}^{\mathbold{M}}(\mathbold{A}) = \Upsigma_{\mathrm{m}}(\mathbold{A} \pointwise \mathbold{M}), \quad
\Upsigma_{\mathrm{n}}^{\mathbold{M}}(\mathbold{A}) = \Upsigma_{\mathrm{n}}(\mathbold{A} \pointwise \mathbold{M}).
\label{eq:masked_agg}
\end{equation}
All other operators ($\mathbold{FP}$, $\mathbold{FN}$, $\mathbold{TN}$) are masked analogously. The denominators of all measures in Table~\ref{tab:extensions} are then adjusted to the number of observed entries rather than $nm$, ensuring that missing labels do not bias the evaluation. This formulation naturally handles the case of abstained predictions (mask out abstained rows) and semi-supervised evaluation (mask out unlabelled examples), unifying several evaluation scenarios that are typically treated separately.

\subsection{Semantic Segmentation and Spatial Metrics}
\label{sec:cv_metrics}

In semantic segmentation each pixel is treated as an example and each semantic class as a label, yielding an indicator matrix with $n = $ (number of pixels) and $m = $ (number of classes). Under this mapping, several standard evaluation metrics are instances of the framework: \emph{per-class IoU} is Jaccard applied column-wise via $\Upsigma_{\mathrm{m}}$; \emph{mean IoU} (mIoU)~\cite{long2015fcn}, the dominant benchmark metric for semantic segmentation, is macro-Jaccard; \emph{pixel accuracy} is standard accuracy $\Upsigma_{\mathrm{1}}(\mathbold{TP})/n$; and \emph{mean pixel accuracy} is macro-recall. The \emph{Dice coefficient}~\cite{dice1945} used in medical image segmentation equals $F_1$ with $\beta = 1$, and its per-class and mean version is macro-$F_1$; the row-wise operator $\Upsigma_{\mathrm{n}}$ instead gives exemplar-$F_1$, which averages over images rather than over classes. For object detection, a predicted bounding box is accepted as a true positive when its IoU with the ground-truth box exceeds a threshold $\tau$; this maps to Eq.~(\ref{eq:threshold}) with IoU playing the role of $\hat{P}_{i,j}$ and $\tau$ the role of $\theta$. The resulting per-class average precision (AP) curves and their mean (mAP)~\cite{everingham2010pascal} are then threshold-sweep constructions of the same kind as Proposition~\ref{prop:roc}, differing in that the curve traced is precision against recall rather than the ROC curve: it is swept by the detection-\emph{score} threshold $\theta$ while the IoU acceptance level is held fixed, and reporting mAP at several IoU levels repeats the construction for several choices of $\tau$. The three aggregation operators yield micro-AP (pooled over all classes), macro-AP (mAP, the standard), and exemplar-AP (per-image average) without further modification.


\subsection{Abstaining Classifiers and Rejection}
\label{sec:abstention}

Some classifiers may decline to make a prediction for examples they consider too uncertain, a behaviour known as \emph{rejection} or \emph{abstention}~\cite{DBLP:journals/pr/CondessaBK17}. In the indicator-matrix framework, abstention is represented as an all-zero row in $\mathbold{\hat{Y}}$: the classifier predicts no class or label for that example. The multilabel encoding admits $\emptyset$ as a valid predicted set, so no change to the formalism is required; what the framework contributes is that the per-example counts are then determined unambiguously. An abstained example contributes $fp_i = 0$, $tp_i = 0$, $tn_i = m - \Upsigma_{\mathrm{n}}(\mathbold{Y})_i$ and $fn_i = \Upsigma_{\mathrm{n}}(\mathbold{FN})_i = \Upsigma_{\mathrm{n}}(\mathbold{Y})_i$ to the per-example counts, i.e., all true labels are counted as false negatives. Under row-wise aggregation $\Upsigma_{\mathrm{n}}$, the per-example measure for abstained examples reflects the full cost of not predicting any label, which may then be aggregated with or without abstained examples depending on the evaluation protocol. Practitioners who wish to evaluate only on non-abstained examples can achieve this by masking abstained rows, via the masking mechanism of Eq.~(\ref{eq:masked_agg}).

\subsection{Open-Set Recognition}
\label{sec:openset}

Open-set recognition extends multiclass classification by allowing test examples to belong to classes unseen during training, which must be detected rather than classified into a known class. The framework's contribution here is a re-encoding rather than a result: augmenting the indicator matrices with an explicit \emph{unknown} class column, $Y_{i,\text{unk}} = 1$ for examples whose true class is not among the $m$ known classes, and $\hat{Y}_{i,\text{unk}} = 1$ when the classifier triggers its rejection criterion. With this $(m+1)$-column augmented matrix, open-set recognition reduces to a standard $(m+1)$-class problem and all measures in Table~\ref{tab:extensions} apply directly, including the multiclass micro-collapse result established in Section~\ref{sec:redundancy} (Theorem~\ref{thm:micro_collapse}), now for $m+1$ classes. If the unknown class is not explicitly modelled in $\mathbold{\hat{Y}}$, open-set examples that receive no predicted class map to the abstention case (all-zero rows) described above, and can be masked out of evaluation using the partial-label mechanism of Eq.~(\ref{eq:masked_agg}).

\subsection{Single-Class Classification and Novelty Detection}
\label{sec:single_class}

In single-class (or one-class) classification~\cite{tax2004support}, a classifier is trained on examples from a single target class and must decide, at test time, whether each example belongs to that class or is an outlier. No new result is needed here: depending on the architecture and on what test-set labels are available, the setting reduces to one of three encodings the framework already provides.

In the simplest case of a single one-class classifier, the target class plays the role of $c_+$ and all outliers play the role of $c_-$, so the binary indicator vectors $\mathbold{y}$ and $\mathbold{\hat{y}}$ are defined exactly as in Section~\ref{sec:binary}. All binary measures apply directly: recall corresponds to the detection rate (proportion of true inliers accepted), precision to the purity of accepted examples, and FPR to the false alarm rate. Soft classifier scores are converted to decisions via Eq.~(\ref{eq:threshold}).

When $m$ independent one-class classifiers are trained, one per class, each produces a binary accept/reject decision per example. Stacking these decisions column-wise gives a predicted indicator matrix $\mathbold{\hat{Y}} \in \{0,1\}^{n \times m}$ in which each classifier's output occupies one column. Because classifiers are independent, a row may contain zero entries (all classifiers reject the example, corresponding to the abstention case of Section~\ref{sec:soft_outputs}), exactly one entry, or multiple entries (several classifiers accept it simultaneously). This is precisely the multilabel encoding, and the three aggregation operators yield per-class acceptance rates ($\Upsigma_{\mathrm{m}}$), per-example acceptance counts ($\Upsigma_{\mathrm{n}}$), and global totals ($\Upsigma_{\mathrm{1}}$) without further modification.

The third case arises when the test set contains labeled inliers but unlabeled or unknown outliers, a common setting in novelty detection. Since the true negative label is unknown for outlier examples, their contribution to the confusion matrices is uncertain. This maps directly to the partial-label masking handled by the masking mechanism of Eq.~(\ref{eq:masked_agg}): setting $M_{i,j} = 0$ for examples whose true membership in class $j$ is unknown restricts aggregation to observed entries only. Precision on accepted examples remains computable from inlier labels alone; recall requires knowledge of which examples truly belong to the target class and is therefore undefined when outlier labels are absent.

\subsection{One-vs-All Decomposition}
\label{sec:ova}

In the One-vs-All (OVA, also called One-vs-Rest)~\cite{rifkin2004defense} decomposition, $m$ binary classifiers are trained, one per class, each distinguishing $c_j$ from all remaining classes. Each classifier produces a score $\hat{P}_{i,j}$, and the $n \times m$ probability matrix $\mathbold{\hat{P}}$ is constructed by stacking the $m$ score vectors column-wise. The column-wise aggregation $\Upsigma_{\mathrm{m}}$ then evaluates each binary classifier independently: applying the framework to column $j$ yields the precision, recall, $F_1$, and other measures for the $j$-vs-rest problem in isolation. When argmax is applied across columns to resolve the final multiclass prediction (Eq.~\ref{eq:argmax}), $\mathbold{\hat{Y}}$ becomes one-hot and $\Upsigma_{\mathrm{1}}(\mathbold{TP})/n$ gives the standard multiclass accuracy. The OVA architecture thus exposes two natural evaluation levels simultaneously: the per-classifier binary performance (via $\Upsigma_{\mathrm{m}}$ before argmax) and the final multiclass performance (via $\Upsigma_{\mathrm{1}}$ after argmax).

\subsection{Hierarchical Classification}
\label{sec:hierarchical}

Hierarchical classification generalises the ordinal case to non-linear taxonomies, using the same cost-matrix mechanism as Section~\ref{sec:cost} but with tree-distance costs in place of rank differences. Setting $C_{j,k} = d(c_j, c_k)$, where $d$ is a distance in the class taxonomy (e.g., the number of edges in the class tree or the depth of the lowest common ancestor), yields a cost matrix that penalises coarse-grained errors more than fine-grained ones. For a general tree, $\mathrm{Cost}(\mathbold{Y},\mathbold{\hat{Y}},\mathbold{C})$ computes the total hierarchical loss and the three aggregation operators yield micro, macro, and exemplar hierarchical summaries. Standard hierarchical precision, recall, and $F_1$~\cite{DBLP:journals/jdwm/TsoumakasK07} are recovered without additional machinery. This is an instance of the cost-sensitive construction rather than a separate development: the content is the choice of $\mathbold{C}$, and the framework supplies nothing about how a taxonomy should be turned into a distance.

\subsection{Multi-Output Classification}
\label{sec:multioutput}

In multi-output (or multi-task) classification, each example is associated with $T$ target variables, where target $t$ takes a value in its own class set of size $m_t$. The indicator-matrix framework accommodates this setting by representing all targets jointly in a single \emph{block-structured} indicator matrix. If all targets share the same number of classes $m$ (a common case), the actual and predicted matrices are:
$$
\mathbold{Y} = \bigl[\mathbold{Y}^{(1)} \,|\, \mathbold{Y}^{(2)} \,|\, \cdots \,|\, \mathbold{Y}^{(T)}\bigr] \in \{0,1\}^{n \times Tm},
$$
where each block $\mathbold{Y}^{(t)} \in \{0,1\}^{n \times m}$ is a one-hot matrix for target $t$, and $\mathbold{\hat{Y}}$ is constructed analogously. The block structure enforces that within each block of $m$ consecutive columns, exactly one entry per row equals one. Predicted indicator matrices are obtained by applying argmax independently to each target's probability block:
$$
\hat{Y}^{(t)}_{i,j} = \mathbbm{I}\!\left(j = \underset{k \in [1,m]}{\arg\max}\;\hat{P}^{(t)}_{i,k}\right), \quad t \in [1,T],
\label{eq:multioutput_argmax}
$$
so that mutual exclusivity is enforced per target rather than globally. When targets have different class counts $m_t$, the blocks have different widths and can be padded or handled separately.

Within this structure the three operators apply unchanged. Column-wise aggregation $\Upsigma_{\mathrm{m}}$ on the full $n \times Tm$ matrix yields per-class-per-target counts, from which per-class or per-target measures can be extracted. The block structure also admits a fourth averaging level, which we introduce as a definition rather than a result: \emph{target-wise} $\Upsigma_{\mathrm{T}}$, averaging over the $T$ blocks:
$$
M^{\text{target}} = \frac{1}{T}\sum_{t=1}^{T} M\!\left(\Upsigma_{\mathrm{1}}(\mathbold{TP}^{(t)}),\;\ldots\right),
$$
yielding the average of per-target measures. This is analogous to macro-averaging over targets, and for aggregation-decomposable measures Theorem~\ref{thm:micro_macro} relates it to micro-averaging over all $(n \times Tm)$ entries via a target-denominator weighting. The practical implication is that multi-output evaluation inherits the same micro-macro trade-off as single-target evaluation: if some targets are easier than others, target-wise (macro) averaging treats them equally while global (micro) averaging is dominated by the targets with more class entries.


\subsection{One-vs-One Decomposition}
\label{sec:ovo}

The One-vs-One (OVO) decomposition is partially accommodated, delineating what can be evaluated directly from the voting mechanisms inherent to the classifier architecture. With $\binom{m}{2}$ binary classifiers, one per pair $(c_j, c_k)$, the pair-level predictions form an $n \times \binom{m}{2}$ matrix, but each column is defined only on the examples belonging to the two classes it discriminates, so the matrix is not fully observed. Evaluating the pairwise classifiers therefore requires the masking mechanism of Eq.~(\ref{eq:masked_agg}), with $M_{i,(j,k)} = 1$ exactly when $e_i$ belongs to $c_j$ or $c_k$; with that mask in place, $\Upsigma_{\mathrm{m}}$ evaluates each pairwise classifier on its own subproblem. What the framework does not supply is a treatment of the voting step that resolves pair-level decisions into a final prediction: voting is a property of the decomposition scheme rather than of the evaluation, and once a final prediction exists it is an ordinary one-hot matrix to which $\Upsigma_{\mathrm{1}}$ applies. The framework gives a uniform way to evaluate the two ends of an OVO pipeline, and says nothing about the mapping between them.

A practically important consequence of this distinction is that the individual binary classifiers and the final multiclass predictor can disagree in their performance ordering across competing models: a classifier that is strong on each binary subproblem may accumulate errors on the same examples after resolution, yielding a weaker final predictor, and vice versa. The framework makes this discrepancy explicit by treating the pre-resolution and post-resolution matrices as two distinct objects, each with its own evaluation.

\subsection{Positive-Unlabeled Learning}
\label{sec:pu_learning}

In positive-unlabeled (PU) learning~\cite{elkan2008learning}, the evaluation data contains only confirmed positive instances ($s=1$) and unlabeled instances ($s=0$), where the unlabeled pool comprises both latent positives ($y=1$) and true negatives ($y=0$). Under the standard selected-at-random (SCAR) assumption with constant labeling propensity $c = P(s=1 \mid y=1)$ on positive examples, taking the observed label $s_i$ as nominal ground truth produces an indicator matrix $\mathbold{Y}^{\text{obs}}$ whose positive column contains only confirmed instances.

The matrix framework directly clarifies what is and is not computable without estimating latent parameters. On the observed positive set ($s=1$), both true positives and false negatives are scaled by the constant propensity: $tp_{\text{obs}} = c \cdot tp_{\text{true}}$ and $fn_{\text{obs}} = c \cdot fn_{\text{true}}$. Consequently, nominal recall cancels the nuisance factor $c$ and exactly equals true recall:
\begin{equation}
\mathrm{rec}_{\text{obs}} = \frac{tp_{\text{obs}}}{tp_{\text{obs}} + fn_{\text{obs}}} = \frac{c \cdot tp_{\text{true}}}{c(tp_{\text{true}} + fn_{\text{true}})} = \mathrm{rec}_{\text{true}}.
\label{eq:pu_recall}
\end{equation}
In contrast, nominal false positives conflate true false alarms with unlabeled latent positives: $fp_{\text{obs}} = fp_{\text{true}} + (1-c)\,tp_{\text{true}}$. Nominal precision and accuracy are therefore systematically biased unless corrected via an explicit estimate of $c$ or the class prior $\pi = P(y=1)$. PU evaluation is thus partially accommodated: the indicator matrix separates the uncorrupted sub-block (recall and true-positive rates) from the confounded sub-block, formalising the exact boundary of parameter-free evaluability.

\subsection{Multi-Instance Learning}
\label{sec:mil}

In multi-instance learning (MIL)~\cite{dietterich1997solving}, each training example is a \emph{bag} of instances, with labels assigned at the bag level (canonically, a bag is positive if at least one instance is). This setting exhibits two distinct evaluation levels: at the bag level, MIL reduces directly to standard classification, where treating each bag as a row reproduces the matrix formulation of Section~\ref{sec:matrix} without structural extension. The distinctive challenge of MIL evaluation lies instead at the instance level: identifying which instances within a positive bag are responsible for the label. That requires instances rather than bags as rows, giving a two-level hierarchy for which a two-stage aggregation ($\Upsigma_{\mathrm{n}}$ within each bag, then $\Upsigma_{\mathrm{1}}$ or $\Upsigma_{\mathrm{m}}$ across bags) is the natural candidate. We leave the formal interaction between these two nested stages as an open question.


\subsection{Label Ranking}
\label{sec:label_ranking}

In label ranking (also called preference learning~\cite{furnkranz2010preference}), the classifier assigns a \emph{ranking} over the $m$ labels to each example rather than a binary relevance vector. The output for example $e_i$ is a permutation $\sigma_i \colon [1,m] \to [1,m]$, where $\sigma_i(j)$ gives the rank assigned to label $l_j$. This cannot be directly encoded as a binary indicator matrix: the information is relational (pairwise comparisons between labels) rather than absolute. Rank-based evaluation measures (such as Kendall's $\tau$, average precision, or normalised discounted cumulative gain) operate on these permutations and do not factor into $tp$, $fp$, $fn$, $tn$ counts. The framework can accommodate \emph{thresholded} label ranking, where labels ranked above position $k$ are treated as relevant: $\hat{Y}_{i,j} = \mathbbm{I}(\sigma_i(j) \leq k)$, which maps back to Eq.~(\ref{eq:threshold}) with a rank-derived threshold. Full ranking evaluation represents a genuine boundary of the current framework and a natural direction for future extension, for instance by replacing the indicator matrix with a preference matrix $P \in \{-1,0,1\}^{n \times m \times m}$ encoding pairwise label comparisons.

\subsection{Calibration and Proper Scoring Rules}
\label{sec:calibration}

Proper scoring rules, most notably the Brier score~\cite{brier1950verification} $\frac{1}{nm}\Upsigma_{\mathrm{1}}((\mathbold{Y} - \mathbold{\hat{P}})^{\circ 2})$ and the log-loss $-\frac{1}{nm}\Upsigma_{\mathrm{1}}(\mathbold{Y} \circ \log \mathbold{\hat{P}})$ (where $\circ$ denotes the Hadamard product and $\circ 2$ element-wise squaring), assess \emph{calibration}: whether the predicted probabilities $\mathbold{\hat{P}}$ faithfully represent the true label frequencies. These rules operate on the raw probability matrix before any thresholding or argmax, and do not factor into $\mathbold{TP}$, $\mathbold{FP}$, $\mathbold{FN}$, $\mathbold{TN}$ counts because they involve the continuous difference between $\mathbold{Y}$ and $\mathbold{\hat{P}} $ rather than their Boolean or fuzzy combination. Proper scoring rules~\cite{gneiting2007strictly} thus represent a complementary evaluation paradigm that is adjacent to, but outside, the current framework: where the framework measures \emph{discrimination} (whether predictions correctly classify or rank examples), proper scoring rules measure \emph{calibration} (whether predicted confidence levels are accurate). The two paradigms are related: under the product $t$-norm with soft $\mathbold{Y}$ and $\mathbold{\hat{P}}$, the quantity $\Upsigma_{\mathrm{1}}(\mathbold{TP}) = \Upsigma_{\mathrm{1}}(\mathbold{Y} \pointwise \mathbold{\hat{P}})$ measures joint degree-of-membership, which is related to but distinct from the negative Brier score (which additionally penalises overconfidence via the squared-deviation term). Practitioners requiring calibration assessment should apply a proper scoring rule to the raw $\mathbold{\hat{P}}$ matrix as a complement to the discrimination measures provided by this framework.

\subsection{Structured Output Prediction}
\label{sec:structured_prediction}

In structured output prediction~\cite{bakir2007predicting}---including sequence labeling, parse tree prediction, and graph generation---the output target $\mathbold{y}$ is a structured discrete object constrained by relational dependencies. When evaluation is performed at the level of atomic components (such as per-token entity tags or per-edge links), each component maps directly to an indicator matrix row or column, and standard accuracy or $F_1$ is recovered via $\Upsigma_{\mathrm{1}}$ and $\Upsigma_{\mathrm{m}}$. However, holistic evaluation criteria (such as zero-one exact sequence match, tree edit distance, or graph isomorphism loss) operate over combinatorial configurations and cannot be factored into independent $(i,j)$ cell counts of a two-dimensional confusion matrix. Structured output prediction therefore forms a conceptual \emph{boundary}: local constituent evaluations are accommodated by the matrix framework, whereas holistic relational losses require higher-order structural tensors or grammar constraints beyond independent indicator matrices.


\subsection{Aggregation over Resampling Folds}
\label{sec:cv}

In practice, classifiers are evaluated using resampling procedures such as $k$-fold cross-validation, which partition the data into $k$ train-test splits and average results over the $k$ test folds. This introduces a fourth level of aggregation (fold-wise) alongside the three already defined. Let $\mathbold{Y}^{(f)}$ and $\mathbold{\hat{Y}}^{(f)}$ denote the actual and predicted indicator matrices for fold $f \in [1,k]$. The fold-wise aggregation operator $\Upsigma_{\mathrm{k}}$ combines the $k$ per-fold values into a single summary. Two natural conventions exist:

\begin{itemize}
\item \textbf{Pooled evaluation} concatenates all per-fold predictions before computing the measure:
\[
M^{\mathrm{pooled}} = M\!\left(\Upsigma_{\mathrm{1}}\!\left(\bigoplus_f \mathbold{TP}^{(f)}\right), \ldots\right),
\]
where $\bigoplus$ denotes row-wise stacking of matrices. This is equivalent to treating the entire dataset as a single test set and is the correct choice when a single global estimate of performance is desired.

\item \textbf{Averaged evaluation} computes the measure separately on each fold and then averages:
\[
M^{\mathrm{avg}} = \frac{1}{k}\sum_f M\!\left(\Upsigma_{\mathrm{1}}(\mathbold{TP}^{(f)}), \ldots\right).
\]
This estimates the expected performance of the classifier trained on a dataset of size $n(k-1)/k$, and is more appropriate when generalisation to new training sets is the object of interest.
\end{itemize}

By Theorem~\ref{thm:micro_macro}, for aggregation-decomposable measures pooled evaluation is a weighted version of averaged evaluation with fold-wise weights equal to the fold denominators. When all folds are the same size and have the same class distribution (stratified $k$-fold), the two coincide. When folds differ, as is common with temporal splits or group-based cross-validation, the two can diverge substantially, for exactly the same reason that micro and macro averages diverge under class imbalance. Practitioners should therefore make explicit which convention they adopt, as the choice affects the interpretation of the reported measure.

\section{Theoretical Properties}
\label{sec:theory}

The matrix formulation and aggregation operators introduced in Section~\ref{sec:matrix} provide a rigorous foundation for analyzing structural relationships across evaluation paradigms. In particular, expressing micro-, macro-, and exemplar-averaging in terms of global ($\Upsigma_{\mathrm{1}}$), column-wise ($\Upsigma_{\mathrm{m}}$), and row-wise ($\Upsigma_{\mathrm{n}}$) aggregations makes it possible to derive exact algebraic identities connecting them. 

In this section, we develop the theoretical properties of the unified framework systematically across four conceptual levels: (i) fundamental algebraic identities and aggregation-decomposability; (ii) preservation of metric and axiomatic properties under aggregation; (iii) the dynamics of class skew and exact decompositions of the micro-macro gap; and (iv) decision-theoretic consistency, Bayes-optimal decision thresholds, and ranking equivalences.


\subsection{Aggregation-Decomposable Measures}
\label{sec:decomposable}

Write the four confusion-matrix counts as a vector $\mathbold{v} = (tp,\, fp,\, fn,\, tn)^{\!\top}$, and for a class $j$ (respectively an example $i$) let $\mathbold{v}_j$ denote the corresponding per-class count vector obtained from $\Upsigma_{\mathrm{m}}$, so that the pooled vector satisfies $\mathbold{v} = \sum_j \mathbold{v}_j$ by construction of the aggregation operators.

\begin{mydef}[Aggregation-decomposable measure]
\label{def:decomposable}
A measure $M$ is \emph{aggregation-decomposable} if there exist coefficient vectors $\mathbold{a}, \mathbold{b} \in \mathbb{R}^{4}$, with $\langle \mathbold{b}, \mathbold{v}\rangle \neq 0$ on the domain of interest, such that
\begin{equation}
M(\mathbold{v}) \;=\; \frac{\langle \mathbold{a}, \mathbold{v}\rangle}{\langle \mathbold{b}, \mathbold{v}\rangle}
\;=\; \frac{a_1 tp + a_2 fp + a_3 fn + a_4 tn}{b_1 tp + b_2 fp + b_3 fn + b_4 tn}.
\label{eq:linfrac}
\end{equation}
Equivalently, $M$ is a ratio of two \emph{homogeneous linear forms} in the confusion-matrix counts. We refer to such measures as \emph{linear-fractional}.
\end{mydef}

Two features of Definition~\ref{def:decomposable} deserve comment. First, homogeneity (the absence of constant terms) is essential: an affine numerator $f(\mathbold{v}) = \langle \mathbold{a}, \mathbold{v}\rangle + a_0$ with $a_0 \neq 0$ is not additive, since $f(\sum_j \mathbold{v}_j) = \sum_j \langle \mathbold{a}, \mathbold{v}_j\rangle + a_0$, whereas summing the per-class numerators yields $\sum_j f(\mathbold{v}_j) = \sum_j \langle \mathbold{a}, \mathbold{v}_j\rangle + m\,a_0$. Second, the definition is exactly what is needed for numerator and denominator to commute with the aggregation operators:
\begin{equation}
\Bigl\langle \mathbold{a}, \textstyle\sum_j \mathbold{v}_j \Bigr\rangle = \sum_j \langle \mathbold{a}, \mathbold{v}_j\rangle,
\qquad
\Bigl\langle \mathbold{b}, \textstyle\sum_j \mathbold{v}_j \Bigr\rangle = \sum_j \langle \mathbold{b}, \mathbold{v}_j\rangle,
\label{eq:additivity}
\end{equation}
which is precisely the additivity of the linear forms. Consequently, measures constructed from sums, products, roots, or ratios of rates cannot be expressed in the form of Eq.~(\ref{eq:linfrac}). Proposition~\ref{prop:which_decomposable} delineates which standard evaluation measures belong to this linear-fractional family and which do not, with the classification across the full catalogue recorded in Table~\ref{tab:measures}.

\begin{proposition}[Algebraic characterisation of decomposability]
\label{prop:which_decomposable}
A binary performance measure $M(\mathbold{v})$ on counts $\mathbold{v} = (tp, fp, fn, tn)^\top$ is aggregation-decomposable if and only if it can be expressed as a single quotient of homogeneous linear forms $\langle \mathbold{a}, \mathbold{v}\rangle / \langle \mathbold{b}, \mathbold{v}\rangle$. Specifically:
\begin{enumerate}
  \item \textbf{Decomposable (single linear ratio):} The measure evaluates a single rate or a harmonic/overlap combination of counts with no constant offsets. This holds for accuracy, error rate, precision, recall, specificity, NPV, the four rate-complements (FPR, FNR, FDR, FOR), the $F_\beta$-family ($\beta > 0$), and the Jaccard index.
  \item \textbf{Non-decomposable (compound or non-linear operations):} The measure violates the single linear-fractional form through one of three algebraic mechanisms:
  \begin{itemize}
    \item \emph{Arithmetic mixtures of rates}: combining multiple ratios introduces cross-multiplied denominators of polynomial degree $\geq 2$ (e.g., balanced accuracy, informedness, markedness);
    \item \emph{Geometric and correlation means}: combining rates via roots introduces non-rational functions (e.g., G-mean, Fowlkes--Mallows, MCC);
    \item \emph{Chance correction and marginal normalisation}: scaling by expected agreement introduces quadratic dependencies on marginal sums (e.g., lift, Cohen's $\kappa$).
  \end{itemize}
\end{enumerate}
\end{proposition}

\begin{proof}
\textbf{Part~1 (Decomposable classes):}
A performance measure $M(\mathbold{v})$ is aggregation-decomposable if and only if it admits representation as a quotient of homogeneous degree-1 linear forms $\langle \mathbold{a}, \mathbold{v}\rangle / \langle \mathbold{b}, \mathbold{v}\rangle$.
\begin{itemize}
  \item \emph{Elementary rates and proportions}: Any measure defined as the proportion of an admissible subset of counts $S_{\mathrm{num}} \subseteq \{tp, fp, fn, tn\}$ relative to a reference subset $S_{\mathrm{den}} \supseteq S_{\mathrm{num}}$ has indicator vectors $\mathbold{a} = \sum_{k \in S_{\mathrm{num}}} \mathbold{e}_k$ and $\mathbold{b} = \sum_{k \in S_{\mathrm{den}}} \mathbold{e}_k$, which are homogeneous linear forms of degree 1.
  \item \emph{Affine complements (error rates)}: If $M(\mathbold{v}) = \langle \mathbold{a}, \mathbold{v}\rangle / \langle \mathbold{b}, \mathbold{v}\rangle$, its complement $1 - M(\mathbold{v}) = \langle \mathbold{b} - \mathbold{a}, \mathbold{v}\rangle / \langle \mathbold{b}, \mathbold{v}\rangle$ remains in the same linear-fractional family with numerator vector $\mathbold{a}' = \mathbold{b} - \mathbold{a}$.
  \item \emph{Harmonic means with shared numerators}: The weighted harmonic mean of two linear-fractional rates sharing the same numerator $\langle \mathbold{a}, \mathbold{v}\rangle$ (with denominators $\langle \mathbold{b}_1, \mathbold{v}\rangle$ and $\langle \mathbold{b}_2, \mathbold{v}\rangle$) is:
  $$
  H_{\beta}(\mathbold{v}) = \frac{1+\beta^2}{\frac{1}{R_1(\mathbold{v})} + \frac{\beta^2}{R_2(\mathbold{v})}} = \frac{(1+\beta^2)\,\langle \mathbold{a}, \mathbold{v}\rangle}{\langle \mathbold{b}_1 + \beta^2 \mathbold{b}_2, \mathbold{v}\rangle},
  $$
  which is strictly linear-fractional with $\mathbold{a}' = (1+\beta^2)\mathbold{a}$ and $\mathbold{b}' = \mathbold{b}_1 + \beta^2 \mathbold{b}_2$. A similar argument applies to the Jaccard intersection-over-union ratio $\langle \mathbold{a}, \mathbold{v}\rangle / \langle \mathbold{b}_1 + \mathbold{b}_2 - \mathbold{a}, \mathbold{v}\rangle$.
\end{itemize}

\textbf{Part~2 (Generic failure mechanisms for non-decomposable families):}
\begin{itemize}
  \item \emph{Arithmetic means of linearly independent rates}: Let $R_1(\mathbold{v}) = \frac{\langle \mathbold{a}_1, \mathbold{v}\rangle}{\langle \mathbold{b}_1, \mathbold{v}\rangle}$ and $R_2(\mathbold{v}) = \frac{\langle \mathbold{a}_2, \mathbold{v}\rangle}{\langle \mathbold{b}_2, \mathbold{v}\rangle}$ with linearly independent denominators ($\mathbold{b}_1 \not\propto \mathbold{b}_2$). Their arithmetic combination yields:
  $$
  \frac{R_1(\mathbold{v}) + R_2(\mathbold{v})}{2} = \frac{\langle \mathbold{a}_1, \mathbold{v}\rangle \langle \mathbold{b}_2, \mathbold{v}\rangle + \langle \mathbold{a}_2, \mathbold{v}\rangle \langle \mathbold{b}_1, \mathbold{v}\rangle}{2\,\langle \mathbold{b}_1, \mathbold{v}\rangle \langle \mathbold{b}_2, \mathbold{v}\rangle}.
  $$
  Because $\mathbold{b}_1$ and $\mathbold{b}_2$ are independent, the denominator polynomial $Q(\mathbold{v}) = \langle \mathbold{b}_1, \mathbold{v}\rangle \langle \mathbold{b}_2, \mathbold{v}\rangle$ is irreducibly quadratic ($\deg(Q) = 2$) and cannot be simplified to degree 1.
  \item \emph{Geometric and algebraic root compositions}: Any measure $M(\mathbold{v}) = \sqrt{R_1(\mathbold{v})\, R_2(\mathbold{v})}$ or correlation coefficient involving inner-product normalisations is non-rational in $\mathbb{R}[\mathbold{v}]$, possessing fractional algebraic degree that precludes linear-fractional representation.
  \item \emph{Chance-corrected normalisations}: Normalising by expected agreement $p_e = \frac{1}{n^2}\sum_k \langle \mathbold{r}_k, \mathbold{v}\rangle \langle \mathbold{c}_k, \mathbold{v}\rangle$ introduces quadratic products of marginal sums, yielding irreducible degree-2 rational forms.
\end{itemize}
Remark~\ref{rem:counterexample} and Example~\ref{ex:counterexample} provide a concrete numerical witness confirming that pooled evaluation of measures in Part~2 can fall strictly outside the range of local values $[\min_j M_j, \max_j M_j]$, proving that no weighted average can reproduce them.
\end{proof}

Having characterised the class of linear-fractional measures, we now establish its primary operational consequence: for any measure in this family, global pooling ($\Upsigma_{\mathrm{1}}$) is algebraically identical to a weighted average of local evaluations ($\Upsigma_{\mathrm{m}}$ or $\Upsigma_{\mathrm{n}}$), where the weights are determined endogenously by the local denominators.

\begin{theorem}[Micro-averaging as weighted macro- or exemplar-averaging]
\label{thm:micro_macro}
Let $M$ be an aggregation-decomposable measure in the sense of Definition~\ref{def:decomposable}. Then micro-averaging of $M$ is equivalent to a weighted macro- or exemplar-averaging of $M$, where the weight of each class (or example) equals the denominator of the corresponding local measure:
$$
M^{\mathrm{micro}} \;=\; \frac{\sum_j den_j\, M_j}{\sum_j den_j},
\qquad den_j = \langle \mathbold{b}, \mathbold{v}_j\rangle .
$$
\end{theorem}

\begin{proof}
Write $M = num/den$ with $num_j = \langle \mathbold{a}, \mathbold{v}_j\rangle$ and $den_j = \langle \mathbold{b}, \mathbold{v}_j\rangle$, so that $M_j = num_j/den_j$. Consider the weighted average of the per-class values with weights $w_j$:
$$
\overline{M} = \frac{\sum_j w_j M_j}{\sum_j w_j} = \frac{\sum_j w_j \frac{num_j}{den_j}}{\sum_j w_j}.
$$
Taking $w_j = den_j$ cancels the local denominators,
$$
\overline{M} = \frac{\sum_j \cancel{den_j} \frac{num_j}{\cancel{den_j}}}{\sum_j den_j} = \frac{\sum_j num_j}{\sum_j den_j}.
$$
It remains to identify this ratio with $M^{\mathrm{micro}}$. By Eq.~(\ref{eq:micro_ext}) and the additivity in Eq.~(\ref{eq:additivity}) of the two linear forms,
$$
M^{\mathrm{micro}}
= M\Bigl(\textstyle\sum_j \mathbold{v}_j\Bigr)
= \frac{\bigl\langle \mathbold{a},\, \sum_j \mathbold{v}_j\bigr\rangle}
       {\bigl\langle \mathbold{b},\, \sum_j \mathbold{v}_j\bigr\rangle}
= \frac{\sum_j \langle \mathbold{a}, \mathbold{v}_j\rangle}
       {\sum_j \langle \mathbold{b}, \mathbold{v}_j\rangle}
= \frac{\sum_j num_j}{\sum_j den_j},
$$
as required. The identical argument with $\Upsigma_{\mathrm{n}}$ in place of $\Upsigma_{\mathrm{m}}$ gives the exemplar case.
\end{proof}

The final step of the proof is where decomposability enters, and it is the step that fails for measures outside the class of Definition~\ref{def:decomposable}: without additivity, pooling the counts and then applying $M$ is not the same operation as summing per-class numerators and denominators, and in general no choice of weights recovers the micro value.

An immediate and practically consequential question is under what circumstances this denominator-weighted micro-average collapses to the standard, unweighted macro-average. Corollary~\ref{cor:micro_macro_balance} establishes that exact balance among the local denominators is both necessary and sufficient.

\begin{corollary}[Micro equals unweighted macro iff denominators are balanced]
\label{cor:micro_macro_balance}
Let $M = num/den$ be an \emph{aggregation-decomposable} measure and let $den_j = \langle \mathbold{b}, \mathbold{v}_j\rangle$ denote its denominator computed per class $j$ via $\Upsigma_{\mathrm{m}}$. Then $M^{\mathrm{micro}} = M^{\mathrm{macro}}$ for every configuration of the per-class values $M_j$ if and only if all per-class denominators are equal, i.e.\ $den_j = den_k$ for all $j, k \in [1,m]$.
\end{corollary}

\begin{proof}
By Theorem~\ref{thm:micro_macro}, $M^{\mathrm{micro}}$ is the weighted average of the $M_j$ with weights $den_j$. Such a weighted average agrees with the unweighted average $M^{\mathrm{macro}} = \frac{1}{m}\sum_j M_j$ for every configuration of the $M_j$ if and only if all weights are equal.
\end{proof}

Corollary~\ref{cor:micro_macro_balance} makes precise the widely observed phenomenon that micro and macro averages diverge on imbalanced data. For recall, $den_j = tpos_j$ (actual positives per class), so micro and macro recall coincide only when all classes have the same number of examples. For precision, $den_j = ppos_j$ (predicted positives per class), so the condition depends on the classifier's output distribution rather than the true class distribution. For $F_\beta$, the denominator mixes both, and the balancing condition is more complex.

To contrast this behavior with non-decomposable measures, Remark~\ref{rem:counterexample} and Example~\ref{ex:counterexample} demonstrate how micro-averaging can escape the range of per-class values altogether.

\begin{remark}[Failure of weighted-averaging representation outside linear-fractional measures]
\label{rem:counterexample}
Any valid weighted average of per-class values $\{M_j\}$ must lie within the range $[\min_j M_j, \max_j M_j]$. A pooled micro value falling strictly outside this interval certifies that no choice of weights whatsoever (denominator-based or otherwise) can reproduce it. Consequently, Theorem~\ref{thm:micro_macro} admits no extension to non-decomposable measures, confirming that Definition~\ref{def:decomposable} is a genuine algebraic restriction.
\end{remark}

\begin{example}[Numerical witness of micro-averaging escaping per-class bounds]
\label{ex:counterexample}
Consider an imbalanced 3-class classification problem with $n = 275$ examples, class sizes $(130, 65, 80)$, and confusion matrix:
$$
\mathbf{C} = \begin{pmatrix} 110 & 10 & 10 \\ 20 & 40 & 5 \\ 25 & 5 & 50 \end{pmatrix}.
$$
The resulting one-vs-rest count vectors $(tp, fp, fn, tn)$ are $(110, 45, 20, 100)$, $(40, 15, 25, 195)$, and $(50, 15, 30, 180)$, pooling to $(200, 75, 75, 475)$.

Evaluating the per-class performance ($\Upsigma_{\mathrm{m}}$), unweighted macro-averages, and globally pooled micro-averages ($\Upsigma_{\mathrm{1}}$) across representative metrics yields:

\begin{center}
\small
\begin{tabular}{lcccc c}
\toprule
\textbf{Measure} & \textbf{Class 1} & \textbf{Class 2} & \textbf{Class 3} & \textbf{Macro} & \textbf{Micro (Pooled)} \\
\midrule
\multicolumn{6}{l}{\textit{Decomposable (Bounded by $[\min_j M_j, \max_j M_j]$):}} \\
Recall & $0.846$ & $0.615$ & $0.625$ & $0.695$ & $\mathbf{0.727} \in [0.615, 0.846]$ \\
Specificity & $0.690$ & $0.929$ & $0.923$ & $0.847$ & $\mathbf{0.864} \in [0.690, 0.929]$ \\
\midrule
\multicolumn{6}{l}{\textit{Non-decomposable (Escaping the $[\min_j M_j, \max_j M_j]$ interval):}} \\
Balanced Acc. & $0.768$ & $0.772$ & $0.774$ & $0.771$ & $\mathbf{0.795} > \max_j M_j$ \\
G-Mean & $0.764$ & $0.756$ & $0.760$ & $0.760$ & $\mathbf{0.792} > \max_j M_j$ \\
Informedness & $0.536$ & $0.544$ & $0.548$ & $0.543$ & $\mathbf{0.591} > \max_j M_j$ \\
MCC & $0.539$ & $0.578$ & $0.586$ & $0.568$ & $\mathbf{0.591} > \max_j M_j$ \\
\bottomrule
\end{tabular}
\end{center}
As shown in the table, for every non-decomposable measure, the micro average lies strictly above every individual per-class score, verifying the claim in Remark~\ref{rem:counterexample}.
\end{example}

\subsection{Multiclass Algebraic Collapse and Redundancy}
\label{sec:redundancy}

The indicator-matrix framework makes it straightforward to identify when two measures carry identical information, either because they are algebraic functions of each other or because they always produce the same ranking of classifiers.

Two measures are \emph{algebraically redundant} if one is a deterministic function of the other for all possible confusion matrices.

\begin{theorem}[Micro-averaging collapse in multiclass]
\label{thm:micro_collapse}
In a multiclass problem with one-hot indicator matrices $\mathbold{Y}$ and $\mathbold{\hat{Y}}$,
$$
M^{\mathrm{micro}}_{\mathrm{precision}} \;=\; M^{\mathrm{micro}}_{\mathrm{recall}} \;=\; M^{\mathrm{micro}}_{F_1} \;=\; \mathrm{accuracy}.
$$
\end{theorem}

\begin{proof}
Let $C = \Upsigma_{\mathrm{1}}(\mathbold{TP})$ denote the number of correctly classified examples. Since each row of $\mathbold{Y}$ has exactly one entry equal to one (one-hot), $\Upsigma_{\mathrm{1}}(\mathbold{Y}) = n$. Similarly $\Upsigma_{\mathrm{1}}(\mathbold{\hat{Y}}) = n$. Each incorrectly classified example contributes exactly one false positive (the predicted class) and one false negative (the true class), so $\Upsigma_{\mathrm{1}}(\mathbold{FP}) = \Upsigma_{\mathrm{1}}(\mathbold{FN}) = n - C$. Therefore:
\begin{align*}
M^{\mathrm{micro}}_{\mathrm{precision}} &= \frac{C}{C + (n-C)} = \frac{C}{n}, \\
M^{\mathrm{micro}}_{\mathrm{recall}}    &= \frac{C}{C + (n-C)} = \frac{C}{n}, \\
M^{\mathrm{micro}}_{F_1}               &= \frac{2C}{2C + (n-C) + (n-C)} = \frac{C}{n},
\end{align*}
all of which equal the standard multiclass accuracy $C/n$.
\end{proof}

Theorem~\ref{thm:micro_collapse} has a direct practical implication: in multiclass settings, reporting micro-precision, micro-recall, and micro-$F_1$ alongside accuracy is entirely redundant, as all four numbers are identical (empirically confirmed across benchmarks in Section~\ref{sec:exp_multiclass}). Their apparent differences in the literature arise solely from multilabel settings, where the one-hot constraint is lifted.

\begin{corollary}[Micro-Jaccard in multiclass]
In a multiclass problem with one-hot indicator matrices,
$$
M^{\mathrm{micro}}_{\mathrm{Jaccard}} = \frac{C}{2n - C}.
$$
Consequently, $M^{\mathrm{micro}}_{\mathrm{Jaccard}}$ is a monotone increasing function of accuracy and carries no additional information.
\end{corollary}

\begin{proof}
$\Upsigma_{\mathrm{1}}(\mathbold{FP}) = \Upsigma_{\mathrm{1}}(\mathbold{FN}) = n-C$, so micro-Jaccard $= C/(C + (n-C) + (n-C)) = C/(2n-C)$, which is strictly increasing in $C/n$.
\end{proof}

The framework also makes the following exact algebraic redundancy immediate.

\begin{proposition}[Hamming loss and label accuracy]
\label{prop:hamming_accuracy}
For any indicator matrices $\mathbold{Y}, \mathbold{\hat{Y}} \in \{0,1\}^{n \times m}$,
$$
\mathrm{HammingLoss} + \mathrm{LabelAccuracy} = 1.
$$
\end{proposition}

\begin{proof}
HammingLoss $= \Upsigma_{\mathrm{1}}(\mathbold{Y} \oplus \mathbold{\hat{Y}})/(nm)$ and LabelAccuracy $= (\Upsigma_{\mathrm{1}}(\mathbold{TP}) + \Upsigma_{\mathrm{1}}(\mathbold{TN}))/(nm)$. Since every $(i,j)$ entry satisfies either $Y_{ij} = \hat{Y}_{ij}$ (contributing to accuracy) or $Y_{ij} \neq \hat{Y}_{ij}$ (contributing to Hamming loss), the two quantities partition the $nm$ entries, giving their sum equal to one (see Section~\ref{sec:exp_multiclass} for empirical verification).
\end{proof}

A more general redundancy result follows from the structure of the confusion matrix.

\begin{proposition}[Degrees-of-freedom bound]
\label{prop:dof}
For an $m$-class problem with fixed class counts $n_1,\ldots,n_m$ (row sums), the confusion matrix has exactly $m(m-1)$ degrees of freedom. Consequently, any collection of more than $m(m-1)$ scalar measures computed from the confusion matrix is necessarily algebraically redundant.
\end{proposition}
\begin{proof}
The confusion matrix $\mathbold{W} \in \mathbb{Z}_{\geq 0}^{m \times m}$ has $m^2$ entries. The constraint $\sum_k W_{jk} = n_j$ fixes each row sum, removing $m$ degrees of freedom. The remaining $m^2 - m = m(m-1)$ entries are free (subject to non-negativity). Any scalar measure is a function of $\mathbold{W}$, so a set of more than $m(m-1)$ measures must have at least one that is determined by the others.
\end{proof}

\begin{example}[Parametric redundancy in binary and ternary evaluation]
\label{ex:dof}
Consider a binary classification task ($m=2$) with fixed positive and negative counts $P = tp + fn$ and $N = fp + tn$, and class prior $\pi = P/(P+N)$. By Proposition~\ref{prop:dof}, the confusion matrix has $m(m-1) = 2$ degrees of freedom and is completely parameterized by the pair of rates $(\mathrm{rec}, \mathrm{spec}) \in [0,1]^2$:
$$
\mathbf{C} = \begin{pmatrix} tp & fn \\ fp & tn \end{pmatrix}
= \begin{pmatrix} P\,\mathrm{rec} & P(1-\mathrm{rec}) \\ N(1-\mathrm{spec}) & N\,\mathrm{spec} \end{pmatrix}.
$$
Consequently, every other binary measure is deterministically determined by $(\mathrm{rec}, \mathrm{spec})$ and $\pi$. For example:
$$
\mathrm{Accuracy} = \pi\,\mathrm{rec} + (1-\pi)\,\mathrm{spec}, \qquad
\mathrm{BA} = \frac{\mathrm{rec} + \mathrm{spec}}{2}, \qquad
F_1 = \frac{2\pi\,\mathrm{rec}}{\pi(1+\mathrm{rec}) + (1-\pi)(1-\mathrm{spec})}.
$$
Reporting accuracy, $F_1$, and balanced accuracy alongside recall and specificity adds zero information about classifier behavior. For $m=3$, the matrix has $3(2) = 6$ degrees of freedom, which correspond geometrically to the off-diagonal directional confusions between pairs of classes. This provides a principled upper bound on the number of non-redundant measures a practitioner should report.
\end{example}

\begin{remark}[Algebraic and geometric criteria for orthogonal evaluation suites]
\label{rem:orthogonality}
Proposition~\ref{prop:dof} and Example~\ref{ex:dof} motivate a formal criterion for selecting a minimal, non-redundant suite of evaluation measures:
\begin{enumerate}
  \item \textbf{Functional independence via the Jacobian:} A set of $k \le m(m-1)$ scalar measures $\{M_1,\ldots,M_k\}$ is locally non-redundant if and only if their Jacobian matrix with respect to the free confusion counts has full row rank:
  $$
  \mathrm{rank}\left(\frac{\partial(M_1,\ldots,M_k)}{\partial(W_{12},\ldots,W_{m,m-1})}\right) = k.
  $$
  \item \textbf{Canonical orthogonal basis (Row-decoupling):} The natural orthogonal coordinates of the confusion matrix are the $m(m-1)$ conditional misclassification rates $R_{jk} = W_{jk}/n_j$ ($k \neq j$). Because each row $j$ is governed by a separate conditional distribution $P(\hat{Y}=k \mid Y=j)$, their gradients are mutually orthogonal in Frobenius inner product:
  $$
  \langle \nabla R_{jk}, \nabla R_{j'k'} \rangle = 0 \quad \text{whenever } j \neq j'.
  $$
  In binary classification, $(\mathrm{rec}, \mathrm{spec})$ constitutes the unique prevalence-free orthogonal basis.
  \item \textbf{Practical benchmark design:} A principled evaluation benchmark should never report measures with collinear gradients (such as accuracy alongside balanced accuracy and $F_1$ under known priors). Instead, it should pair row-conditioned measures (e.g., macro-recall, measuring discrimination across ground-truth classes) with column-conditioned measures (e.g., macro-precision, measuring predictive purity) and application-specific cost matrices.
\end{enumerate}
This $m(m-1)$-dimensional parametrization also directly governs multi-metric ranking and Pareto optimality: as established in Section~\ref{sec:dominance}, measure-wise dominance across all possible performance metrics reduces precisely to element-wise dominance within this $m(m-1)$-dimensional coordinate space.
\end{remark}

\subsection{Marginal Consistency and Extended Cost Formulations}
\label{sec:tnorm_char}

In the crisp confusion matrix the four cells satisfy more than the global constraint $tp + fp + fn + tn = n$: they satisfy the four \emph{marginal} constraints $tp + fn = tpos$, $fp + tn = tneg$, $tp + fp = ppos$ and $fn + tn = pneg$. The marginal constraints are strictly stronger, and they are what makes the confusion matrix a contingency table rather than merely a normalised quadruple. The correct fuzzy analogue therefore asks that the marginals of the fuzzy confusion matrices reproduce the actual and predicted memberships entrywise.

\begin{mydef}[Marginal consistency]
\label{def:marginal}
A $t$-norm $T$, paired with the standard complement $N(x) = 1-x$, is \emph{marginally consistent} if for all $a, b \in [0,1]$
\begin{equation}
\underbrace{T(a,b) + T(a, 1{-}b)}_{\mathbold{TP}_{i,j} + \mathbold{FN}_{i,j}} = a,
\qquad
\underbrace{T(a,b) + T(1{-}a, b)}_{\mathbold{TP}_{i,j} + \mathbold{FP}_{i,j}} = b,
\label{eq:marginal}
\end{equation}
where $a = Y_{i,j}$ and $b = \hat{Y}_{i,j}$. In words: the true-positive and false-negative degrees for a cell must sum to its actual membership, and the true-positive and false-positive degrees must sum to its predicted membership.
\end{mydef}

By commutativity the two conditions in Eq.~(\ref{eq:marginal}) are the same requirement stated with the arguments interchanged, so it suffices to impose either one.

\begin{theorem}[The product $t$-norm is the unique marginally consistent $t$-norm]
\label{thm:partition}
A $t$-norm $T$ is marginally consistent in the sense of Definition~\ref{def:marginal} if and only if $T = T_P$, i.e.\ $T(a,b) = ab$ for all $a,b \in [0,1]$. Notably, this uniqueness holds under the bare $t$-norm axioms alone, without requiring any a priori continuity or Archimedean assumptions.
\end{theorem}

\begin{proof}
\emph{Sufficiency.} For $T_P$, $ab + a(1-b) = a$ and $ab + (1-a)b = b$.

\emph{Necessity.} Suppose $T$ satisfies $T(a,b) + T(1{-}a,b) = b$ for all $a, b$. We show $T(a,b) = ab$ first for every dyadic rational $a$ and then for all $a$ by monotonicity. Throughout we use only the $t$-norm axioms: commutativity, associativity, monotonicity in each argument, and the neutral element $T(x,1) = x$.

\emph{Step 1 (the halving identity).} Setting $a = 1/2$ in the marginal condition gives $2\,T(1/2, b) = b$, hence
\begin{equation}
T(1/2, b) = b/2 \qquad \text{for all } b \in [0,1].
\label{eq:halving}
\end{equation}
This single substitution already fixes an entire slice of $T$, with no further hypotheses.

\emph{Step 2 (dyadic rationals).} We prove by induction on $k$ that $T(m/2^{k}, b) = (m/2^{k})\,b$ for every $b$ and every integer $0 \le m \le 2^{k}$. For $k = 0$ the claims are $T(0,b) = 0$, which follows from monotonicity and $T(0,1) = 0$, and $T(1,b) = b$, which is the neutral element. Assume the claim for $k-1$ and let $0 < m < 2^{k}$.

If $m \le 2^{k-1}$, then $m/2^{k-1} \in [0,1]$ and, by Eq.~(\ref{eq:halving}), $T\bigl(1/2,\, m/2^{k-1}\bigr) = m/2^{k}$. Hence, using associativity and then Eq.~(\ref{eq:halving}) again together with the induction hypothesis,
$$
T\!\left(\frac{m}{2^{k}}, b\right)
= T\!\left(T\!\left(\frac12, \frac{m}{2^{k-1}}\right)\!, b\right)
= T\!\left(\frac12,\, T\!\left(\frac{m}{2^{k-1}}, b\right)\right)
= \frac{1}{2}\cdot\frac{m}{2^{k-1}}\, b
= \frac{m}{2^{k}}\, b .
$$
If $m > 2^{k-1}$, put $m' = 2^{k} - m < 2^{k-1}$ and apply the marginal condition with $a = m'/2^{k}$, which gives $T(m/2^{k}, b) = b - T(m'/2^{k}, b) = b - (m'/2^{k})b = (m/2^{k})b$ by the case already established. This completes the induction.

\emph{Step 3 (from dyadics to $[0,1]$).} Fix $b$ and let $\varphi(a) = T(a,b)$, which is non-decreasing by the monotonicity axiom, and let $\psi(a) = ab$, which is continuous. By Step~2, $\varphi = \psi$ on the dyadic rationals, which are dense in $[0,1]$. For arbitrary $a \in [0,1]$, choose sequences of dyadic rationals $(d_r)$ and $(e_r)$ converging monotonically to $a$ from below and from above, respectively. Monotonicity implies $\varphi(d_r) \le \varphi(a) \le \varphi(e_r)$, which is equivalent to $\psi(d_r) \le \varphi(a) \le \psi(e_r)$. Since both bounding sequences converge to $\psi(a) = ab$ by continuity of $\psi$, the squeeze theorem gives $\varphi(a) = ab$. As $b$ was arbitrary, $T = T_P$.
\end{proof}

A critical operational consequence of this marginal consistency is that the fuzzy confusion counts preserve the total sample mass entrywise. Corollary~\ref{cor:partition} formalises how $T_P$ guarantees that the four soft confusion matrices partition the unit measure exactly as in the crisp setting.

\begin{corollary}[Partition of the unit measure]
\label{cor:partition}
Under $T_P$ the four fuzzy confusion matrices partition the unit measure entrywise,
$$
\mathbold{TP}_{i,j} + \mathbold{FP}_{i,j} + \mathbold{FN}_{i,j} + \mathbold{TN}_{i,j} = 1
\qquad \text{for all } i,j,
$$
so that $\Upsigma_{\mathrm{1}}(\mathbold{TP}) + \Upsigma_{\mathrm{1}}(\mathbold{FP}) + \Upsigma_{\mathrm{1}}(\mathbold{FN}) + \Upsigma_{\mathrm{1}}(\mathbold{TN}) = nm$, exactly as in the crisp case.
\end{corollary}

\begin{proof}
Applying the second identity in Eq.~(\ref{eq:marginal}) at $b$ and at $1-b$ and adding gives $\bigl[T(a,b) + T(1{-}a,b)\bigr] + \bigl[T(a,1{-}b) + T(1{-}a,1{-}b)\bigr] = b + (1-b) = 1$. Summing over the $nm$ entries gives the second claim.
\end{proof}

Theorem~\ref{thm:partition} provides a principled justification for preferring $T_P$ in probabilistic settings: it is the only $t$-norm under which the fuzzy confusion matrices behave like a contingency table, with row and column marginals reproducing the actual and predicted memberships, in exact analogy with the crisp case. Under the Gödel and Łukasiewicz $t$-norms the marginals do not match, and by Corollary~\ref{cor:partition} the total mass differs from $nm$, so measures computed from them have denominators that differ from the expected normalisation, complicating interpretation.

\begin{remark}[On the weaker partition-only condition]
\label{rem:partition_only}
It is natural to ask whether the weaker requirement obtained by imposing only the sum of the four cells,
\begin{equation}
T(a,b) + T(1{-}a,b) + T(a,1{-}b) + T(1{-}a,1{-}b) = 1
\quad \text{for all } a,b,
\label{eq:partition_only}
\end{equation}
already characterises $T_P$. It does not follow from Theorem~\ref{thm:partition}, since Eq.~(\ref{eq:marginal}) implies Eq.~(\ref{eq:partition_only}) but not conversely, and we do not claim a characterisation under Eq.~(\ref{eq:partition_only}) alone. Two necessary conditions are, however, immediate and hold for every $t$-norm: setting $a = b = 1/2$ in Eq.~(\ref{eq:partition_only}) gives $T(1/2,1/2) = 1/4$, and setting $b = 1/2$ with $a$ arbitrary gives $T(a,1/2) + T(1{-}a,1/2) = 1/2$. The first excludes the Gödel $t$-norm, for which $T_M(1/2,1/2) = 1/2$, and the Łukasiewicz $t$-norm, for which $T_L(1/2,1/2) = 0$; more generally it excludes every $t$-norm for which $1/2$ is idempotent or is a zero divisor of itself. Direct computation further shows that within each of the standard parametric families that contain it (Frank, Hamacher, Schweizer--Sklar, and Aczél--Alsina), the product $t$-norm is the only member satisfying Eq.~(\ref{eq:partition_only}), while the Yager and Dombi families contain no member that satisfies it. Whether Eq.~(\ref{eq:partition_only}) characterises $T_P$ over the full class of $t$-norms is left open; for the purposes of this paper the question is not essential, since marginal consistency is the property the crisp analogy requires, and Theorem~\ref{thm:partition} settles it.
\end{remark}

Just as the matrix framework provides an algebraic foundation for soft-label evaluation through marginally consistent $t$-norms, it unifies structured evaluation in ordinal settings through generalized cost operators. Theorem~\ref{thm:ordinal_cost} shows that standard ordinal regression losses (MAE and MSE) reduce identically to cost-sensitive matrix evaluations with structured metric cost matrices.

\begin{theorem}[Ordinal losses as cost-sensitive classification]
\label{thm:ordinal_cost}
Let $\mathbold{Y}$ and $\mathbold{\hat{Y}}$ be one-hot indicator matrices for a multiclass problem with ordered classes $c_1 \prec \cdots \prec c_m$. Then:
\begin{enumerate}
  \item With cost matrix $C_{j,k} = |j-k|$, $\frac{1}{n}\,\mathrm{Cost}(\mathbold{Y}, \mathbold{\hat{Y}}, \mathbold{C}) = \mathrm{MAE}$;
  \item With cost matrix $C_{j,k} = (j-k)^2$, $\frac{1}{n}\,\mathrm{Cost}(\mathbold{Y}, \mathbold{\hat{Y}}, \mathbold{C}) = \mathrm{MSE}$.
\end{enumerate}
\end{theorem}

\begin{proof}
Since $\mathbold{Y}$ is one-hot with $Y_{i,k_i} = 1$ and $\mathbold{\hat{Y}}$ is one-hot with $\hat{Y}_{i,\hat{k}_i} = 1$, the product $(\mathbold{Y}\mathbold{C})_{i,k} = \sum_j Y_{i,j} C_{j,k} = C_{k_i,k}$. Applying $\Upsigma_{\mathrm{1}}$ to $(\mathbold{Y}\mathbold{C}) \pointwise \mathbold{\hat{Y}}$ then gives $\sum_i \sum_k C_{k_i,k}\hat{Y}_{i,k} = \sum_i C_{k_i,\hat{k}_i}$. Substituting $C_{j,k} = |j-k|$ yields $\sum_i |k_i - \hat{k}_i| = n \cdot \mathrm{MAE}$, and substituting $C_{j,k} = (j-k)^2$ yields $\sum_i (k_i - \hat{k}_i)^2 = n \cdot \mathrm{MSE}$.
\end{proof}

Theorem~\ref{thm:ordinal_cost} unifies ordinal classification losses with cost-sensitive classification: MAE and MSE are not separate evaluation paradigms but instances of the general cost-sensitive framework with specific cost matrices. Combined with Eq.~(\ref{eq:ordinal_membership}), which encodes ordinal structure in $\mathbold{Y}$ itself, there are therefore two orthogonal ways to introduce ordinal sensitivity (through the label matrix or through the cost matrix) that can be combined independently.

A crucial consequence of this linear cost-matrix formulation is that applying arbitrary misclassification costs does not disrupt the algebraic additivity of the confusion counts. Corollary~\ref{cor:cost_micro_macro} confirms that the fundamental micro--macro equivalence of Theorem~\ref{thm:micro_macro} extends directly to all cost-sensitive and ordinal measures.

\begin{corollary}[Cost-sensitive micro equals weighted macro]
\label{cor:cost_micro_macro}
Let $M^C$ be an \emph{aggregation-decomposable} performance measure computed using cost-weighted confusion matrices as in Section~\ref{sec:cost}. Then $M^{C,\mathrm{micro}}$ equals the weighted average of $\{M^{C,\mathrm{class-}j}\}_j$, with weights equal to the per-class denominator of $M^C$.
\end{corollary}

\begin{proof}
Cost weighting replaces each count by a cost-scaled count, $\mathbold{v} \mapsto \mathbold{v}^C$, where the entries of $\mathbold{v}^C$ are obtained from $\Upsigma_{\mathrm{1}}\bigl((\mathbold{Y}\mathbold{C}) \pointwise \mathbold{\hat{Y}}\bigr)$ and its per-class restrictions. Since $\mathbold{Y}\mathbold{C}$ is linear in $\mathbold{Y}$ and the Hadamard product with $\mathbold{\hat{Y}}$ acts entrywise, the cost-weighted counts remain additive across classes and examples: $\mathbold{v}^C = \sum_j \mathbold{v}^C_j$. Decomposability of $M^C$ then supplies the additivity in Eq.~(\ref{eq:additivity}) of its numerator and denominator in the cost-weighted counts, and Theorem~\ref{thm:micro_macro} applies verbatim with $\mathbold{v}^C_j$ in place of $\mathbold{v}_j$. Note that the hypothesis is required here for the same reason as in Theorem~\ref{thm:micro_macro}: it is the additivity of $num$ and $den$, not the presence of costs, that drives the result.
\end{proof}

\subsection{Preservation of Monotonicity and Boundary Correctness}
\label{sec:properties}

A fundamental theoretical question when lifting binary performance measures to complex learning paradigms via algebraic operator substitution is whether the resulting extensions inherit the axiomatic properties of their binary progenitors. We formalise two cornerstone properties below: monotonicity and boundary correctness. While monotonicity and upper boundary correctness (maximal attainment under perfect classification) are preserved unconditionally across all aggregation operators, lower boundary correctness (minimal attainment under completely erroneous predictions) fails systematically under one-hot multiclass encoding for any measure exhibiting true-negative dependence. Characterising this breakdown requires first formalising the reachable prediction spaces across learning settings.

To formalise attainability, note that the classification settings differ in which matrices $\mathbold{\hat{Y}}$ a classifier can produce:
$$
\mathcal{P}_{\mathrm{multilabel}} = \{0,1\}^{n \times m},
\qquad
\mathcal{P}_{\mathrm{multiclass}} = \Bigl\{ \mathbold{\hat{Y}} \in \{0,1\}^{n\times m} : \textstyle\sum_j \hat{Y}_{i,j} = 1 \ \forall i \Bigr\}.
$$
Boundary correctness is a statement about \emph{attainability} of extreme values, so it must be evaluated over the admissible set. In the multilabel case $\mathcal{P}$ is unrestricted and the complement $\mathbold{\hat{Y}} = \mathbf{1} - \mathbold{Y}$ is admissible; under one-hot multiclass with $m > 2$ it is not, since $\mathbf{1} - \mathbold{Y}$ has $m-1$ ones per row.

\begin{mydef}[Boundary correctness]
A binary measure $M(tp, fp, fn, tn)$ is \emph{upper boundary correct} if it attains its theoretical maximum whenever $fp = fn = 0$, and \emph{lower boundary correct} if it attains its theoretical minimum whenever $tp = tn = 0$. It is \emph{boundary correct} if both hold.
\end{mydef}

These are the properties Hand et al.~\cite{hand2024selecting} call boundedness above and below, together with containment in an interval. Splitting them is what the analysis below requires, since the two halves behave differently under the aggregation operators.

\begin{mydef}[Monotonicity under admissible corrections]
\label{def:monotone}
A binary measure $M(tp,fp,fn,tn)$ is \emph{monotone} if it is non-decreasing in $tp$ and $tn$ and non-increasing in $fp$ and $fn$, each with the remaining arguments held fixed. A perturbation of $\mathbold{\hat{Y}}$ within $\mathcal{P}$ is a \emph{correction} if it changes the count vector by a non-negative amount in $tp$ and $tn$ and a non-positive amount in $fp$ and $fn$.
\end{mydef}

The second half of Definition~\ref{def:monotone} is needed because in the one-hot setting no admissible perturbation changes a single count in isolation: correcting one example's prediction necessarily moves four counts at once. In binary classification the corresponding moves affect two cells, and Hand et al.~\cite{hand2024selecting} accordingly state monotonicity as the pair of conditions
$$
M(tp{+}1,\, fp,\, fn{-}1,\, tn) \geq M(tp,fp,fn,tn),
\qquad
M(tp,\, fp{-}1,\, fn,\, tn{+}1) \geq M(tp,fp,fn,tn),
$$
each of which reclassifies one object correctly within its own true class. Definition~\ref{def:monotone} is the extension of that formulation to the admissible moves of a general setting. Monotonicity of $M$ in each argument separately still implies monotonicity along such joint moves, which is what the following proposition uses.

\begin{proposition}[Upper boundary correctness and monotonicity are preserved]
\label{prop:preserved}
Let $M$ be upper boundary correct and monotone. Then in every setting, and for each of $M^{\mathrm{micro}}$, $M^{\mathrm{macro}}$ and $M^{\mathrm{exemplar}}$:
\begin{enumerate}
  \item the extension attains its maximum at $\mathbold{\hat{Y}} = \mathbold{Y}$; and
  \item the extension does not decrease under any correction in the sense of Definition~\ref{def:monotone}.
\end{enumerate}
\end{proposition}

\begin{proof}
Part~1: when $\mathbold{\hat{Y}} = \mathbold{Y}$ we have $\mathbold{FP} = \mathbold{FN} = \mathbf{0}$ entrywise, hence $fp = fn = 0$ after applying any of the three operators, and upper boundary correctness of $M$ gives the maximum in each case; averaging maxima over classes or examples yields the maximum of the average. Note that $\mathbold{\hat{Y}} = \mathbold{Y}$ is admissible in every setting, so the maximum is attainable.

Part~2: consider an admissible correction in the sense of Definition~\ref{def:monotone}---that is, an elementary perturbation of the prediction matrix $\mathbold{\hat{Y}} \in \mathcal{P}$ where an incorrectly predicted example is reclassified to its ground-truth label. By definition, this modification induces count increments satisfying $\Delta tp, \Delta tn \ge 0$ and $\Delta fp, \Delta fn \le 0$. Decomposing the pooled change into coordinate steps and applying coordinate-wise monotonicity of $M$ shows that $M^{\mathrm{micro}}$ does not decrease. For $\Upsigma_{\mathrm{m}}$, in the one-hot setting correcting an example whose true class is $a$ and predicted class is $b \neq a$ converts a false negative into a true positive in column $a$ and a false positive into a true negative in column $b$. Thus, both per-class measures $M_a$ and $M_b$ experience an admissible correction in their own right and neither decreases, while all other $M_j$ remain unchanged. The same argument applies row-wise for $\Upsigma_{\mathrm{n}}$.
\end{proof}

\begin{proposition}[Lower boundary correctness may fail under one-hot encoding]
\label{prop:lower_boundary}
Let $M$ be lower boundary correct. Then:
\begin{enumerate}
  \item In the binary and multilabel settings the three extensions are lower boundary correct, the minimum being attained at $\mathbold{\hat{Y}} = \mathbf{1} - \mathbold{Y}$.
  \item In the one-hot multiclass setting with $m > 2$, the state $tp = tn = 0$ is \emph{unreachable}. Every $\mathbold{\hat{Y}} \in \mathcal{P}_{\mathrm{multiclass}}$ in which no prediction is correct yields
  \begin{equation}
  \Upsigma_{\mathrm{1}}(\mathbold{TP}) = 0,\quad
  \Upsigma_{\mathrm{1}}(\mathbold{FP}) = \Upsigma_{\mathrm{1}}(\mathbold{FN}) = n,\quad
  \Upsigma_{\mathrm{1}}(\mathbold{TN}) = n(m-2).
  \label{eq:worst_onehot}
  \end{equation}
  Consequently, if $M(0, fp, fn, tn)$ is independent of $tn$ then $M^{\mathrm{micro}}$ attains its theoretical minimum at the worst admissible prediction; and if $M(0,fp,fn,tn)$ is strictly increasing in $tn$ then it does not, the attainable floor lying strictly above the theoretical minimum and depending on $m$.
\end{enumerate}
\end{proposition}

\begin{proof}
Part~1: in these settings $\mathbf{1} - \mathbold{Y} \in \mathcal{P}$, and it gives $\mathbold{TP} = \mathbold{TN} = \mathbf{0}$ entrywise, hence $tp = tn = 0$ under each operator.

Part~2: fix an example $i$ with true class $a$ and predicted class $b \neq a$. Column $a$ contributes a false negative, column $b$ a false positive, and each of the remaining $m-2$ columns has $Y_{i,j} = \hat{Y}_{i,j} = 0$ and so contributes a true negative. Summing over the $n$ examples gives Eq.~(\ref{eq:worst_onehot}). Since $m > 2$ implies $\Upsigma_{\mathrm{1}}(\mathbold{TN}) = n(m-2) > 0$, the state $tp = tn = 0$ is not attained. If $M(0,fp,fn,tn)$ does not depend on $tn$ then its value in Eq.~(\ref{eq:worst_onehot}) coincides with its value at $tn = 0$, which is the theoretical minimum; if it depends on $tn$ and is non-decreasing in it, the value at $tn = n(m-2)$ is strictly larger.
\end{proof}

This is a substantive limitation rather than a technicality. The equivalence between ``no prediction is correct'' and $tp = tn = 0$ holds in binary classification and in multilabel classification, but under one-hot multiclass encoding an incorrectly classified example still contributes $m-2$ true negatives, so $tn$ is large precisely when performance is worst. Table~\ref{tab:floors} records the resulting floors.

\begin{table}[!htb]
\centering
\small
\begin{tabular}{l|c|cccc}
\hline
\textbf{Measure} & \textbf{Attained bound} & $m=3$ & $m=5$ & $m=10$ & $m\to\infty$ \\
\hline
Recall, precision, $F_\beta$ & $0$ & $0$ & $0$ & $0$ & $0$ \\
Jaccard, G-mean & $0$ & $0$ & $0$ & $0$ & $0$ \\
Accuracy $\Upsigma_{\mathrm{1}}(\mathbold{TP})/n$ & $0$ & $0$ & $0$ & $0$ & $0$ \\
\hline
Label accuracy $\frac{tp+tn}{nm}$ & $\dfrac{m-2}{m}$ & $0.333$ & $0.600$ & $0.800$ & $1$ \\[4pt]
Specificity & $\dfrac{m-2}{m-1}$ & $0.500$ & $0.750$ & $0.889$ & $1$ \\[4pt]
Balanced accuracy & $\dfrac{m-2}{2(m-1)}$ & $0.250$ & $0.375$ & $0.444$ & $1/2$ \\[4pt]
MCC & $-\dfrac{1}{m-1}$ & $-0.500$ & $-0.250$ & $-0.111$ & $0$ \\[4pt]
Cohen's $\kappa$ & $-\dfrac{1}{m-1}$ & $-0.500$ & $-0.250$ & $-0.111$ & $0$ \\[4pt]
Error rate (ceiling) & $\dfrac{2}{m}$ & $0.667$ & $0.400$ & $0.200$ & $0$ \\
\hline
\end{tabular}
\caption{Value attained by each micro-averaged measure at the worst admissible one-hot multiclass prediction (no prediction correct). Measures whose value at $tp = 0$ does not involve $tn$ reach $0$ as expected; those involving $tn$ are bounded away from their theoretical minimum, with the gap growing in $m$. The two accuracy rows must be distinguished (Section~\ref{sec:notable}): standard multiclass accuracy $\Upsigma_{\mathrm{1}}(\mathbold{TP})/n$ does not involve $tn$ and correctly reaches $0$, consistent with Theorem~\ref{thm:micro_collapse}, whereas label accuracy $(tp+tn)/(nm)$ tends to $1$ as $m$ grows even though every prediction is wrong. MCC and Cohen's $\kappa$ (both of which reach $-1$ in the binary case) tend to $0$, so neither can register strong disagreement in a many-class problem. The bound is a floor for every row except error rate, for which it is a ceiling.}
\label{tab:floors}
\end{table}

The G-mean warrants particular attention: although it explicitly depends on specificity, it correctly attains $0$ at the worst-case prediction because $\mathrm{rec} = 0$ annihilates the geometric product $\sqrt{\mathrm{rec}\cdot\mathrm{spec}}$ regardless of the specificity term. By contrast, balanced accuracy is an arithmetic rather than geometric mean; lacking this multiplicative nullification, it directly inherits the positive floor $\frac{m-2}{2(m-1)}$. This highlights a fundamental algebraic divergence between additive and multiplicative combinations of rates under multi-output aggregation, directly paralleling the decomposability distinction of Section~\ref{sec:decomposable}. The MCC and Cohen's $\kappa$ floors are perhaps the most consequential entries in the table: both are designed to take the value $-1$ under complete disagreement and $0$ under chance agreement, and under one-hot encoding their floor of $-1/(m-1)$ collapses towards $0$, so on a ten-class problem a classifier that is wrong on every single example scores $-0.111$, barely distinguishable from chance by the measure's own scale.

\begin{remark}[Completeness and lower boundary correctness are logically independent]
\label{rem:completeness}
Hand et al.~\cite{hand2024selecting} define \emph{completeness} as whether an evaluation measure explicitly depends on all four cells of the confusion matrix. One might suspect from Proposition~\ref{prop:lower_boundary} that completeness and lower boundary correctness are mutually incompatible under one-hot multiclass encoding, since complete measures incorporate $tn$. However, this intuition does not hold in general, and the G-mean demonstrates why: it is complete, yet $\sqrt{\mathrm{rec}\cdot\mathrm{spec}} = 0$ whenever $\mathrm{rec} = 0$, so its dependence on $tn$ vanishes at $tp = 0$, preserving its theoretical minimum. What Proposition~\ref{prop:lower_boundary} strictly requires is not that $M$ ignore $tn$ everywhere, but rather that the boundary evaluation $M(0, fp, fn, tn)$ be independent of $tn$, which is a strictly weaker condition.

The two properties are logically independent, all four combinations being realised among the measures of Table~\ref{tab:measures}:
\begin{center}
\small
\begin{tabular}{l|ll}
\hline
& \textbf{Lower boundary correct} & \textbf{Not lower boundary correct} \\
\hline
\textbf{Complete} & G-mean & Label accuracy, balanced accuracy, MCC, Cohen's $\kappa$ \\
\textbf{Incomplete} & Recall, precision, $F_\beta$, Jaccard & Specificity \\
\hline
\end{tabular}
\end{center}
The practically important cell is the top right. Several measures that are recommended precisely because they are complete (such as accuracy, MCC, and Cohen's $\kappa$) lose lower boundary correctness under one-hot multiclass encoding, and the loss widens with $m$. The bottom-right cell shows the converse is also possible: specificity is incomplete and still fails. Completeness is therefore neither necessary nor sufficient for the failure, and a practitioner cannot use it as a proxy. The correct test is the one in Proposition~\ref{prop:lower_boundary}, and it is straightforward to apply: substitute $tp = 0$ into the formula and check whether $tn$ survives. More broadly, this underscores a fundamental methodological principle: axiomatic guarantees that justify selecting a specific binary measure do not automatically transfer to multi-class or structured settings, and their theoretical validity must be explicitly re-evaluated within the admissible prediction space of the target paradigm.
\end{remark}

The earlier remark in Section~\ref{sec:notable} about per-example accuracy reflects the same phenomenon seen through a different operator, and the two are worth reconciling explicitly. In the one-hot case per-example accuracy is exactly one for a correctly classified example and $(m-2)/m$ for an incorrect one, so
\begin{equation}
\mathrm{Accuracy}^{\mathrm{exemplar}} = \frac{m-2}{m} + \frac{2}{m}\,\mathrm{Accuracy},
\label{eq:exemplar_acc}
\end{equation}
an increasing affine transformation of standard accuracy with intercept $(m-2)/m$. The intercept is exactly the label-accuracy floor in Table~\ref{tab:floors}, so the two observations are the same phenomenon seen from different operators: the transformation is order-preserving, which is why monotonicity survives, but its range is $[(m-2)/m, 1]$ rather than $[0,1]$, which is why lower boundary correctness does not.

\begin{remark}[Benchmarking and optimization implications of axiomatic preservation]
\label{rem:boundary_implications}
The preservation of monotonicity alongside the selective failure of lower boundary correctness establishes three direct methodological consequences for machine learning practice:
\begin{enumerate}
  \item \textbf{Guaranteed gradient and post-processing consistency:} Because monotonicity under admissible corrections is preserved universally (Proposition~\ref{prop:preserved}), any local optimization step or threshold adjustment that corrects an individual prediction is guaranteed never to decrease the global evaluation score, regardless of which aggregation operator is employed.
  \item \textbf{Distortion in cross-dataset benchmarking:} For all measures with non-trivial $tn$-dependence (Table~\ref{tab:floors}), the effective dynamic evaluation range $[(m-2)/(m-1), 1]$ contracts severely as the number of classes $m$ increases. Consequently, comparing raw scores of balanced accuracy, label accuracy, or correlation metrics across datasets with varying numbers of classes introduces artificial distortion unless normalized relative to their $m$-dependent theoretical floors.
  \item \textbf{Asymmetry in chance-corrected scales:} The contraction of the negative sub-interval of MCC and Cohen's $\kappa$ from $[-1, 0]$ down to $[-\frac{1}{m-1}, 0]$ introduces strong scale asymmetry as $m$ grows, substantially compressing their dynamic range for penalizing systematic misclassifications below chance agreement.
\end{enumerate}
\end{remark}

From a practical perspective, two consequences are worth stating for practitioners. First, the measures whose $tn$-dependence survives at $tp = 0$ (label accuracy, specificity, balanced accuracy, negative predictive value, markedness, MCC, and Cohen's $\kappa$) should not be applied to one-hot multiclass indicator matrices without rescaling, because their effective range shrinks as $m$ grows; reporting a label accuracy of $0.8$ or an MCC of $-0.11$ on a ten-class problem is uninformative without knowing that these are the floors. Each requires its own affine renormalisation, obtained by mapping its floor in Table~\ref{tab:floors} to the nominal minimum; Eq.~(\ref{eq:exemplar_acc}) gives the case of exemplar accuracy. Where such a measure is wanted, the affine renormalisation implied by Eq.~(\ref{eq:exemplar_acc}) restores the full $[0,1]$ range. Second, the measures that are safe in this respect are exactly those that ignore $tn$ at $tp = 0$, which for the catalogue in the supporting material can be read off directly from the coefficient $b_4$ of the denominator and the presence of $tn$ in the numerator.

In summary, monotonicity and upper boundary correctness are preserved by all three operators in all settings as shown in Proposition~\ref{prop:preserved}, and lower boundary correctness is preserved in the binary and multilabel settings but can fail under one-hot multiclass encoding, precisely for those measures whose value at $tp = 0$ still depends on $tn$ by Proposition~\ref{prop:lower_boundary}. The framework's advantage over ad-hoc extension remains that these questions are settled once for all generated measures rather than case by case~\cite{DBLP:journals/ipm/SokolovaL09}; what the analysis above shows is that the answer is partly negative, and that the framework is what makes the exact scope of the failure visible.


\subsection{Characterisation of Binary Skew-Invariance}
\label{sec:skew}

Class imbalance is one of the most practically consequential sources of variation in evaluation outcomes. We now formalise what it means for a measure to be sensitive or insensitive to the class distribution, and use the indicator matrix framework to characterise which measures in Table~\ref{tab:extensions} belong to each category and how the three aggregation operators interact with skew.

Skew-invariance is a property that must be defined separately for each classification setting, because what plays the role of ``the negative class'' differs between them. We therefore begin with the binary case, where the notion is cleanest, and treat the multiclass and multilabel cases separately in Section~\ref{sec:skew_settings}. Conflating the three leads to statements that are true in one setting and false in another, so we keep them separate throughout.

In the binary case, recall the parametric decomposition of the confusion matrix from Example~\ref{ex:dof}:
$$
\mathbf{C} = \begin{pmatrix} tp & fn \\ fp & tn \end{pmatrix} = \begin{pmatrix} n\pi\,\mathrm{rec} & n\pi(1-\mathrm{rec}) \\ n(1-\pi)(1-\mathrm{spec}) & n(1-\pi)\,\mathrm{spec} \end{pmatrix},
$$
where $\pi = P/(P+N) \in (0,1)$ is the positive class prior and the conditional operating rates $(\mathrm{rec}, \mathrm{spec}) \in [0,1]^2$ characterize the classifier. When class prevalence shifts from $\pi$ to $\pi'$ while the classifier's conditional behavior $(\mathrm{rec}, \mathrm{spec})$ remains fixed, the entries of the positive row $(tp, fn)$ are scaled by $\lambda = \pi'/\pi > 0$, while the negative row $(fp, tn)$ is scaled by $\mu = (1-\pi')/(1-\pi) > 0$. Class skew therefore acts geometrically as an independent two-parameter rescaling of the rows of $\mathbf{C}$, which motivates the following definition.

\begin{mydef}[Skew-invariant binary measure]
\label{def:skew}
Let
$$
\mathcal{D} = \bigl\{(tp,fp,fn,tn) \in \mathbb{R}_{\geq 0}^{4} \;:\; tp + fn > 0 \text{ and } tn + fp > 0\bigr\}
$$
be the set of count vectors for which both recall and specificity are defined. A binary measure $M : \mathcal{D} \to \mathbb{R}$ is \emph{skew-invariant} if
$$
M(\lambda\, tp,\; \mu\, fp,\; \lambda\, fn,\; \mu\, tn) = M(tp, fp, fn, tn)
\qquad \text{for all } \lambda, \mu > 0
$$
and all $(tp,fp,fn,tn) \in \mathcal{D}$. A measure that is not skew-invariant is \emph{skew-sensitive}.
\end{mydef}

Definition~\ref{def:skew} is related to, but distinct from, the \emph{constant baseline} property of G{\"o}sgens et al.~\cite{gosgens2021good}, discussed also by Hand et al.~\cite{hand2024selecting}, which requires that random assignments receive the same expected score regardless of how many objects are assigned to each class. Constant baseline constrains the measure's behaviour at a single reference point (random prediction), whereas skew-invariance constrains it over the whole space, requiring the value to be unchanged for every fixed conditional behaviour. Skew-invariance is therefore the stronger requirement, and Theorem~\ref{thm:skew} shows how strong: it forces the measure to be a function of recall and specificity alone.

The restriction to the domain $\mathcal{D}$ is mathematically essential. If $tp + fn = 0$, the positive class has zero prevalence: recall is undefined ($0/0$) and the positive row remains identically zero under any scaling $\lambda > 0$, preventing any meaningful invariance analysis; the same degeneracy occurs if $tn + fp = 0$. Requiring strictly positive class marginals guarantees that both classes are populated, allowing the scaling action $(\lambda,\mu) \in \mathbb{R}_{>0}^2$ to uniquely parameterize classifier behavior through the rate pair $(\mathrm{rec}, \mathrm{spec})$. Section~\ref{sec:degenerate} treats the excluded boundary cases separately.

\begin{theorem}[Characterisation of skew-invariant binary measures]
\label{thm:skew}
A measure $M : \mathcal{D} \to \mathbb{R}$ is skew-invariant in the sense of Definition~\ref{def:skew} if and only if there exists $f : [0,1]^2 \to \mathbb{R}$ with $M = f(\mathrm{rec}, \mathrm{spec})$, where $\mathrm{rec} = tp/(tp+fn)$ and $\mathrm{spec} = tn/(tn+fp)$.
\end{theorem}

\begin{proof}
($\Leftarrow$) On $\mathcal{D}$ both denominators are strictly positive, so under the scaling $(\lambda,\mu)$ with $\lambda, \mu > 0$,
$$
\mathrm{rec}' = \frac{\lambda\,tp}{\lambda\,tp + \lambda\,fn} = \mathrm{rec},
\qquad
\mathrm{spec}' = \frac{\mu\,tn}{\mu\,tn + \mu\,fp} = \mathrm{spec},
$$
the cancellations being legitimate precisely because $\lambda(tp+fn) > 0$ and $\mu(tn+fp) > 0$. Hence $M' = f(\mathrm{rec},\mathrm{spec}) = M$.

($\Rightarrow$) Let $\Phi(tp,fp,fn,tn) = (\mathrm{rec}, \mathrm{spec}) \in [0,1]^2$. For any count vector $(tp,fp,fn,tn) \in \mathcal{D}$, setting $\lambda = 1/(tp+fn) > 0$ and $\mu = 1/(tn+fp) > 0$ yields the canonical row-normalized form:
$$
\bigl(\mathrm{rec},\; 1-\mathrm{spec},\; 1-\mathrm{rec},\; \mathrm{spec}\bigr),
$$
which depends exclusively on the rate pair $(\mathrm{rec}, \mathrm{spec})$. Hence, two count vectors in $\mathcal{D}$ can be transformed into each other by positive row scaling $(\lambda, \mu) \in \mathbb{R}_{>0}^2$ if and only if they share identical $(\mathrm{rec}, \mathrm{spec})$ rates. Since a skew-invariant measure $M$ is invariant under all such scalings by Definition~\ref{def:skew}, its value is completely determined by $(\mathrm{rec}, \mathrm{spec})$, establishing the factorization $M = f(\mathrm{rec}, \mathrm{spec})$ for some function $f: [0,1]^2 \to \mathbb{R}$.
\end{proof}

Table~\ref{tab:skew} classifies each measure in Table~\ref{tab:extensions} accordingly. We emphasise that Theorem~\ref{thm:skew} is a statement about \emph{binary} measures on $\mathcal{D}$; whether the corresponding per-class quantities in a multiclass or multilabel problem are themselves prior-free is a separate question, taken up next.

\begin{table}[!t]
\centering
\small
\begin{tabular}{l|c|l}
\hline
& \textbf{Skew-invariant?} & \textbf{Reason} \\
\hline
Recall            & Yes & $= tp/(tp+fn)$; scales with $\lambda$ only \\
Specificity       & Yes & $= tn/(tn+fp)$; scales with $\mu$ only \\
Balanced accuracy & Yes & Average of recall and specificity \\
G-mean            & Yes & Geometric mean of recall and specificity \\
\hline
Precision         & No  & $tp/(tp+fp)$ mixes $\lambda$ and $\mu$ \\
$F_\beta$         & No  & Involves precision \\
Jaccard           & No  & $tp/(tp+fp+fn)$ mixes $\lambda$ and $\mu$ \\
Accuracy          & No  & $(tp+tn)/(tp+tn+fp+fn)$ mixes both \\
MCC               & No  & Involves products of mixed terms \\
Error rate        & No  & $(fp+fn)/n$ mixes both \\
Cohen's $\kappa$  & No  & Depends on marginal distributions \\
\hline
\end{tabular}
\caption{Skew-invariance of each \emph{binary} measure in Table~\ref{tab:extensions}, on the domain $\mathcal{D}$ of Definition~\ref{def:skew}. A binary measure is skew-invariant iff it is a function of recall and specificity only (Theorem~\ref{thm:skew}). Table~\ref{tab:skew_settings} records how these classifications transfer (and where they fail to transfer) to the multiclass and multilabel settings.}
\label{tab:skew}
\end{table}

\subsection{Skew-Invariance in Multiclass and Multilabel Settings}
\label{sec:skew_settings}

Definition~\ref{def:skew} concerns a binary measure applied to a single count vector. In multiclass and multilabel problems the aggregation operators produce a family of count vectors, one per class or label, and the question becomes whether the per-class quantities are themselves unaffected by the class distribution. The answer differs sharply between the two settings.

In a multilabel problem, each label $l_j$ defines a genuine binary sub-problem: $Y_{i,j} \in \{0,1\}$ with prevalence $p_j$, and the classifier's conditional behaviour on that label is captured by $r_j = P(\hat{Y}_{i,j} = 1 \mid Y_{i,j} = 1)$ and $s_j = P(\hat{Y}_{i,j} = 0 \mid Y_{i,j} = 0)$. Then $\mathrm{rec}_j = r_j$ and $\mathrm{spec}_j = s_j$ exactly, both independent of $p_j$. Column-wise aggregation of a skew-invariant binary measure therefore inherits the invariance directly.

In contrast, the one-vs-rest decomposition of a multiclass problem is not a collection of independent binary tasks, and this structural coupling leads to fundamentally different behavior under prior shift. Let $\pi_k = P(c_k)$ denote the class priors and $R_{kj} = P(\hat{c}_j \mid c_k)$ the conditional classification rates. The one-vs-rest counts for class $j$ are
$$
tp_j = n\pi_j R_{jj}, \quad
fn_j = n\pi_j(1 - R_{jj}), \quad
fp_j = n\!\!\sum_{k \neq j}\!\pi_k R_{kj}, \quad
tn_j = n\!\!\sum_{k \neq j}\!\pi_k (1 - R_{kj}).
$$

\begin{proposition}[Prior-dependence of one-vs-rest rates]
\label{prop:ovr_skew}
In a multiclass problem with fixed conditional rates $R_{kj}$ and $\pi_j > 0$:
\begin{enumerate}
  \item Per-class recall is prior-free: $\mathrm{rec}_j = R_{jj}$ for every prior vector $\boldsymbol{\pi}$.
  \item Per-class specificity is a prior-weighted average of the other classes' correct-rejection rates,
  \begin{equation}
  \mathrm{spec}_j
  = \frac{\sum_{k \neq j}\pi_k\,(1 - R_{kj})}{\sum_{k \neq j}\pi_k}
  = \sum_{k \neq j} \tilde{\pi}^{(j)}_k \,(1 - R_{kj}),
  \qquad \tilde{\pi}^{(j)}_k = \frac{\pi_k}{\sum_{l \neq j}\pi_l},
  \label{eq:ovr_spec}
  \end{equation}
  and is therefore not prior-free in general. It is prior-free for all $\boldsymbol{\pi}$ if and only if $R_{kj}$ is constant over $k \neq j$, in which case $\mathrm{spec}_j = 1 - R_{\cdot j}$.
\end{enumerate}
\end{proposition}

\begin{proof}
Part~1 is immediate: $\mathrm{rec}_j = tp_j/(tp_j+fn_j) = \pi_j R_{jj}/\pi_j = R_{jj}$, using $\pi_j > 0$. For part~2, substituting the expressions for $tn_j$ and $fp_j$ gives $tn_j + fp_j = n\sum_{k\neq j}\pi_k$ and hence Eq.~(\ref{eq:ovr_spec}). This is a convex combination of the values $\{1 - R_{kj}\}_{k \neq j}$ with weights $\tilde{\pi}^{(j)}$ ranging over the whole relative interior of the simplex on $\{k : k \neq j\}$ as $\boldsymbol{\pi}$ varies. A convex combination is constant over all such weight vectors if and only if the combined values are all equal, i.e.\ iff $R_{kj}$ does not depend on $k$.
\end{proof}

Proposition~\ref{prop:ovr_skew} has consequences that are easy to state but easy to miss, since the natural expectation is that macro-averaging inherits binary skew-invariance in every setting. It does not:

\begin{itemize}\itemsep2pt
\item \textbf{Macro-recall is skew-invariant in all three settings}, since $\mathrm{rec}_j$ is prior-free by Proposition~\ref{prop:ovr_skew}(1) and by the multilabel argument above.
\item \textbf{Macro-balanced-accuracy, macro-G-mean and macro-informedness are skew-invariant in binary and multilabel settings but only approximately so in multiclass}, where they inherit a residual prior-dependence through the specificity term. Table~\ref{tab:skew_settings} summarises this.
\item The residual dependence vanishes exactly when the classifier distributes its errors into each class uniformly across the other classes ($R_{kj}$ independent of $k$), a condition that is restrictive and easy to check.
\end{itemize}

\begin{example}[Residual skew-sensitivity of macro-averages in multiclass]
\label{ex:ovr_skew}
Consider a 3-class classification problem ($m = 3$) with fixed conditional transition matrix:
$$
\mathbf{R} = \begin{pmatrix} 0.80 & 0.15 & 0.05 \\ 0.30 & 0.60 & 0.10 \\ 0.05 & 0.25 & 0.70 \end{pmatrix}.
$$
When the class prior vector $\boldsymbol{\pi}$ shifts from uniform balance $(\frac{1}{3}, \frac{1}{3}, \frac{1}{3})$ to severe skew $(0.8, 0.1, 0.1)$:
\begin{itemize}
  \item Macro-recall remains strictly constant at $0.7000$, as guaranteed by Proposition~\ref{prop:ovr_skew}(1);
  \item Macro-balanced-accuracy drifts from $0.7750$ to $0.7847$;
  \item Macro-G-mean drifts from $0.7700$ to $0.7783$.
\end{itemize}
Under modified error rates satisfying the column-uniform condition $R_{kj} = c_j$ for $k \neq j$ with $\boldsymbol{c} = (0.10, 0.15, 0.05)$, both macro-balanced-accuracy and macro-G-mean become strictly invariant across all priors, verifying Proposition~\ref{prop:ovr_skew}(2) (see Section~\ref{sec:exp_skew} and Figure~\ref{fig:skew} for a comprehensive empirical illustration across skew profiles).
\end{example}

\begin{table}[!htb]
\centering
\small
\begin{tabular}{l|ccc}
\hline
& \textbf{Binary} & \textbf{Multilabel} & \textbf{Multiclass (OVR)} \\
\hline
Per-class recall $\mathrm{rec}_j$ & prior-free & prior-free & prior-free \\
Per-class specificity $\mathrm{spec}_j$ & prior-free & prior-free & prior-dependent \\
\hline
Macro-recall & invariant & invariant & invariant \\
Macro-specificity & invariant & invariant & sensitive \\
Macro-BA, macro-G-mean & invariant & invariant & sensitive \\
Micro (any of the above) & --- $^{\ddagger}$ & sensitive$^{*}$ & sensitive$^{*}$ \\
Exemplar (any of the above) & --- $^{\ddagger}$ & sensitive$^{\dagger}$ & sensitive$^{\dagger}$ \\
\hline
\end{tabular}

\smallskip
\raggedright
$^{*}$ Subject to the non-degeneracy conditions of Corollary~\ref{cor:micro_skew}; micro coincides with macro when the relevant per-class rates are all equal.\\
$^{\dagger}$ Because the average is taken over a population of examples whose composition changes with the class distribution; see Proposition~\ref{prop:exemplar_skew}.\\
$^{\ddagger}$ Not applicable: with a single binary sub-problem there is nothing to aggregate across, and all three operators reduce to the binary measure itself.
\caption{Skew behaviour of the aggregation operators, separated by setting. Only the binary and multilabel column-wise cases inherit binary skew-invariance unconditionally.}
\label{tab:skew_settings}
\end{table}

Similarly, one might expect exemplar averaging to preserve skew-invariance because it operates directly on individual instances rather than class aggregates. It does not, and the reason is structural: the uniform average over instances is mathematically a mixture of per-class (or per-configuration) conditional scores, where the mixture weights are directly determined by the class prevalence.

\begin{proposition}[Exemplar averaging is not skew-invariant]
\label{prop:exemplar_skew}
Partition the examples into types $\tau \in \mathcal{T}$ according to their true label configuration, and let $\rho_\tau$ denote the proportion of examples of type $\tau$. For any per-example measure $M_i$ whose conditional expectation given the type, $\bar{M}_\tau = \mathbb{E}[M_i \mid \tau]$, is determined by the classifier's conditional behaviour, the exemplar average satisfies
$$
M^{\mathrm{exemplar}} = \sum_{\tau \in \mathcal{T}} \rho_\tau\, \bar{M}_\tau .
$$
Hence $M^{\mathrm{exemplar}}$ is invariant to the class distribution if and only if $\bar{M}_\tau$ is the same for all types $\tau$ with $\rho_\tau > 0$; otherwise it varies with $\boldsymbol{\rho}$ even when every $\bar{M}_\tau$ is held fixed.
\end{proposition}

\begin{proof}
The exemplar operator $\Upsigma_{\mathrm{n}}$ averages the per-example values uniformly over the $n$ examples. Grouping by type and applying the law of total expectation gives the stated mixture. A convex combination with weights ranging over the simplex is constant if and only if the combined values coincide.
\end{proof}

\begin{example}[Composition-induced drift in exemplar averaging]
\label{ex:exemplar_skew}
Consider a multilabel evaluation task with two distinct instance label configurations (types) $\tau_1, \tau_2 \in \mathcal{T}$, where a classifier achieves fixed per-type conditional scores $\bar{M}_{\tau_1} = 0.667$ and $\bar{M}_{\tau_2} = 0.500$ for exemplar-$F_1$. By Proposition~\ref{prop:exemplar_skew}, the global exemplar score is the linear mixture:
$$
M^{\mathrm{exemplar}} = \rho_1 \cdot 0.667 + (1 - \rho_1) \cdot 0.500,
$$
where $\rho_1 \in [0,1]$ is the relative frequency of type $\tau_1$ in the test sample. As $\rho_1$ varies from $0$ to $1$, the aggregate exemplar-$F_1$ score drifts monotonically from $0.500$ to $0.667$, driven purely by sample composition while the classifier's per-instance predictions remain identical. This confirms that for exemplar averaging, the relevant source of skew sensitivity is the distribution over label configurations $\boldsymbol{\rho}$ rather than individual class marginals.
\end{example}

Before examining how aggregation operators interact with class skew, it is instructive to analyze the algebraic relationship between skew-invariance and aggregation-decomposability, as their joint satisfaction imposes an exceptionally stringent structural constraint.

\begin{proposition}[Skew-invariance and decomposability are almost incompatible]
\label{prop:skew_decomp}
A non-constant measure is both skew-invariant and aggregation-decomposable if and only if it is a M\"obius function of recall alone or a M\"obius function of specificity alone. In particular, balanced accuracy, the G-mean and informedness are skew-invariant but not aggregation-decomposable. (MCC and Cohen's $\kappa$ are neither skew-invariant, as shown in Table~\ref{tab:skew}, nor decomposable, by Proposition~\ref{prop:which_decomposable}.)
\end{proposition}

\begin{proof}
Let $M(\mathbold{v}) = \langle \mathbold{a},\mathbold{v}\rangle/\langle \mathbold{b},\mathbold{v}\rangle$ and group the terms tied to each class, writing $P = a_1 tp + a_3 fn$, $Q = a_2 fp + a_4 tn$, $R = b_1 tp + b_3 fn$ and $S = b_2 fp + b_4 tn$, so that $M = (P+Q)/(R+S)$. Under the skew rescaling $(tp,fp,fn,tn) \mapsto (\lambda tp, \mu fp, \lambda fn, \mu tn)$ the measure becomes $(\lambda P + \mu Q)/(\lambda R + \mu S)$. Skew-invariance requires this to be independent of $\lambda, \mu > 0$; differentiating with respect to $\lambda$ and setting the result to zero gives $P(\lambda R + \mu S) = R(\lambda P + \mu Q)$, i.e.\ $\mu(PS - QR) = 0$, so $PS = QR$ identically. Hence either $Q = S = 0$, in which case $M = P/R$ is a M\"obius function of $tp$ and $fn$ and therefore of recall alone; or $P = R = 0$, in which case $M = Q/S$ is a M\"obius function of $fp$ and $tn$ and therefore of specificity alone; or $P/R = Q/S$ holds identically, which forces $M$ to be constant. Balanced accuracy, the G-mean and informedness are excluded because each depends non-trivially on both recall and specificity, so by Theorem~\ref{thm:skew} they are skew-invariant, while Proposition~\ref{prop:which_decomposable} shows none is linear-fractional.
\end{proof}

Proposition~\ref{prop:skew_decomp} has a direct methodological consequence: a proof of skew behaviour that invokes Theorem~\ref{thm:micro_macro} at the level of the composite measure $M$ would apply only to recall and specificity, and would say nothing about the balanced measures that motivate the discussion. The following corollary is therefore proved by applying Theorem~\ref{thm:micro_macro} to recall and specificity separately (each being decomposable) and then composing. This route requires no decomposability assumption on $M$ itself and consequently covers the full class of skew-invariant measures.

\begin{corollary}[Skew behaviour of macro- and micro-averaging]
\label{cor:micro_skew}
Let $M$ be a skew-invariant measure, written $M = g(\mathrm{rec}, \mathrm{spec})$ by Theorem~\ref{thm:skew}. Then:
\begin{enumerate}
  \item If the per-class rates $\mathrm{rec}_j$ and $\mathrm{spec}_j$ are prior-free (which holds in binary and multilabel settings, and in multiclass problems satisfying the uniform-error condition of Proposition~\ref{prop:ovr_skew}(2)), then $M^{\mathrm{macro}}$ is skew-invariant.
  \item $M^{\mathrm{micro}} = g\bigl(\sum_j \pi_j \mathrm{rec}_j,\; \sum_j \tilde{\pi}_j \mathrm{spec}_j\bigr)$ with $\pi_j \propto tpos_j$ and $\tilde{\pi}_j \propto tneg_j$. Consequently $M^{\mathrm{micro}}$ is skew-invariant whenever the per-class rates entering each non-degenerate argument of $g$ are all equal, and skew-sensitive otherwise, provided $g$ is strictly monotone in that argument.
\end{enumerate}
\end{corollary}

\begin{proof}
Part~1: under the stated hypothesis each $M_j = g(\mathrm{rec}_j, \mathrm{spec}_j)$ depends only on the conditional rates, and $M^{\mathrm{macro}} = \frac1m\sum_j M_j$ uses prior-free weights, so the invariance is inherited.

Part~2, which requires no hypothesis on the per-class rates: since the micro extension substitutes pooled counts into $M$, and $M$ is a function of recall and specificity evaluated at those counts, $M^{\mathrm{micro}} = g(\mathrm{rec}^{\mathrm{micro}}, \mathrm{spec}^{\mathrm{micro}})$. Recall and specificity are individually aggregation-decomposable (Proposition~\ref{prop:which_decomposable}), so Theorem~\ref{thm:micro_macro} applies to each separately:
$$
\mathrm{rec}^{\mathrm{micro}} = \frac{\sum_j tpos_j\, \mathrm{rec}_j}{\sum_j tpos_j} = \sum_j \pi_j\, \mathrm{rec}_j,
\qquad
\mathrm{spec}^{\mathrm{micro}} = \frac{\sum_j tneg_j\, \mathrm{spec}_j}{\sum_j tneg_j} = \sum_j \tilde{\pi}_j\, \mathrm{spec}_j .
$$
Each argument is a convex combination of the per-class rates with weights determined by the class distribution, whatever those rates themselves depend on. If the rates entering an argument are all equal, that convex combination equals their common value for every weight vector, and the argument is prior-free; if all non-degenerate arguments are of this kind, $M^{\mathrm{micro}}$ is skew-invariant. Conversely, if the rates entering some argument differ and $g$ is strictly monotone in it, then varying $\boldsymbol{\pi}$ changes that argument (as a convex combination of non-identical values varies over the simplex) and hence changes $M^{\mathrm{micro}}$.
\end{proof}

Two points about the statement deserve emphasis.

First, non-uniformity of the class distribution is not by itself sufficient for micro-averaging to be skew-sensitive. Unequal weights change a weighted average only if the values being averaged differ. A classifier with $\mathrm{rec}_j = 0.7$ for all $j$ has micro-recall exactly $0.7$ for every prior vector, however skewed. The condition in part~2 is therefore stated in terms of the per-class rates rather than the priors, and it is an equivalence rather than a one-directional claim.

Second, the mechanism is now visible. Micro-averaging is skew-sensitive not because the measure is, but because $\Upsigma_{\mathrm{1}}$ silently replaces prior-free per-class rates by prevalence-weighted mixtures of them before $g$ is applied. For balanced accuracy, $\mathrm{BA}^{\mathrm{micro}} = \tfrac12\bigl(\sum_j \pi_j \mathrm{rec}_j + \sum_j \tilde{\pi}_j \mathrm{spec}_j\bigr)$ is prior-dependent while $\mathrm{BA}^{\mathrm{macro}} = \tfrac{1}{2m}\sum_j(\mathrm{rec}_j + \mathrm{spec}_j)$ is not. The same substitution explains the values reported in Example~\ref{ex:counterexample}. Note also that exemplar averaging is excluded from this corollary: by Proposition~\ref{prop:exemplar_skew} it is not skew-invariant in general, for a reason unrelated to the operator's weighting of classes.

\subsection{Exact Decomposition of the Micro-Macro Gap}
\label{sec:gap_decomposition}

\begin{corollary}[Gap between micro and macro as a function of skew]
\label{cor:gap}
Let $M$ be an aggregation-decomposable measure with per-class denominators $w_j = \langle \mathbold{b}, \mathbold{v}_j\rangle$. Then
$$
M^{\mathrm{micro}} - M^{\mathrm{macro}} = \frac{\frac{1}{m}\sum_j (w_j - \bar{w})(M_j - \bar{M})}{\bar{w}}
\;=\; \frac{\mathrm{Cov}(\boldsymbol{w}, \boldsymbol{M})}{\bar{w}},
$$
where $\bar{w} = \frac{1}{m}\sum_j w_j$, $\bar{M} = M^{\mathrm{macro}}$, $\boldsymbol{M} = (M_1,\ldots,M_m)$ are the per-class measure values, and $\mathrm{Cov}$ is the empirical covariance over the $m$ classes under the uniform distribution.
\end{corollary}

\begin{proof}
By Theorem~\ref{thm:micro_macro}, $M^{\mathrm{micro}} = \sum_j w_j M_j / \sum_j w_j$. Writing $\sum_j w_j = m\bar{w}$ and expanding the numerator about the means,
$$
\sum_j w_j M_j = \sum_j (w_j - \bar{w})(M_j - \bar{M}) + \bar{M}\sum_j w_j + \bar{w}\sum_j (M_j - \bar{M})
= \sum_j (w_j - \bar{w})(M_j - \bar{M}) + m\bar{w}\bar{M},
$$
since $\sum_j (M_j - \bar{M}) = 0$. Dividing by $m\bar{w}$ gives $M^{\mathrm{micro}} = \bar{M} + \mathrm{Cov}(\boldsymbol{w}, \boldsymbol{M})/\bar{w}$, which is the claim.
\end{proof}

The normalisation by $\bar{w}$ was left implicit in earlier informal statements of this identity; it is what makes the two sides dimensionally consistent, since $w_j$ carries the units of a count while $M_j$ is dimensionless. If the weights are expressed in normalised form $\pi_j = w_j / \sum_k w_k$, the identity reads $M^{\mathrm{micro}} - M^{\mathrm{macro}} = m\,\mathrm{Cov}(\boldsymbol{\pi}, \boldsymbol{M})$.

\begin{remark}[The gap for non-decomposable measures]
\label{rem:gap_general}
Corollary~\ref{cor:gap} does not extend to measures outside Definition~\ref{def:decomposable}, and Remark~\ref{rem:counterexample} shows why: a covariance correction to $\bar{M}$ can never move the result outside $[\min_j M_j, \max_j M_j]$, yet micro-balanced-accuracy does exactly that. A substitute is available whenever the measure can be written as $M = g(u^{(1)},\ldots,u^{(K)})$ with each $u^{(k)}$ decomposable, which by Theorem~\ref{thm:skew} covers every skew-invariant measure with $K = 2$ and $(u^{(1)}, u^{(2)}) = (\mathrm{rec}, \mathrm{spec})$. Applying Corollary~\ref{cor:gap} componentwise and then expanding $g$ to second order about the macro means $\bar{u}^{(k)}$ gives
$$
M^{\mathrm{micro}} - M^{\mathrm{macro}}
= \underbrace{\sum_k \frac{\partial g}{\partial u^{(k)}}\cdot\frac{\mathrm{Cov}(\boldsymbol{w}^{(k)}, \boldsymbol{u}^{(k)})}{\bar{w}^{(k)}}}_{\text{skew term}}
\;-\; \underbrace{\tfrac{1}{2}\,\mathrm{tr}\bigl(\nabla^2 g\,\mathrm{Var}(\boldsymbol{u})\bigr)}_{\text{curvature (Jensen) term}}
\;+\; O(\|\boldsymbol{u} - \bar{\boldsymbol{u}}\|^3),
$$
where the derivatives are evaluated at $\bar{\boldsymbol{u}}$, and $\boldsymbol{w}^{(k)}$ is the weight vector supplied by Theorem~\ref{thm:micro_macro} for the $k$-th component. The first term generalises Corollary~\ref{cor:gap}; the second is a Jensen gap arising from the curvature of $g$, which is non-negative for concave $g$ such as the G-mean and vanishes identically for linear $g$.

Two observations about this expansion clarify why micro-averaged scores can fall outside the convex hull of per-class values. First, this violation of the convex-hull property is not driven by non-linear curvature in the link function $g$. Balanced accuracy uses an affine combination $g(u^{(1)}, u^{(2)}) = \frac{1}{2}(u^{(1)} + u^{(2)})$, for which the Hessian and curvature terms vanish identically; nonetheless, Example~\ref{ex:counterexample} shows micro-balanced-accuracy strictly exceeding every individual per-class score. Second, the underlying mechanism is that the components carry different weight vectors: recall is pooled with weights $\pi_j \propto tpos_j$ while specificity is pooled with weights $\tilde{\pi}_j \propto tneg_j$, and these coincide only under exact balance. In the numerical instance of Example~\ref{ex:counterexample}, $\boldsymbol{\pi} = (0.473, 0.236, 0.291)$ but $\tilde{\boldsymbol{\pi}} = (0.264, 0.382, 0.355)$. Because a single weight vector is what the convex-hull argument requires, mixing two different ones produces a value that no common weighting can reproduce. For balanced accuracy the skew term alone accounts for the entire gap of $+0.0241$ exactly; for the G-mean the split is $+0.0250$ from the skew term and $+0.0078$ from curvature, summing to $+0.0328$ against an exact gap of $+0.0327$, the residual being third order.
\end{remark}

Corollary~\ref{cor:gap} gives an interpretable decomposition: micro and macro averages agree only when class weights and per-class performance are uncorrelated. If a classifier performs better on larger classes (a common outcome when training data is imbalanced), the covariance is positive and micro exceeds macro. If it performs better on smaller classes (e.g., after oversampling), the gap reverses. The magnitude of the gap is thus a diagnostic for the interaction between class imbalance and classifier behaviour.

\subsection{Degenerate Cases: Zero Prevalence and Undefined Denominators}
\label{sec:degenerate}

Definition~\ref{def:skew} restricts attention to the domain $\mathcal{D}$ on which recall and specificity are defined, and the results above assume $\pi_j > 0$. Real test sets violate these assumptions routinely: a rare class may be absent from a fold, or a classifier may never predict some class. The framework should say what happens then, both because the aggregation operators behave asymmetrically in these cases and because the standard implementations resolve them by convention rather than by derivation.

To identify which quantities become undefined, let $tpos_j = \Upsigma_{\mathrm{m}}(\mathbold{TP} + \mathbold{FN})_j$ and $ppos_j = \Upsigma_{\mathrm{m}}(\mathbold{TP} + \mathbold{FP})_j$. Then $\mathrm{rec}_j$ is undefined exactly when $tpos_j = 0$ (class $j$ absent from the test set) and $\mathrm{prec}_j$ exactly when $ppos_j = 0$ (class $j$ never predicted); $F_{\beta,j}$ is undefined when both vanish, and $\mathrm{spec}_j$ when $tneg_j = 0$, which in multiclass requires all other classes to be absent. In general, a per-class measure is undefined precisely when its denominator $den_j = \langle \mathbold{b}, \mathbold{v}_j\rangle$ vanishes, which for the decomposable measures of Definition~\ref{def:decomposable} is a linear condition on the counts and hence directly checkable.

The three operators are affected very differently by undefined entries, representing a structural property of the framework rather than an implementation detail:
\begin{itemize}\itemsep2pt
\item $\Upsigma_{\mathrm{1}}$ (micro) pools before dividing, so its denominator is $\sum_j den_j$. This vanishes only if every per-class denominator vanishes. Micro-averaged measures are therefore well defined in almost all practical cases, including those where several per-class values are not.
\item $\Upsigma_{\mathrm{m}}$ (macro) divides before averaging and so is undefined as soon as one per-class denominator vanishes. A convention is unavoidable.
\item $\Upsigma_{\mathrm{n}}$ (exemplar) is undefined for any example whose row makes the denominator vanish (for instance, an example with no true labels and none predicted, for which per-example $F_1$ is $0/0$). This case is common in sparse multilabel data.
\end{itemize}

In practice, three conventions are in use for resolving an undefined per-class value: assign zero, assign one, or exclude the class from the average and divide by the number of defined classes. Scikit-learn exposes the choice through its \texttt{zero\_division} parameter, defaulting to zero with a warning. The three are not interchangeable, and the framework makes the difference precise. Writing $\mathcal{J} = \{j : den_j > 0\}$ for the defined classes and $u = m - |\mathcal{J}|$ for the number of undefined ones, the three conventions give
$$
\frac{1}{m}\sum_{j \in \mathcal{J}} M_j,
\qquad
\frac{1}{m}\Bigl(\sum_{j \in \mathcal{J}} M_j + u\Bigr),
\qquad
\frac{1}{|\mathcal{J}|}\sum_{j \in \mathcal{J}} M_j ,
$$
which differ by $u/m$ and by a factor $m/|\mathcal{J}|$ respectively. With $m = 14$ labels and two undefined, the zero and exclusion conventions differ by roughly $14\%$ of the mean over defined classes, which is easily enough to reverse a classifier comparison. As stated in Eq.~(\ref{eq:undef_convention}), this paper adopts the exclusion convention throughout, and we recommend it in general: it is the only one of the three that does not inject an arbitrary performance value for a class on which no evidence exists. Whichever convention is chosen, it should be reported alongside any macro- or exemplar-averaged measure, since the value is not interpretable without it.

Zero-prevalence classes also mark a genuine discontinuity in the skew analysis, which is worth separating from the conventions above. The population quantity $\mathrm{rec}_j = R_{jj}$ remains perfectly well defined as $\pi_j \to 0$, because it describes the classifier rather than the sample. Its empirical estimate, by contrast, becomes $0/0$ at $\pi_j = 0$. Macro-averaging is therefore discontinuous at the boundary of the simplex: the limit of the $m$-class macro average as $\pi_j \to 0$ is not the macro average of the $(m-1)$-class problem obtained by deleting class $j$, since the former retains a contribution from class $j$ while the latter does not. Results such as Corollary~\ref{cor:micro_skew} should accordingly be read as statements about the relative interior of the simplex, where all $\pi_j > 0$; we have stated the positivity hypothesis explicitly in Proposition~\ref{prop:ovr_skew} for this reason.

These findings establish a direct guideline for benchmarking: practitioners seeking skew-invariant evaluation should prefer macro-averaging over micro-averaging, but only in binary and multilabel settings, or in multiclass problems satisfying the uniform-error condition of Proposition~\ref{prop:ovr_skew}(2); in general multiclass problems, even macro-averaged balanced accuracy retains a residual prior-dependence. Exemplar-averaging is not an alternative here, since by Proposition~\ref{prop:exemplar_skew} it is prior-dependent for a different reason. When per-class rates differ, micro-averaging replaces them with class-weighted mixtures and therefore biases scores toward classes with larger denominators; for decomposable measures this follows the denominator weighting of Theorem~\ref{thm:micro_macro} (which for recall coincides with prevalence weighting), whereas for non-decomposable measures such as balanced accuracy the micro value need not represent any weighted average of per-class values at all (Remark~\ref{rem:counterexample}).


\subsection{Decision-Theoretic Consistency and Optimal Thresholds}
\label{sec:consistency}

A measure $M$ is consistent with a loss $L$ if the decision rule minimising the population risk $\mathbb{E}[L]$ also maximises the population value of $M$~\cite{DBLP:journals/jmlr/Hernandez-OralloFF12}. Consistency matters because it determines whether optimising $M$ during model selection or threshold tuning can be achieved by standard empirical risk minimisation, or instead requires a measure-specific procedure. To state results of this kind precisely we must work at the population level rather than with test-set counts, since the question is about which decision rule is optimal, not about the value attained on a particular sample.

Formally, let $(X, \mathbold{Y})$ be distributed on $\mathcal{X} \times \{0,1\}^m$, and write
$$
\eta_j(x) = \mathbb{P}(Y_j = 1 \mid X = x), \qquad j \in [1,m],
$$
for the per-label conditional probabilities. A (deterministic) decision rule is a map $f : \mathcal{X} \to \{0,1\}^m$ with components $f_j$. For a fixed label $j$ the population confusion-matrix rates induced by $f$ are
$$
\begin{aligned}
\mathrm{TP}_j(f) &= \mathbb{E}\bigl[f_j(X)\,\eta_j(X)\bigr], &
\mathrm{FP}_j(f) &= \mathbb{E}\bigl[f_j(X)(1-\eta_j(X))\bigr], \\
\mathrm{FN}_j(f) &= \mathbb{E}\bigl[(1-f_j(X))\,\eta_j(X)\bigr], &
\mathrm{TN}_j(f) &= \mathbb{E}\bigl[(1-f_j(X))(1-\eta_j(X))\bigr].
\end{aligned}
$$
The per-label $0/1$ risk is $R_{0/1}(f) = \frac{1}{m}\sum_j \bigl[\mathrm{FP}_j(f) + \mathrm{FN}_j(f)\bigr]$, whose Bayes rule is the familiar plug-in rule $f^{\mathrm{MAP}}_j(x) = \mathbbm{I}\bigl(\eta_j(x) > 1/2\bigr)$. We call $M^{\mathrm{micro}}$ ERM-consistent when its population maximiser is $f^{\mathrm{MAP}}$, since that is the rule standard empirical risk minimisation targets.

\begin{proposition}[Bayes-optimal decision rule for linear measures]
\label{prop:consistency}
Let $M$ be a measure whose micro extension has the linear-fractional form
$$
M^{\mathrm{micro}}(f) = \frac{\alpha\,\mathrm{TP}(f) - \beta\,\mathrm{FP}(f) - \gamma\,\mathrm{FN}(f) + \delta\,\mathrm{TN}(f)}{D},
$$
where $\alpha,\beta,\gamma,\delta \geq 0$ with $\alpha+\beta+\gamma+\delta > 0$, and the denominator $D > 0$ does not depend on the decision rule $f$. Then:
\begin{enumerate}
  \item The population maximiser of $M^{\mathrm{micro}}$ is the componentwise threshold rule
  \begin{equation}
  f^{\star}_j(x) = \mathbbm{I}\left(\eta_j(x) \;\geq\; \frac{\beta+\delta}{\alpha+\beta+\gamma+\delta}\right), \qquad j \in [1,m].
  \label{eq:bayes_threshold}
  \end{equation}
  \item $M^{\mathrm{micro}}$ is ERM-consistent if and only if
  \begin{equation}
  \alpha + \gamma \;=\; \beta + \delta .
  \label{eq:consistency_condition}
  \end{equation}
\end{enumerate}
\end{proposition}

\begin{proof}
Substituting the expectations into the numerator and taking the expectation over $X$, the objective separates into independent pointwise maximisations for each $x \in \mathcal{X}$ and each label $j$:
$$
\max_{f_j(x) \in \{0,1\}} f_j(x)\,\bigl[\alpha\eta_j(x) - \beta(1-\eta_j(x))\bigr] + (1-f_j(x))\,\bigl[-\gamma\eta_j(x) + \delta(1-\eta_j(x))\bigr].
$$
Setting $f_j(x) = 1$ is optimal if and only if the first bracket is at least the second (suppressing $x$ and $j$):
$$
\alpha\eta - \beta(1-\eta) \;\geq\; -\gamma\eta + \delta(1-\eta)
\iff \eta\,(\alpha + \gamma) \;\geq\; (1-\eta)(\beta + \delta)
\iff \eta \;\geq\; \frac{\beta+\delta}{\alpha+\beta+\gamma+\delta},
$$
where the denominator is positive by assumption. This gives Eq.~(\ref{eq:bayes_threshold}), with ties broken arbitrarily. For part~2, $t^{\star} = 1/2$ holds iff $2(\beta+\delta) = \alpha+\beta+\gamma+\delta$, i.e.\ iff $\alpha + \gamma = \beta + \delta$, and by the Bayes rule for $R_{0/1}$ this is exactly the condition that $f^{\star} = f^{\mathrm{MAP}}$.
\end{proof}

Condition~(\ref{eq:consistency_condition}) has a direct reading: the total weight the measure attaches to the two outcomes that favour a positive prediction must equal the total weight attached to the two that favour a negative one. Table~\ref{tab:consistency} applies Proposition~\ref{prop:consistency} to standard linear performance measures and asymmetric cost-weighted evaluation functions, listing their recovered numerator coefficients $(\alpha, \beta, \gamma, \delta)$, the optimal Bayes threshold $t^\star$, and their ERM-consistency status. Its upper block serves as a sanity check: accuracy and error rate satisfy~(\ref{eq:consistency_condition}) and recover the standard MAP rule ($t^\star = 1/2$), while recall and specificity return the boundary thresholds $t^\star = 0$ and $t^\star = 1$, recovering the classical fact that recall is maximized by predicting all instances positive and specificity by predicting all negative.

The primary practical value lies in the lower block, where optimal thresholds govern asymmetric risk and decision curves. For instance, in clinical risk assessment via Net Benefit~\cite{vickers2006decision} with intervention threshold $p_t$, the Bayes threshold matches the clinical preference $t^\star = p_t$ exactly. In cost-sensitive learning~\cite{elkan2001foundations} with severe false-negative penalties ($c_{\mathrm{fn}} : c_{\mathrm{fp}} = 5:1$ or $10:1$, common in fraud detection and rare-disease screening), the optimal Bayes thresholds shift to $t^\star = 1/6 \approx 0.167$ and $t^\star = 1/11 \approx 0.091$, formalizing why standard MAP classification fails under asymmetric costs. Two broader points carry special methodological weight:
\begin{enumerate}
  \item Condition~(\ref{eq:consistency_condition}) is a characterisation of an entire class of objective functions rather than a one-off derivation, allowing automated consistency checks for any newly proposed linear utility;
  \item A fixed denominator alone does not ensure ERM-consistency (recall and specificity have data-fixed denominators yet induce degenerate extreme rules), demonstrating that consistency is driven by numerator balance rather than normalisation terms.
\end{enumerate}

\begin{table}[!htb]
\centering
\small
\setlength{\tabcolsep}{5pt}
\begin{tabular}{lccccccl}
\hline
\textbf{Measure / Objective} & $\alpha$ & $\beta$ & $\gamma$ & $\delta$ & $t^{\star}$ & \textbf{ERM-cons.} & \textbf{Optimal Decision Rule} \\
\hline
\multicolumn{8}{l}{\emph{Sanity checks (symmetric / baseline metrics)}}\\
Accuracy & 1 & 0 & 0 & 1 & $1/2$ & yes & Standard MAP ($\eta > 0.5$) \\
Error rate$^{*}$ & 0 & 1 & 1 & 0 & $1/2$ & yes & Standard MAP ($\eta > 0.5$) \\
Recall & 1 & 0 & 0 & 0 & $0$ & no & Trivial all-positive \\
Specificity & 0 & 0 & 0 & 1 & $1$ & no & Trivial all-negative \\
\hline
\multicolumn{8}{l}{\emph{Asymmetric cost and utility measures (threshold not obvious by inspection)}}\\
Weighted Accuracy ($w$) & $w$ & 0 & 0 & $1{-}w$ & $1 - w$ & iff $w = 1/2$ & Preference threshold at $1 - w$ \\
Clinical Net Benefit ($p_t$)~\cite{vickers2006decision} & 1 & $\frac{p_t}{1-p_t}$ & 0 & 0 & $p_t$ & iff $p_t = 1/2$ & Clinical threshold at $p_t$ \\
Cost-sensitive ($c_{\mathrm{fn}}{:}c_{\mathrm{fp}} = 5{:}1$)~\cite{elkan2001foundations} & 0 & 1 & 5 & 0 & $1/6$ & no & High-recall threshold at $0.167$ \\
Cost-sensitive ($c_{\mathrm{fn}}{:}c_{\mathrm{fp}} = 10{:}1$)~\cite{elkan2001foundations} & 0 & 1 & 10 & 0 & $1/11$ & no & Screening threshold at $0.091$ \\
General Linear Cost-Utility & $\alpha$ & $\beta$ & $\gamma$ & $\delta$ & $\frac{\beta+\delta}{\alpha+\beta+\gamma+\delta}$ & iff $\alpha{+}\gamma = \beta{+}\delta$ & Cost-adjusted Bayes threshold \\
\hline
\end{tabular}

\smallskip
\raggedright
$^{*}$ Treated as a quantity to be maximised after sign reversal.
\caption{Recovered linear coefficients, optimal Bayes thresholds $t^{\star}$, and ERM-consistency status under Proposition~\ref{prop:consistency}. All theoretical thresholds are evaluated against exhaustive grid search on real data in Section~\ref{sec:exp_threshold}.}
\label{tab:consistency}
\end{table}

When a measure's denominator depends on the predictions (as in precision and $F_\beta$, where $ppos = \Upsigma_{\mathrm{1}}(\mathbold{\hat{Y}})$ varies with $f$), the simple pointwise argument of Proposition~\ref{prop:consistency} no longer applies. These measures exhibit qualitatively distinct Bayes optimal behavior in the literature: the population maximizer of micro-$F_\beta$ remains a threshold rule on $\eta_j(x)$, but the optimal threshold is coupled to the optimal $F_\beta$ value itself, requiring iterative or plug-in estimation~\cite{koyejo2014consistent,zhang2014review}. For macro-$F_1$, the optimal decision rule demands class-specific thresholds~\cite{koyejo2014consistent}. For exemplar (instance-wise) $F_\beta$ in multilabel problems, the optimal predictor is not a per-label threshold rule at all and requires dedicated combinatorial inference over label subsets, since the evaluation metric couples labels within each example~\cite{dembczynski2011exact}. The indicator-matrix framework cleanly delineates this structural boundary, separating measures amenable to closed-form Bayes thresholding from those requiring adaptive or joint inference.

The aggregation operator does not fix the underlying loss function (which is governed by the numerator coefficients), but instead specifies the weighting scheme applied across classes and instances. For the linear measures of Proposition~\ref{prop:consistency}, this weighting operates as follows:
\begin{itemize}
\item $\Upsigma_{\mathrm{1}}$ (micro) assigns uniform weight $w_{ij}^{\mathrm{micro}} = \frac{1}{n m}$ to every individual label decision $(i,j)$. Aggregating across instances assigns total weight $\frac{den_j}{\sum_k den_k}$ to class $j$, biasing evaluation toward majority classes. For measures satisfying~(\ref{eq:consistency_condition}), this optimizes the standard uniform per-label risk; for asymmetric measures, it shifts the Bayes threshold to $t^\star$ in Eq.~(\ref{eq:bayes_threshold}).
\item $\Upsigma_{\mathrm{m}}$ (macro) assigns equal weight $\frac{1}{m}$ to each class, giving instance $i$ in class $j$ the inverse-prevalence weight $w_{ij}^{\mathrm{macro}} = \frac{1}{m \cdot den_j}$. This compensates for class imbalance by upweighting rare classes, inducing a class-reweighted risk whose Bayes decision rule requires per-class threshold tuning~\cite{koyejo2014consistent}.
\item $\Upsigma_{\mathrm{n}}$ (exemplar) assigns equal weight $\frac{1}{n}$ to each instance, giving label $j$ on instance $i$ the per-example normalized weight $w_{ij}^{\mathrm{exemp}} = \frac{1}{n \cdot den_i}$ (where $den_i$ is the example's active support). This couples label decisions per instance, making it suitable when predictions are judged holistically per example, though optimal inference no longer reduces to per-label thresholding~\cite{dembczynski2011exact}.
\end{itemize}
These results yield two concrete guidelines for benchmarking and deployment:
\begin{enumerate}
  \item \textbf{Aggregation choice:} Select the operator that aligns with the downstream unit of evaluation---decisions over individual labels ($\Upsigma_{\mathrm{1}}$), parity across classes ($\Upsigma_{\mathrm{m}}$), or holistic instance-level performance ($\Upsigma_{\mathrm{n}}$);
  \item \textbf{Training alignment:} When evaluating with measures that violate ERM-consistency ($\alpha+\gamma \neq \beta+\delta$), standard models trained under $0/1$ loss are sub-optimal, and decision thresholds must be calibrated explicitly to $t^\star$ via Eq.~(\ref{eq:bayes_threshold}) or cost-sensitive training.
\end{enumerate}

\subsection{Ranking Equivalence and Curve-Based Evaluation}
\label{sec:ranking_curves}

Two measures are ranking-equivalent if they always produce the same ordering of classifiers, i.e., one is a strictly monotone transformation of the other on the relevant domain.

\begin{proposition}[Accuracy and error rate are ranking-equivalent]
For any two classifiers evaluated on the same test set, accuracy$_1 \geq$ accuracy$_2$ if and only if error rate$_1 \leq$ error rate$_2$.
\end{proposition}

\begin{proof}
Error rate $= 1 -$ accuracy, so the ranking is reversed. The two measures are thus ranking-equivalent up to direction.
\end{proof}

A less obvious case is Cohen's $\kappa$ and accuracy. Since $\kappa = (\mathrm{acc} - p_e)/(1-p_e)$ and $p_e$ depends on the predicted class distribution (which varies across classifiers), $\kappa$ is not in general a monotone function of accuracy: two classifiers with the same accuracy may differ in $\kappa$ if their predicted class distributions differ. Ranking equivalence between $\kappa$ and accuracy therefore holds only conditional on the predicted class distribution being fixed, a condition that is rarely satisfied in practice. Within the indicator-matrix framework, this is seen directly: $p_e = \frac{1}{n^2}\sum_j [\Upsigma_{\mathrm{m}}(\mathbold{Y})]_j \cdot [\Upsigma_{\mathrm{m}}(\mathbold{\hat{Y}})]_j$, which depends on the column sums of $\mathbold{\hat{Y}}$ and therefore on the classifier.

\begin{proposition}[ROC curve and AUC]
\label{prop:roc}
For a binary classifier with score output and threshold $\theta \in (0,1)$, let $\mathbold{\hat{Y}}(\theta)$ denote the indicator vector obtained from Eq.~(\ref{eq:threshold}). The ROC curve is the parametric curve:
\begin{align*}
\mathrm{FPR}(\theta) &= \frac{\Upsigma_{\mathrm{1}}(\mathbold{FP}(\theta))}{\Upsigma_{\mathrm{1}}(\mathbold{FP}(\theta))+\Upsigma_{\mathrm{1}}(\mathbold{TN}(\theta))}, \\
\mathrm{TPR}(\theta) &= \frac{\Upsigma_{\mathrm{1}}(\mathbold{TP}(\theta))}{\Upsigma_{\mathrm{1}}(\mathbold{TP}(\theta))+\Upsigma_{\mathrm{1}}(\mathbold{FN}(\theta))},
\end{align*}
and $\mathrm{ROC}(\theta) = \bigl(\mathrm{FPR}(\theta),\;\mathrm{TPR}(\theta)\bigr)$,
and the area under the ROC curve (AUC) is $\int_0^1 \mathrm{TPR}(\theta)\,\mathrm{d}\,\mathrm{FPR}(\theta)$. In multiclass settings under OVA, column-wise aggregation $\Upsigma_{\mathrm{m}}$ yields one ROC curve per class, and the unweighted average of per-class AUC values gives macro-AUC~\cite{hand2001simple}.
\end{proposition}

Proposition~\ref{prop:roc} shows that the entire ROC apparatus is a consequence of varying the threshold $\theta$ in Eq.~(\ref{eq:threshold}) and applying $\Upsigma_{\mathrm{1}}$ at each value. All points on the ROC curve correspond to a particular choice of threshold, and choosing a fixed $\theta$ (e.g., $\theta = 0.5$) picks a single point on the curve: the operating point used for computing precision, recall, and $F_1$ in Table~\ref{tab:extensions}. AUC integrates the performance over all operating points~\cite{fawcett2006introduction,flach2003geometry} and is therefore threshold-independent, unlike the measures in Table~\ref{tab:extensions}, which are all threshold-specific. Macro-AUC under OVA is also an instance of $\Upsigma_{\mathrm{m}}$ (the same operator used for macro-precision and macro-recall) applied to the AUC formula rather than a fixed-threshold measure.

While standard multiclass AUC aggregates over One-vs-All binary ROC curves, ordinal classification requires evaluating performance respecting the natural class ordering. Within our framework, this corresponds to applying column-wise aggregation $\Upsigma_{\mathrm{m}}$ across the cumulative binary threshold representations defined in Eq.~(\ref{eq:cumulative}).

\begin{proposition}[Ordinal AUC as column-wise aggregation]
\label{prop:ordinal_auc}
Let $\mathbold{\hat{P}} \in [0,1]^{n \times m}$ be a probabilistic classifier's output and define, for each threshold $j \in [1, m-1]$, a binary classification problem with labels $\mathbf{y}^{(j)}_i = \mathbbm{I}(k_i \geq j)$ and scores $\hat{p}^{(j)}_i = \hat{P}_{i,j}$. Let $\mathrm{AUC}^{(j)}$ denote the area under the ROC curve for threshold $j$. Then the ordinal AUC,
$$
\mathrm{AUC}^{\mathrm{ord}} = \frac{1}{m-1}\sum_{j=1}^{m-1} \mathrm{AUC}^{(j)},
$$
is the unweighted average of per-column AUCs induced by the cumulative binary encoding of Eq.~(\ref{eq:cumulative}), and is therefore an instance of $\Upsigma_{\mathrm{m}}$ applied to the binary AUC formula.
\end{proposition}

\subsection{Dominance Relations}
\label{sec:dominance}
Building directly on the $m(m-1)$ degrees-of-freedom structure established in Section~\ref{sec:redundancy}, we can now formalize when one classifier strictly outperforms another across multiple performance measures simultaneously. Classifier~A \emph{dominates} classifier~B under measure $M$ if $M(A) \geq M(B)$ for all possible test sets drawn from the same distribution. A stronger notion is \emph{measure-wise dominance}: A dominates B if $M(A) \geq M(B)$ for every measure $M$ in Table~\ref{tab:extensions} simultaneously. The reduction below is specific to the multiclass case: in multilabel problems there is no single $m \times m$ confusion matrix, and the exemplar extensions are not functions of one, so measure-wise dominance does not reduce in the same way. In a multiclass problem, Proposition~\ref{prop:dof} shows that the measures in Table~\ref{tab:extensions} obtained via $\Upsigma_{\mathrm{1}}$ or $\Upsigma_{\mathrm{m}}$ are functions of the same $m(m-1)$ confusion matrix entries, so measure-wise dominance reduces to element-wise dominance of the (normalised) confusion matrix: A dominates B if and only if $A$'s confusion matrix assigns at least as much probability to correct predictions and at most as much to each type of error as $B$'s, for every cell simultaneously.

In binary classification ($m=2$, two degrees of freedom), this reduces to dominance in the (recall, specificity) space, which is the single-operating-point analogue of \emph{ROC dominance}: A dominates B if and only if $A$'s ROC curve lies pointwise above $B$'s (Proposition~\ref{prop:roc}). The result is clean because two degrees of freedom span a two-dimensional space in which a total order on the Pareto frontier is easy to characterise.

For $m > 2$, the dominance relation becomes $m(m-1)$-dimensional and Pareto-optimal classifiers form a higher-dimensional surface. In this case there typically exists no single classifier that dominates all others across all measures simultaneously, which is the formal justification for reporting multiple complementary measures rather than a single summary statistic: different measures emphasise different cells of the confusion matrix, and the choice among them should reflect the cost structure of the application, which the framework encodes explicitly via the cost matrix $\mathbold{C}$ of Section~\ref{sec:cost}.

\section{Empirical Illustrations}
\label{sec:experiments}

The experiments fall into two groups. Sections~\ref{sec:exp_skew}--\ref{sec:tnorm_emp} are synthetic illustrations that instantiate the analytic results of Section~\ref{sec:theory} in controlled settings where the ground truth is stipulated; they are reproducible from the indicator-matrix computations alone and require no external data. Sections~\ref{sec:real_exp}--\ref{sec:exp_floors} use real datasets across multilabel, binary decision-theoretic, and multiclass domains to demonstrate the practical consequences of the theoretical findings. In several cases we report the theoretical prediction and the measured value side by side, so that agreement (or its absence) is visible rather than asserted. Code and the full numerical output are available in the supplementary material.

\subsection{Micro vs.\ Macro under Class Skew}
\label{sec:exp_skew}

We construct a 4-class classification problem with fixed conditional rates (confusion-matrix rows $R_j$), i.e., the probability that a classifier assigns class $c_k$ to an example of true class $c_j$ is held constant across all trials. We then vary the prevalence $\pi_1$ of the majority class from $0.25$ (uniform) to $0.70$ (strongly skewed), distributing the remaining probability equally among the three minority classes. All quantities are computed analytically from the rates, eliminating sampling noise.

Figure~\ref{fig:skew} shows micro and macro averages for four measures. The results separate exactly as Section~\ref{sec:skew_settings} predicts, and the distinction between the two flat-looking curves is instructive. Macro-recall is exactly flat, as required by Proposition~\ref{prop:ovr_skew}(1), since $\mathrm{rec}_j = R_{jj}$ carries no dependence on the priors. Macro-balanced-accuracy is nearly, but not exactly, flat: it drifts by a small amount across the range of $\pi_1$, because in a multiclass one-vs-rest decomposition the specificity term is a prior-weighted average of the other classes' correct-rejection rates (see Eq.~(\ref{eq:ovr_spec})). This residual drift is not numerical noise but the quantitative signature of Proposition~\ref{prop:ovr_skew}(2); it would vanish identically if the classifier distributed its errors uniformly across the competing classes. The micro counterparts of both measures drift substantially, as Corollary~\ref{cor:micro_skew}(2) predicts, the per-class rates here being unequal. $F_1$ and MCC are skew-sensitive under both aggregation schemes, though the magnitude of drift differs substantially between micro and macro.

\begin{figure}[!htb]
  \centering
  \includegraphics[width=.8\linewidth]{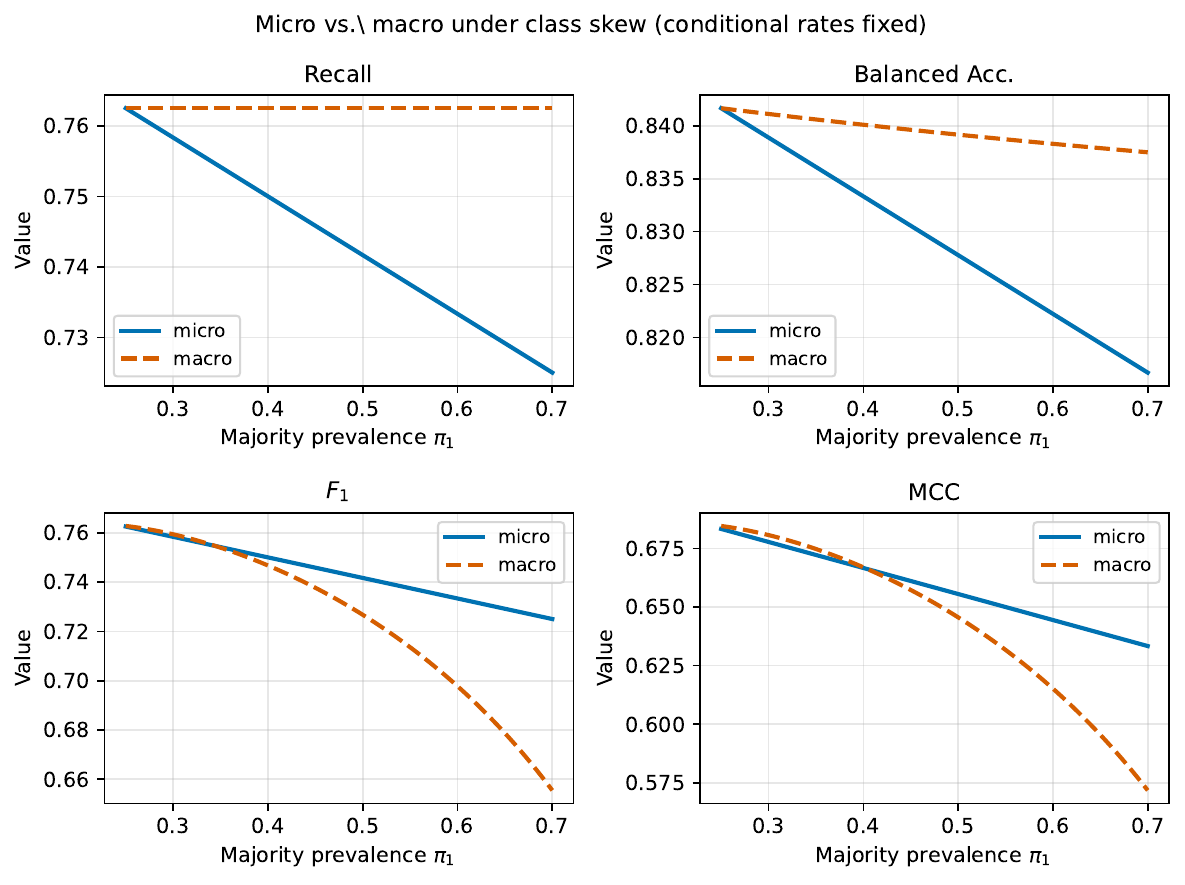}
  \caption{Micro vs.\ macro averages as a function of majority-class prevalence $\pi_1$. Conditional rates are held fixed throughout. Macro-recall is exactly flat (Proposition~\ref{prop:ovr_skew}(1)); macro-balanced accuracy is only approximately flat, retaining a small drift through the one-vs-rest specificity term (Proposition~\ref{prop:ovr_skew}(2)). Both drift markedly under micro-averaging (Corollary~\ref{cor:micro_skew}). $F_1$ and MCC are skew-sensitive under both schemes.}
  \label{fig:skew}
\end{figure}

\subsection{Micro-Macro Gap Equals the Covariance}

To illustrate Corollary~\ref{cor:gap} empirically, we generate 1{,}200 random 5-class classifiers by drawing class priors $\boldsymbol{\pi}$ from a symmetric Dirichlet distribution and confusion-matrix rates from a Dirichlet with concentration 3 (encouraging moderately good classifiers). For each trial we compute, analytically, both $M^{\mathrm{micro}} - M^{\mathrm{macro}}$ and $\mathrm{Cov}(\mathbf{w}, \mathbf{M})$ using the measure-specific denominator weights: $w_j \propto tpos_j$ for recall and $w_j \propto 2tp_j + fp_j + fn_j$ for $F_1$.

Figure~\ref{fig:cov} shows that all points lie on the $y = x$ diagonal to numerical precision, numerically reflecting the identity for both measures. The range of the gap is wider for recall ($\pm 0.15$) than for $F_1$ ($\pm 0.05$ for typical classifiers), reflecting the fact that $F_1$'s denominator weights partially absorb skew variation. Crucially, points appear on both sides of zero: when the classifier performs best on the majority class (positive covariance), micro exceeds macro; when it performs best on minority classes (negative covariance), micro falls below macro.

\begin{figure}[!htb]
  \centering
  \includegraphics[width=0.8\linewidth]{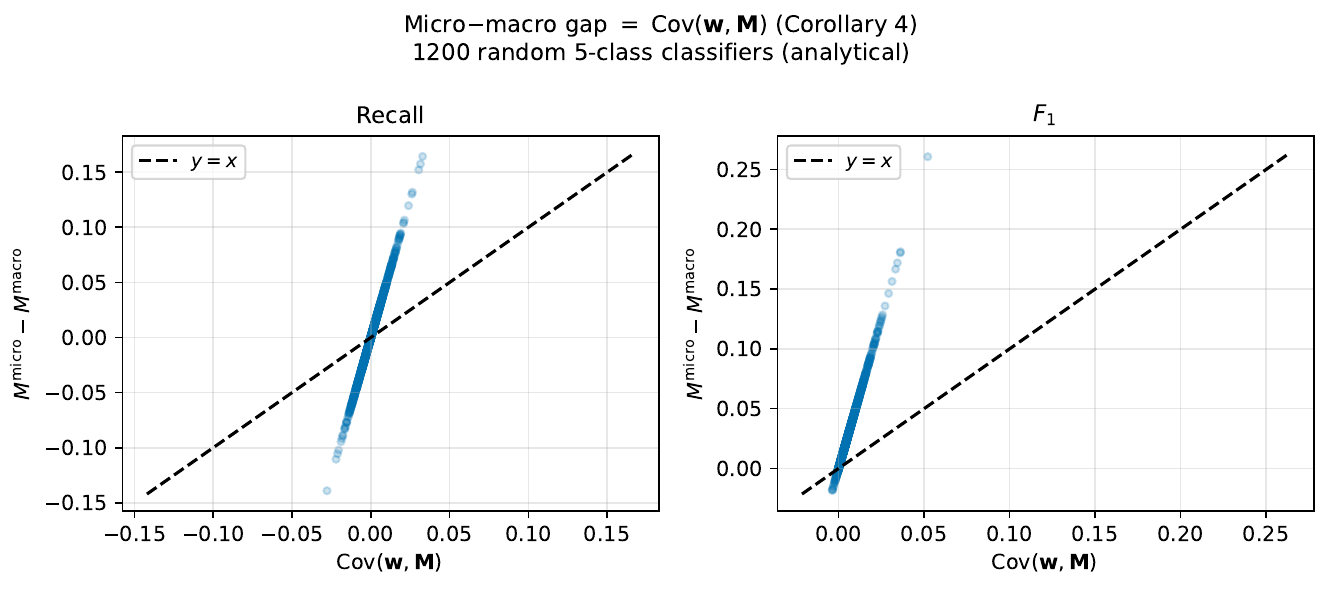}
  \caption{Micro$-$macro gap against $\mathrm{Cov}(\mathbf{w},\mathbf{M})$ for 1{,}200 random 5-class classifiers. All points lie on $y=x$ (dashed), in exact agreement with Corollary~\ref{cor:gap}. Weights are measure-specific: $w_j \propto tpos_j$ for recall and $w_j \propto 2tp_j+fp_j+fn_j$ for $F_1$.}
  \label{fig:cov}
\end{figure}

\subsection{T-norm Ordering under Soft Ground Truth}
\label{sec:tnorm_emp}

We construct a 3-class problem where both the ground-truth matrix $\mathbold{Y}$ and the predicted matrix $\mathbold{\hat{Y}}$ are real-valued. The true labels are soft, generated from a Dirichlet distribution centred at the true class with concentration parameter $\alpha$ controlling label certainty: small $\alpha$ produces near-uniform rows (high annotator uncertainty) and large $\alpha$ produces near-one-hot rows (confident annotation). Predicted probabilities are generated by a fixed noisy classifier whose outputs are also smoothed via a Dirichlet with a peak at the predicted class.

Figure~\ref{fig:tnorm} reports recall, precision, and mean $\mathbold{TP}$ entry across the three $t$-norms. The rightmost panel exhibits the pointwise ordering $T_L \leq T_P \leq T_M$ for the raw $\mathbold{TP}$ values: Łukasiewicz is the most conservative, Gödel the most lenient, and product lies between. As explained in Section~\ref{sec:soft}, the same ordering holds simultaneously for $\mathbold{FP}$, $\mathbold{FN}$ and $\mathbold{TN}$, because all four matrices are produced by the same $t$-norm.

The derived measures do not inherit this ordering, and the direction of the discrepancy depends on the regime. Three behaviours occur. When both matrices concentrate on high memberships, the Łukasiewicz $t$-norm gives the highest precision and recall (the reverse of the confusion-matrix ordering) because thresholding at $a + b > 1$ removes proportionally more mass from the disagreement cells than from $\mathbold{TP}$. In the intermediate regime the ordering is measure-dependent: precision follows $T_L \le T_P \le T_M$ while recall follows the opposite order, so no single $t$-norm is uniformly stricter even for one pair of measures. In the diffuse regime both measures again increase from $T_M$ to $T_L$, as in the concentrated regime. The mean $\mathbold{TP}$ panel is of course unaffected, obeying $T_L \leq T_P \leq T_M$ in every regime. This illustrates the point made in Section~\ref{sec:soft}: the $t$-norm ordering is a property of the confusion matrices themselves and transfers to no derived ratio, so a $t$-norm cannot be described as globally lenient or strict at the level of measures.

\begin{figure*}[!htb]
  \centering
  \includegraphics[width=.9\linewidth]{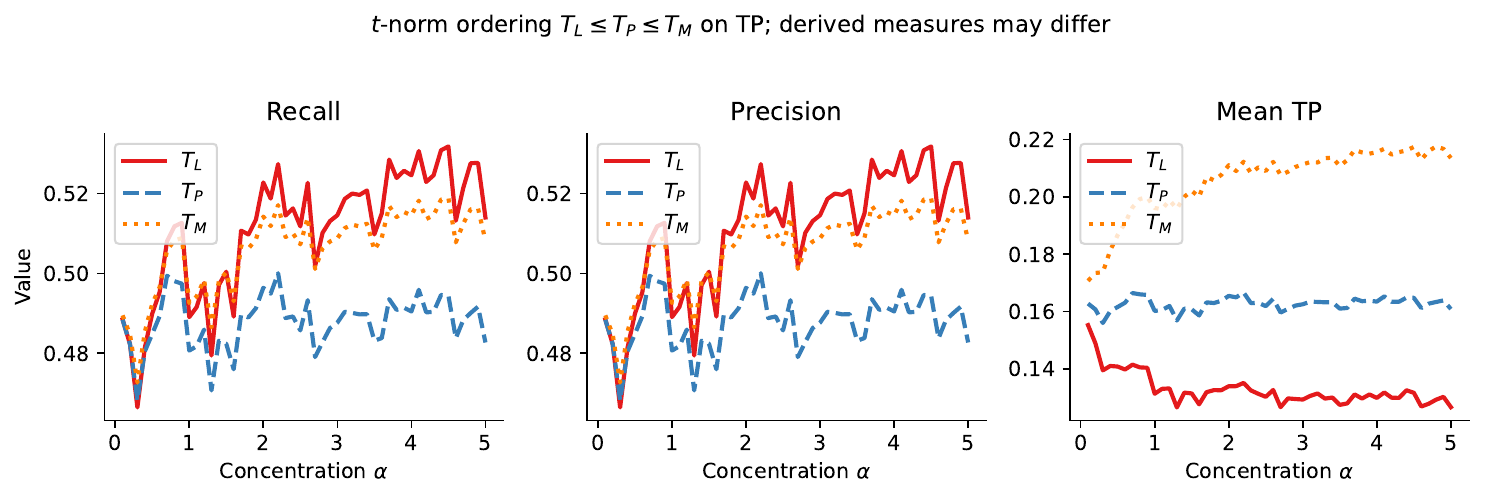}
  \caption{Effect of $t$-norm choice on recall, precision, and mean $\mathbold{TP}$ under soft ground truth ($\alpha$ controls annotation certainty). The $\mathbold{TP}$ ordering $T_L \leq T_P \leq T_M$ is exact (right panel), while derived measures need not preserve it (left and centre panels).}
  \label{fig:tnorm}
\end{figure*}

\subsection{Real Multilabel Data: Aggregation Scheme Determines Classifier Selection}
\label{sec:real_exp}

To evaluate the practical consequences of the aggregation choice on model selection, we compare classifier rankings across different aggregation schemes on real multilabel benchmarks.

We evaluate two standard multilabel benchmarks from the MULAN repository~\cite{tsoumakas2011mulan,elisseeff2001kernel}: \emph{yeast} ($2{,}417$ examples, $103$ features, $14$ labels, prevalence range $0.014$--$0.751$, cardinality $4.24$; strongly imbalanced) and \emph{emotions} ($593$ examples, $72$ features, $6$ labels, prevalence range $0.250$--$0.445$, cardinality $1.87$; nearly balanced). Five Binary Relevance (BR) classifiers are trained with 10-fold cross-validation: logistic regression (BR-LR), support vector machines (BR-SVM), random forests (BR-RF), $k$-nearest neighbours (BR-kNN), and decision trees (BR-DT). All measures are computed on the same predicted matrices via the indicator-matrix framework, ensuring that ranking differences stem solely from the aggregation operator. In yeast, all label- and example-level denominators for recall and $F_\beta$ remain strictly positive, so the convention in Eq.~(\ref{eq:undef_convention}) does not affect $F_1$ rankings (precision undefinedness is analyzed in Section~\ref{sec:exp_conventions}).

Figure~\ref{fig:rank_yeast} shows the rankings on yeast across ten measure--aggregation combinations. While micro and exemplar measures agree across classifiers, the macro operator produces a distinct ordering: exemplar-$F_1$, exemplar-Jaccard, exemplar-BA, and micro-$F_1$ induce the identical ranking (BR-SVM, BR-RF, BR-LR, BR-kNN, BR-DT), whereas macro-$F_1$ and macro-recall rank BR-DT first (which ranks last under the other eight measures).

Two pairs attain complete rank reversal ($\tau = -1.000$): macro-$F_1$ against micro-precision and macro-$F_1$ against macro-precision (Figure~\ref{fig:tau_yeast}). Furthermore, macro-$F_1$ yields $\tau \leq -0.600$ against every micro and exemplar measure, showing that rank inversion stems systematically from the aggregation operator rather than an isolated metric pairing.

This divergence reflects the weighting mechanism analyzed in Section~\ref{sec:skew_settings}: on yeast, macro-$F_1$ weights all labels equally, favoring classifiers that prioritize rare classes even at the expense of per-example fidelity ($\mathrm{macro}\text{-}F_1 = 0.398$ for BR-DT, highest among models). Conversely, exemplar-$F_1$ weights examples uniformly, evaluating instance-level label recovery ($\mathrm{exemplar}\text{-}F_1 = 0.526$ for BR-DT, lowest among models). The choice between macro and exemplar measures thus embodies the target evaluation unit (label parity versus instance-level accuracy), although optimal prediction for exemplar-$F_1$ requires joint subset inference rather than per-label thresholding~\cite{dembczynski2011exact}.

In contrast, on the nearly balanced emotions benchmark, the same ten measures yield strong consensus across all classifiers: the minimum pairwise rank correlation is $\tau = +0.200$, with $\tau(\text{macro-}F_1, \text{exemplar-}F_1) = +0.800$ (Figures~\ref{fig:rank_emotions} and~\ref{fig:tau_emotions}). As predicted by Corollary~\ref{cor:micro_skew} and Section~\ref{sec:skew_settings}, the aggregation scheme becomes decisive primarily under label skew.

Benchmark reporting must therefore explicitly declare the aggregation operator rather than quoting unqualified metrics (such as a generic ``$F_1$'' or ``precision''). Because the rank correlation between aggregation operators can drop to $\tau = -1.000$ under severe imbalance, omitting the aggregation level introduces substantial ambiguity and can invalidate comparisons across studies. More fundamentally, the choice among $\Upsigma_{\mathrm{1}}$, $\Upsigma_{\mathrm{m}}$, and $\Upsigma_{\mathrm{n}}$ is not a cosmetic implementation detail but a decision about the primary unit of evaluation: micro-averaging optimizes pooled decision accuracy, macro-averaging enforces parity across label classes regardless of rarity, and exemplar-averaging reflects the instance-level completeness of predicted label sets. Consequently, the aggregation operator should be specified \emph{a priori} based on the deployment objective rather than selected \emph{post hoc} to favor a particular classifier.

\begin{figure*}[!htb]
  \centering
  \begin{subfigure}[b]{0.62\textwidth}
  \includegraphics[width=\linewidth]{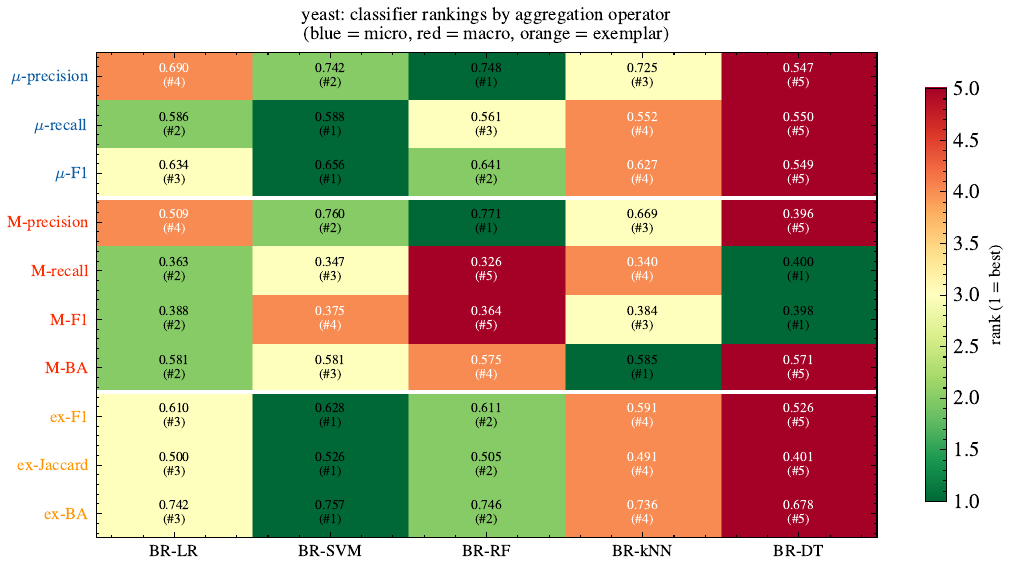}
  \caption{Rankings on yeast (imbalanced). Green = best, red = worst. Micro and exemplar measures agree; macro-$F_1$ and macro-recall invert the ordering, ranking BR-DT first where every other measure ranks it last.}
  \label{fig:rank_yeast}
  \end{subfigure}
  \hfill
  \begin{subfigure}[b]{0.36\textwidth}
  \includegraphics[width=\linewidth]{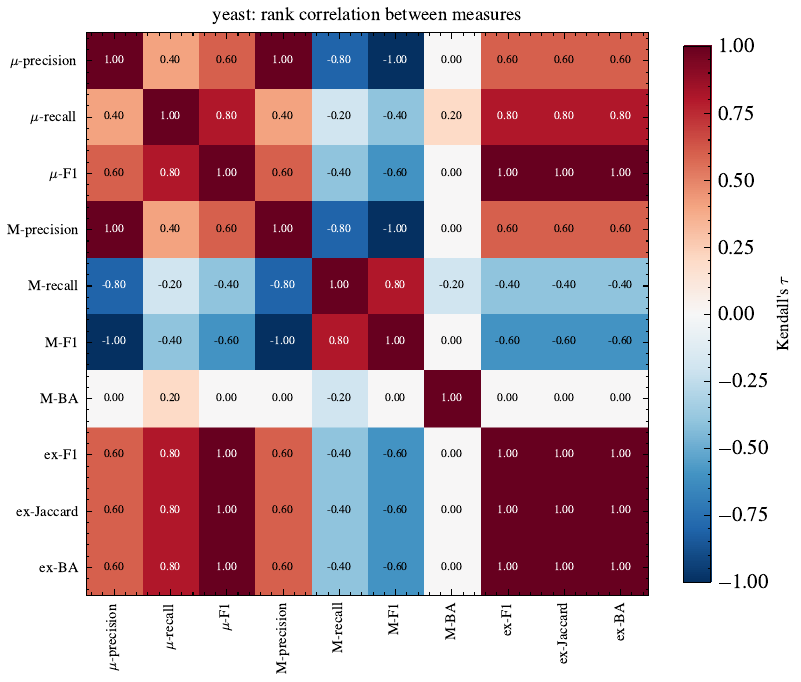}
  \caption{Kendall's $\tau$ between all pairs on yeast. Blue = disagreement. Macro-$F_1$ reaches $\tau = -1.000$ against micro- and macro-precision and $\tau \leq -0.600$ against every micro and exemplar measure.}
  \label{fig:tau_yeast}
  \end{subfigure}

  \medskip
  \begin{subfigure}[b]{0.62\textwidth}
  \includegraphics[width=\linewidth]{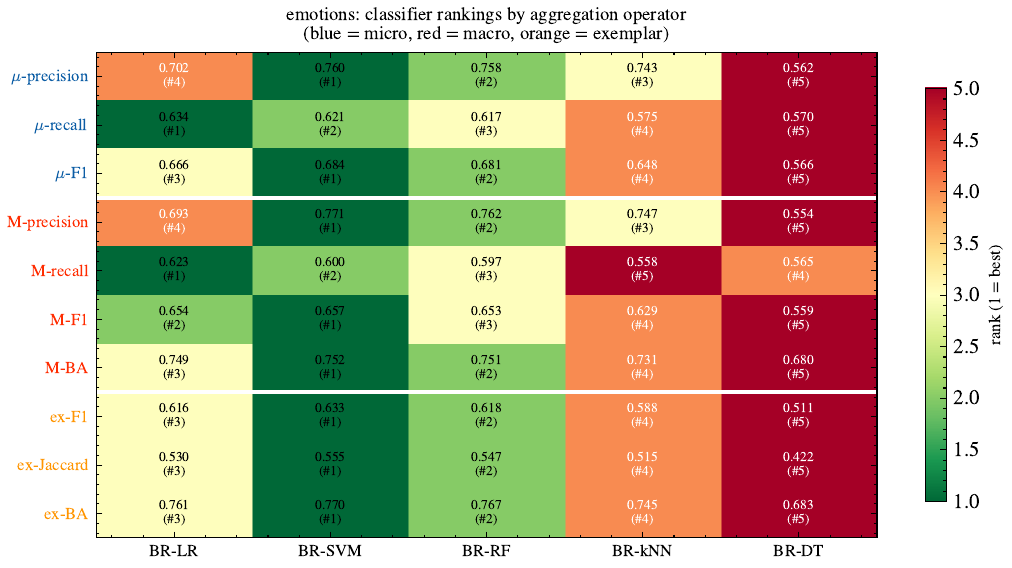}
  \caption{Rankings on emotions (nearly balanced): the same ten measures now agree almost everywhere.}
  \label{fig:rank_emotions}
  \end{subfigure}
  \hfill
  \begin{subfigure}[b]{0.36\textwidth}
  \includegraphics[width=\linewidth]{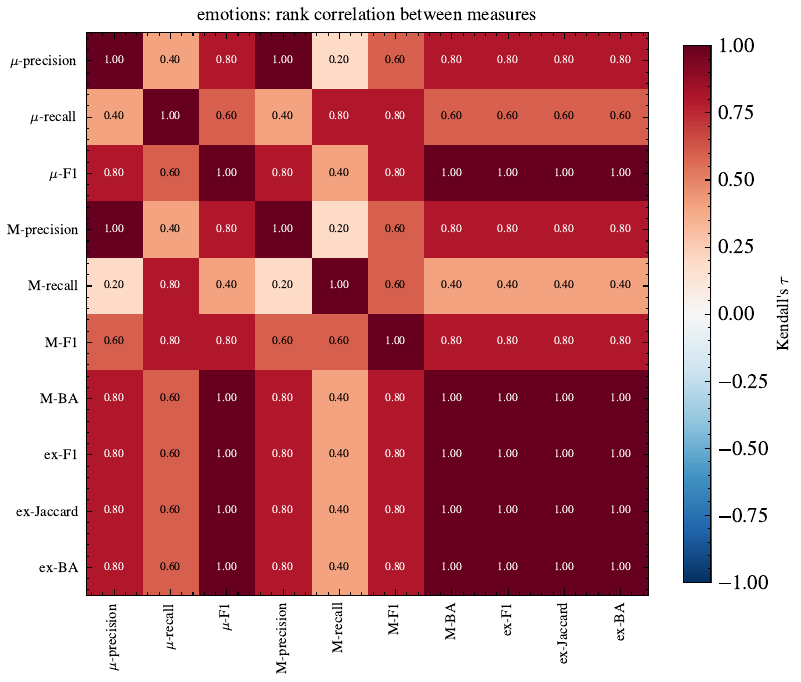}
  \caption{Kendall's $\tau$ on emotions: minimum over all pairs is $+0.200$, so no aggregation operator dissents.}
  \label{fig:tau_emotions}
  \end{subfigure}
  \caption{Aggregation disagreement is driven by label imbalance. The upper row is the imbalanced dataset, the lower row the balanced one; the measures, classifiers and protocol are identical.}
  \label{fig:multilabel}
\end{figure*}

\subsection{Sensitivity to the Undefined-Value Convention}
\label{sec:exp_conventions}

Section~\ref{sec:degenerate} argues that the treatment of undefined per-class and per-example values is a reporting choice with quantitative consequences. We measure it on yeast and emotions with the five classifiers of Section~\ref{sec:real_exp}, comparing the exclusion convention adopted in Eq.~(\ref{eq:undef_convention}) with the substitution of $0$ used by default in scikit-learn.

The pattern is exactly as the analysis predicts. Recall and $F_\beta$ have no undefined entries on either dataset, so their macro and exemplar values are convention-independent. Precision does: on yeast six per-class and $82$ per-example precision denominators vanish across the five classifiers, and on emotions $257$ per-example denominators vanish, in each case because a classifier predicts no label at all for some example or never predicts some label. The resulting discrepancies are not negligible: macro-precision shifts by up to $0.110$ on yeast and exemplar-precision by up to $0.087$ on emotions, on measures whose values are themselves below $0.8$.

While relative classifier rankings on these particular benchmarks remain stable under both conventions, a numerical shift of up to $0.110$ in macro-precision exceeds the performance margins that typically separate competing algorithms in published benchmarks (such as the gaps in Figure~\ref{fig:rank_yeast}). Without explicit reporting of the convention used, such discrepancies remain invisible to readers and can easily invert model selection when classifiers achieve close performance. This reinforces the recommendation of Section~\ref{sec:degenerate}: benchmark reports must state the undefined-value convention explicitly and prefer exclusion, which avoids imputing arbitrary performance values for classes or instances lacking empirical evidence.

\subsection{Threshold Optimality on Real Data}
\label{sec:exp_threshold}

Proposition~\ref{prop:consistency} predicts that for any linear-numerator measure with prediction-independent denominator, the Bayes-optimal decision threshold is $t^\star = (\beta+\delta)/(\alpha+\beta+\gamma+\delta)$. We evaluate this prediction on two medical screening benchmarks with contrasting sample sizes and imbalance profiles:
\begin{enumerate}
\item \textbf{Breast Cancer Wisconsin} ($n_{\mathrm{total}} = 569$, $n_{\mathrm{test}} = 228$, minority prevalence $0.373$)~\cite{street1993nuclear}.
\item \textbf{Mammography} ($n_{\mathrm{total}} = 11{,}183$, $n_{\mathrm{test}} = 4{,}474$, minority prevalence $0.023$)~\cite{woods1993comparative}.
\end{enumerate}
For each benchmark, a logistic regression model produces the conditional probability estimates $\eta(x) = \hat{P}(Y=1 \mid X=x)$ on the test set, and we evaluate operating thresholds across a fine grid of 2{,}001 points $t \in [0, 1]$ with step $0.0005$.

Table~\ref{tab:exp_threshold} presents the results. The top block tests the sanity checks ($\alpha+\gamma = \beta+\delta$): accuracy and error rate are consistent with standard 0--1 loss ($t^\star = 0.500$), while recall ($t^\star = 0.000$) and specificity ($t^\star = 1.000$) fail consistency degenerately. The bottom block tests asymmetric cost and utility objectives: weighted accuracy ($w=0.75 \Rightarrow t^\star = 0.250$), clinical net benefit ($p_t = 0.333 \Rightarrow t^\star = 0.333$), cost-sensitive evaluation with 5:1 and 10:1 false-negative-to-false-positive penalty ratios ($t^\star = 0.167$ and $t^\star = 0.091$, respectively), and a general linear cost-utility objective ($t^\star = 0.167$).

\begin{table}[!htb]
\centering
\small
\setlength{\tabcolsep}{3.2pt}
\begin{tabular}{l|c|c|c|cc|cc}
\hline
\textbf{Measure / Objective} & $(\alpha,\beta,\gamma,\delta)$ & $t^{\star}$ & \textbf{ERM} & \textbf{Breast} & \textbf{Breast+Platt} & \textbf{Mamm.} & \textbf{Mamm.+Platt} \\
 & & & \textbf{cons.} & ($n{=}228$) & ($n{=}228$) & ($n{=}4{,}474$) & ($n{=}4{,}474$) \\
\hline
\multicolumn{8}{l}{\emph{Sanity checks (symmetric / baseline metrics)}}\\
Accuracy & $(1,0,0,1)$ & $0.500$ & yes & $[0.264,\, 0.423]$ & $[0.526,\, 0.531]$ & $[0.519,\, 0.534]$ & $[0.488,\, 0.504]$ \\
Error rate & $(0,1,1,0)$ & $0.500$ & yes & $[0.264,\, 0.423]$ & $[0.526,\, 0.531]$ & $[0.519,\, 0.534]$ & $[0.488,\, 0.504]$ \\
Recall & $(1,0,0,0)$ & $0.000$ & no & $[0.000,\, 0.267]$ & $[0.000,\, 0.428]$ & $[0.000,\, 0.000]$ & $[0.000,\, 0.000]$ \\
Specificity & $(0,0,0,1)$ & $1.000$ & no & $[0.892,\, 1.000]$ & $[0.810,\, 1.000]$ & $[0.763,\, 1.000]$ & $[0.707,\, 1.000]$ \\
\hline
\multicolumn{8}{l}{\emph{Asymmetric cost and utility objectives}}\\
Weighted Acc. ($w=0.75$) & $(3,0,0,1)$ & $0.250$ & no & $[0.264,\, 0.267]$ & $[0.398,\, 0.428]$ & $[0.147,\, 0.172]$ & $[0.164,\, 0.164]$ \\
Clinical Net Benefit ($p_t = 0.333$) & $(2,1,0,0)$ & $0.333$ & no & $[0.264,\, 0.267]$ & $[0.398,\, 0.428]$ & $[0.241,\, 0.252]$ & $[0.228,\, 0.238]$ \\
Cost-sensitive ($c_{\mathrm{fn}}{:}c_{\mathrm{fp}} = 5{:}1$) & $(0,1,5,0)$ & $0.167$ & no & $[0.264,\, 0.267]$ & $[0.398,\, 0.428]$ & $[0.121,\, 0.121]$ & $[0.102,\, 0.104]$ \\
Cost-sensitive ($c_{\mathrm{fn}}{:}c_{\mathrm{fp}} = 10{:}1$) & $(0,1,10,0)$ & $0.091$ & no & $[0.264,\, 0.267]$ & $[0.398,\, 0.428]$ & $[0.076,\, 0.076]$ & $[0.076,\, 0.076]$ \\
General Linear Cost-Utility & $(3,1,2,0)$ & $0.167$ & no & $[0.264,\, 0.267]$ & $[0.398,\, 0.428]$ & $[0.121,\, 0.121]$ & $[0.102,\, 0.104]$ \\
\hline
\end{tabular}
\caption{Threshold optimality on medical screening benchmarks across sample sizes and calibration schemes. Under conditional expectations $\mathbb{E}_\eta$ (Proposition~\ref{prop:consistency}), the continuous argmax coincides identically with $t^{\star}$ in every case. The four rightmost columns report empirical finite-sample plateaus against realised test labels on breast cancer ($n_{\mathrm{test}}=228$) and mammography ($n_{\mathrm{test}}=4{,}474$).}
\label{tab:exp_threshold}
\end{table}

The final four columns of Table~\ref{tab:exp_threshold} illustrate why analytical threshold derivation via Proposition~\ref{prop:consistency} is preferable to empirical grid-search on validation cohorts. Against realised labels $y_i \in \{0,1\}$, the empirical objective is piecewise constant, and its sample maximiser is subject to both finite-sample binomial fluctuations and local probability miscalibration (particularly in the rare-event tails of mammography, prevalence $0.023$). Consequently, empirical plateaus may not cover $t^{\star}$ exactly on uncalibrated realised data: on small cohorts like breast cancer ($n_{\mathrm{test}} = 228$), the plateau is wide ($[0.264, 0.423]$ for accuracy), while on larger cohorts like mammography ($n_{\mathrm{test}} = 4{,}474$), the plateau contracts sharply (width $\leq 0.015$). Applying post-hoc probability calibration such as Platt scaling~\cite{platt1999probabilistic} aligns the empirical operating plateau with the theoretical optimum (for instance, the Platt-calibrated accuracy plateau $[0.488, 0.504]$ covers $t^{\star} = 0.500$ exactly). In contrast, when evaluated in expectation under conditional probabilities $\eta$ (the population quantity characterised by the proposition), the theoretical threshold $t^{\star}$ is recovered with zero error across all metrics. Proposition~\ref{prop:consistency} thus provides the principled target rule, demonstrating that probability calibration and analytical thresholding form complementary components of robust decision-theoretic model deployment.

\subsection{Real Multiclass Data: Collapse, Redundancy, and Residual Skew}
\label{sec:exp_multiclass}

Two multiclass theoretical predictions are checked against empirical behavior on \emph{iris}, \emph{wine}, and \emph{digits} under 10-fold cross-validation with logistic regression.

First, Theorem~\ref{thm:micro_collapse} predicts that micro-precision, micro-recall, and micro-$F_1$ collapse identically to standard accuracy in multiclass problems. Across all three datasets, the measured spread between these four quantities is exactly zero ($0.960$ on iris, $0.983$ on wine, $0.933$ on digits), matching the exact algebraic identity. Similarly, Proposition~\ref{prop:hamming_accuracy} holds to machine precision ($\mathrm{HammingLoss} + \mathrm{LabelAccuracy} = 1.0000$ identically; on digits, $\mathrm{HammingLoss} = 0.0134$ and $\mathrm{LabelAccuracy} = 0.9866$). Micro-Jaccard evaluates to $0.923$, $0.967$, and $0.874$ respectively, strictly following $C/(2n-C)$ and generating a rank correlation $\tau = 1.000$ with accuracy across models. This empirically illustrates Proposition~\ref{prop:dof}: reporting these metrics simultaneously adds zero information over reporting accuracy alone.

Second, Proposition~\ref{prop:ovr_skew} predicts that per-class recall is prior-free while one-vs-rest specificity depends on the class prior distribution. To test this on empirical classifier behavior without confounding prior shift with finite-sample resampling variance, we first estimate the conditional error rate matrix $R_{kj} = P(\hat{Y} = j \mid Y = k)$ from 10-fold cross-validated predictions on \emph{digits} ($m=10$). We then hold $R$ fixed while sweeping the class prior vector $\boldsymbol{\pi} = (\pi_0, \dots, \pi_9)$ analytically: the prior of class $0$ is varied from balanced ($\pi_0 = 0.10$) to heavily skewed ($\pi_0 = 0.85$), distributing the remaining probability mass uniformly across the other nine classes ($\pi_k = (1-\pi_0)/9$ for $k \neq 0$). For each prior $\boldsymbol{\pi}$, the one-vs-rest confusion cells are computed directly as $\mathrm{TP}_j = n\pi_j R_{jj}$, $\mathrm{FP}_j = n\sum_{k \neq j}\pi_k R_{kj}$, $\mathrm{FN}_j = n\pi_j(1-R_{jj})$, and $\mathrm{TN}_j = n\sum_{k \neq j}\pi_k(1-R_{kj})$. Table~\ref{tab:exp_ovr} reports the resulting drift (the max--min range across the prior sweep) for two classifiers of contrasting quality: a strong logistic regression ($\text{accuracy} = 0.933$) and a shallow decision tree ($\text{accuracy} = 0.459$).

\begin{table}[!htb]
\centering
\small
\begin{tabular}{l|cc|ccc}
\hline
\textbf{Classifier} & \textbf{acc.} & \textbf{u.e.\ dev.} & \textbf{macro-recall} & \textbf{macro-spec.} & \textbf{macro-BA} \\
\hline
logistic regression & $0.933$ & $0.0086$ & $0.0000$ & $0.0053$ & $0.0026$ \\
decision tree, depth 3 & $0.459$ & $0.0505$ & $<10^{-15}$ & $0.0439$ & $0.0219$ \\
\hline
\end{tabular}
\caption{Drift of macro-averaged measures on \emph{digits} as the prior on class $0$ moves from $0.1$ to $0.85$ with the conditional rates held fixed. ``u.e.\ dev.'' is the mean over $j$ of the standard deviation of $R_{kj}$ over $k \neq j$, which is zero exactly when the uniform-error condition of Proposition~\ref{prop:ovr_skew}(2) holds.}
\label{tab:exp_ovr}
\end{table}

Macro-recall remains strictly invariant to machine precision across the entire prior sweep for both classifiers (drift $< 10^{-15}$), confirming Proposition~\ref{prop:ovr_skew}(1). In contrast, macro-specificity and macro-balanced-accuracy undergo non-trivial drift, whose magnitude directly tracks the \emph{uniform-error deviation} (the dispersion of off-diagonal conditional rates $R_{kj}$ over $k \neq j$, which Proposition~\ref{prop:ovr_skew}(2) identifies as the exact obstruction to prior-invariance). For the strong logistic regression (uniform-error deviation of $0.0086$), macro-specificity drifts by only $0.0053$ and macro-BA by $0.0026$. For the shallow decision tree, whose structured confusions increase the uniform-error deviation nearly sixfold to $0.0505$, the drift scales up more than eightfold: macro-specificity drifts by $0.0439$ and macro-BA by $0.0219$. Because macro-recall is strictly prior-free, the drift in macro-balanced-accuracy is always exactly half that of macro-specificity ($\Delta_{\mathrm{BA}} = \frac{1}{2}\Delta_{\mathrm{spec}}$). This demonstrates that residual multiclass skew sensitivity in balanced accuracy is not a fixed defect of the metric, but scales with the heterogeneity of the confusion structure, making it smallest on well-performing classifiers and most pronounced on weak models with asymmetric confusion patterns.

\subsection{Attainable Floors under One-Hot Encoding}
\label{sec:exp_floors}

Proposition~\ref{prop:lower_boundary} and Table~\ref{tab:floors} establish that under standard one-hot multiclass encoding, any evaluation measure whose value at $\mathrm{TP} = 0$ retains functional dependence on $\mathrm{TN}$ is bounded away from its nominal binary minimum. To evaluate this analytical scaling behavior as a function of the number of classes, we instantiate the worst-case predictor on \emph{digits} across subproblems of $m \in \{3, \ldots, 10\}$ classes by generating cyclically shifted predictions ($\hat{y}_i = (y_i + 1) \bmod m$), ensuring complete misclassification ($\Upsigma_{\mathrm{1}}(\mathbold{TP}) = 0$) under valid one-hot categorical constraints.

The measured values agree with the closed forms of Table~\ref{tab:floors} to machine precision across all measures and all $m$. Figure~\ref{fig:floors} plots them. The practical consequence is stark: on ten classes, a classifier that is wrong on every single example attains a label accuracy of $0.800$, a specificity of $0.889$, a balanced accuracy of $0.444$, and MCC and Cohen's $\kappa$ of $-0.111$ each. The last two are the most misleading, since both are constructed to report $-1$ under complete disagreement and $0$ at chance: on a ten-class problem their floor of $-1/(m-1)$ places total failure within $0.111$ of chance on their own scale. Standard accuracy $\Upsigma_{\mathrm{1}}(\mathbold{TP})/n$, recall, precision, $F_1$, Jaccard and the G-mean all correctly reach $0$, reflecting that the governing criterion is dependence on $tn$ at $tp=0$ and not merely the presence of $tn$ in the formula.

\begin{figure}[!htb]
  \centering
  \includegraphics[width=.72\linewidth]{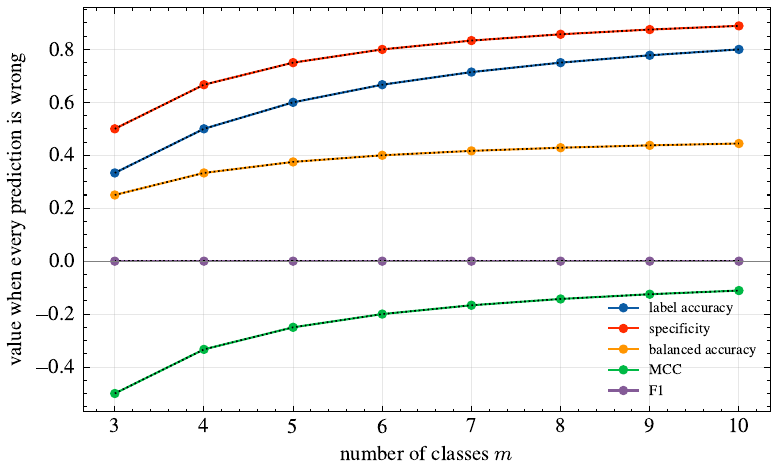}
  \caption{Value attained on \emph{digits} when every prediction is wrong, as a function of the number of classes $m$. Solid lines with markers are measured; dotted lines are the closed forms of Table~\ref{tab:floors}. Measures depending on $tn$ at $tp=0$ drift away from their nominal minimum as $m$ grows; $F_1$ stays at $0$.}
  \label{fig:floors}
\end{figure}

\section{Conclusion}
\label{sec:conclusion}

We have presented a unified algebraic framework for classification performance evaluation, grounded in binary indicator matrices and three fundamental aggregation operators ($\Upsigma_1, \Upsigma_{\mathrm{m}}, \Upsigma_{\mathrm{n}}$). The framework's primary practical contribution is its generative capacity: any binary performance metric expressible in terms of confusion cells $(tp, fp, fn, tn)$ automatically yields micro, macro/weighted, and exemplar multiclass and multilabel counterparts with no ad-hoc derivation required. Table~\ref{tab:extensions} instantiates this for eleven canonical metrics, establishing well-defined operational extensions where standard named measures were previously absent.

Beyond this generative property, the matrix formulation provides a common language for several learning settings that are typically treated in isolation (graded systematically in Table~\ref{tab:settings}). Continuous outputs map to categorical decisions via class-wise argmax (multiclass) or thresholding (multilabel). Soft and probabilistic ground truth is accommodated via triangular norms, with the product $t$-norm identified as the unique marginally consistent operator by Theorem~\ref{thm:partition}. Ordinal classification connects to the framework via cumulative binary encodings, with ordinal AUC emerging as column-wise aggregation of binary AUCs (Proposition~\ref{prop:ordinal_auc}). Cost-sensitive evaluation enters through a cost matrix acting as a bilinear form on the indicator matrices, with MAE and MSE recovered as exact algebraic special cases (Theorem~\ref{thm:ordinal_cost}).

The theoretical development yields several fundamental insights into evaluation behavior across paradigms. It formalizes why micro and macro averages diverge, identifying the gap as the covariance between class prevalence and performance for decomposable metrics, and as a combination of class skew and geometric curvature for non-decomposable ones. It reveals that binary skew-invariance transfers selectively across learning settings: while per-class recall remains universally prior-free, multiclass one-vs-rest specificity is inherently sensitive to class imbalance. It connects metric selection directly to decision theory, providing closed-form Bayes-optimal operating thresholds on conditional probabilities. Finally, it demonstrates that property preservation under aggregation is asymmetric: while monotonicity and upper boundary correctness are universally preserved, lower boundary correctness fails under one-hot multiclass encoding for metrics retaining functional dependence on true negatives.

The empirical investigations substantiate these theoretical predictions across real-world benchmarks. Experiments confirm the exact multiclass collapse of micro-averaged metrics and validate the analytical Bayes threshold rules, illustrating how post-hoc probability calibration aligns empirical operating plateaus with theoretical optima. Furthermore, evaluations on imbalanced multilabel domains demonstrate that the choice of aggregation operator and undefined-value conventions can induce substantial rank reversals and score shifts across competing models. These findings highlight that reproducible benchmarking demands the explicit, a priori declaration of aggregation levels and edge-case conventions aligned with operational deployment goals.

Several research directions remain open. First, extending the fold-wise aggregation operator (Section~\ref{sec:cv}) to characterize estimate variance and construct formal confidence intervals across cross-validation splits. Second, investigating statistical and information-theoretic redundancy among non-algebraically coupled metrics across model families. Third, exploring generalized $t$-norm formulations for task-specific label uncertainty and weaker partition constraints. Fourth, generalizing indicator-matrix representations to structured prediction problems (sequences, trees, and graphs). Finally, establishing formal bridges between the matrix aggregation framework and expected-loss decision frameworks~\cite{DBLP:journals/jmlr/Hernandez-OralloFF12} will further solidify the foundations of machine learning evaluation.

\section*{Funding}

This research was partially supported by the Conselho Nacional de
Desenvolvimento Cient\'{\i}fico e Tecnol\'ogico (CNPq) and the
Funda\c{c}\~ao de Amparo \`a Pesquisa do Estado de S\~ao Paulo (FAPESP), Brazil.

\section*{CRediT Author Statement}

\textbf{Ronaldo C. Prati}: Conceptualization, Methodology, Formal Analysis,
Software, Validation, Writing -- Original Draft, Writing -- Review \& Editing.

\section*{Declaration of Generative AI and AI-assisted Technologies}

During the preparation of this work the author used Claude (Anthropic)
in order to assist with literature search, drafting and revising manuscript
text, implementing experiment code, and checking \LaTeX\ syntax.
After using this tool, the author reviewed and edited all content as needed
and takes full responsibility for the content of the published article.

\clearpage
\appendix
\begin{landscape}
\setlength{\tabcolsep}{3.5pt}
\section*{Catalogue of Performance Measures Generalisable by the Framework}
\label{sec:catalogue}

This catalogue documents binary classification performance measures that can be extended to multiclass and multilabel settings via the indicator-matrix framework. For each measure $M(tp, fp, fn, tn)$, three extensions are obtained by substituting the aggregation operators $\SI$ (micro), $\Sm$ (per-class/macro), and $\Sn$ (per-example/exemplar). Extensions marked with \novel\ have no established named version in the literature. Measures marked \sinv\ in the ``Inv.'' column are skew-invariant (Theorem~\ref{thm:skew}).

\bigskip\noindent
\textbf{Notation.} Throughout, $tp, fp, fn, tn$ are scalar counts; $n = tp+fp+fn+tn$; $tpos = tp+fn$; $tneg = tn+fp$; $ppos = tp+fp$. Micro applies $M$ to global sums; macro applies $M$ per class and averages; exemplar applies $M$ per example and averages.

\section*{Group 1: Accuracy and Error Measures}

{\small
\begin{longtable}{p{3.0cm}|p{5.5cm}|p{3.8cm}|p{3.8cm}|p{3.8cm}|c}
\toprule
\textbf{Measure} & \textbf{Binary formula} &
\textbf{Micro ($\SI$)} & \textbf{Macro ($\Sm$)} &
\textbf{Exemplar ($\Sn$)} & \textbf{Inv.}\\
\midrule\endhead

Accuracy &
$\dfrac{tp+tn}{n}$ &
Global accuracy; = label accuracy in multilabel &
Per-class accuracy\novel &
Per-example accuracy\novel &
\\[8pt]

Error Rate &
$\dfrac{fp+fn}{n}$ &
\textbf{Hamming loss} $= \frac{1}{nm}\SI(\mathbf{Y}{\oplus}\hat{\mathbf{Y}})$ &
Per-class error rate\novel &
Per-example error rate\novel &
\\[8pt]

Balanced Error Rate &
$1 - \frac{1}{2}\!\left(\frac{tp}{tpos}+\frac{tn}{tneg}\right)$ &
Micro BER\novel &
Macro BER\novel &
Exemplar BER\novel &
\sinv\\[8pt]

Subset Accuracy &
$\mathbf{1}[tp{=}tpos,\,tn{=}tneg]$ &
Standard multiclass accuracy &
-- (per-class version reduces to accuracy) &
Proportion of fully correct examples &
\\[4pt]

\bottomrule
\end{longtable}
}

\section*{Group 2: Positive-Class Focused Measures}

{\small
\begin{longtable}{p{3.0cm}|p{5.5cm}|p{3.8cm}|p{3.8cm}|p{3.8cm}|c}
\toprule
\textbf{Measure} & \textbf{Binary formula} &
\textbf{Micro} & \textbf{Macro} &
\textbf{Exemplar} & \textbf{Inv.}\\
\midrule\endhead

Precision (PPV) &
$\dfrac{tp}{tp+fp}$ &
Micro-precision &
Macro-precision &
Exemplar-precision &
\\[8pt]

Recall (Sensitivity, TPR) &
$\dfrac{tp}{tp+fn}$ &
Micro-recall &
Macro-recall &
Exemplar-recall &
\sinv\\[8pt]

False Negative Rate (FNR, Miss Rate) &
$\dfrac{fn}{tp+fn} = 1-\text{Recall}$ &
Micro-FNR\novel &
Macro-FNR\novel &
Exemplar-FNR\novel &
\sinv\\[8pt]

False Discovery Rate (FDR) &
$\dfrac{fp}{tp+fp} = 1-\text{Precision}$ &
Micro-FDR\novel &
Macro-FDR\novel &
Exemplar-FDR\novel &
\\[8pt]

$F_\beta$ measure~\cite{rijsbergen1979,chinchor1992} &
$\dfrac{(1+\beta^2)tp}{(1+\beta^2)tp+\beta^2 fn+fp}$ &
Micro-$F_\beta$ &
Macro-$F_\beta$ &
Exemplar-$F_\beta$ &
\\[8pt]

Jaccard Index~\cite{jaccard1901} (IoU, CSI) &
$\dfrac{tp}{tp+fp+fn}$ &
Micro-Jaccard &
Macro-Jaccard (mean IoU) &
Exemplar-Jaccard &
\\[8pt]

Fowlkes-Mallows Index~\cite{fowlkes1983} &
$\sqrt{\dfrac{tp}{tp+fp}\cdot\dfrac{tp}{tp+fn}}$ &
Micro-FM\novel &
Macro-FM\novel &
Exemplar-FM\novel &
\\[8pt]

Lift &
$\dfrac{tp/(tp+fp)}{(tp+fn)/n}$ &
Micro-Lift\novel &
Macro-Lift\novel &
Exemplar-Lift\novel &
\\[8pt]

Average Precision (AP) &
$\sum_k P_k \Delta R_k$ (threshold-averaged) &
Micro-AP &
Macro-AP &
Exemplar-AP\novel &
\\[4pt]

\bottomrule
\end{longtable}
}

\section*{Group 3: Negative-Class Focused Measures}

{\small
\begin{longtable}{p{3.0cm}|p{5.5cm}|p{3.8cm}|p{3.8cm}|p{3.8cm}|c}
\toprule
\textbf{Measure} & \textbf{Binary formula} &
\textbf{Micro} & \textbf{Macro} &
\textbf{Exemplar} & \textbf{Inv.}\\
\midrule\endhead

Specificity (TNR) &
$\dfrac{tn}{tn+fp}$ &
Micro-specificity &
Macro-specificity &
Exemplar-specificity &
\sinv\\[8pt]

Negative Predictive Value (NPV) &
$\dfrac{tn}{tn+fn}$ &
Micro-NPV\novel &
Macro-NPV\novel &
Exemplar-NPV\novel &
\\[8pt]

False Positive Rate (FPR, Fall-out) &
$\dfrac{fp}{tn+fp} = 1-\text{Specificity}$ &
Micro-FPR\novel &
Macro-FPR\novel &
Exemplar-FPR\novel &
\sinv\\[8pt]

False Omission Rate (FOR) &
$\dfrac{fn}{tn+fn} = 1-\text{NPV}$ &
Micro-FOR\novel &
Macro-FOR\novel &
Exemplar-FOR\novel &
\\[4pt]

\bottomrule
\end{longtable}
}

\section*{Group 4: Measures Balancing Both Classes}

{\small
\begin{longtable}{p{3.0cm}|p{5.5cm}|p{3.8cm}|p{3.8cm}|p{3.8cm}|c}
\toprule
\textbf{Measure} & \textbf{Binary formula} &
\textbf{Micro} & \textbf{Macro} &
\textbf{Exemplar} & \textbf{Inv.}\\
\midrule\endhead

Balanced Accuracy~\cite{brodersen2010} (BA) &
$\dfrac{1}{2}\!\left(\dfrac{tp}{tpos}+\dfrac{tn}{tneg}\right)$ &
Micro-BA &
Macro-BA &
Exemplar-BA\novel &
\sinv\\[8pt]

G-mean~\cite{kubat1997} &
$\sqrt{\dfrac{tp}{tpos}\cdot\dfrac{tn}{tneg}}$ &
Micro-GM &
Macro-GM &
Exemplar-GM\novel &
\sinv\\[8pt]

Informedness~\cite{youden1950,powers2011evaluation} (Youden's J) &
$\dfrac{tp}{tpos}+\dfrac{tn}{tneg}-1$ &
Micro-Informedness\novel &
Macro-Informedness\novel &
Exemplar-Informedness\novel &
\sinv\\[8pt]

Markedness~\cite{powers2011evaluation} &
$\dfrac{tp}{tp+fp}+\dfrac{tn}{tn+fn}-1$ &
Micro-Markedness\novel &
Macro-Markedness\novel &
Exemplar-Markedness\novel &
\\[8pt]

$F_\beta^{neg}$ (F-measure on negative class) &
$\dfrac{(1+\beta^2)tn}{(1+\beta^2)tn+\beta^2 fp+fn}$ &
Micro-$F_\beta^{neg}$\novel &
Macro-$F_\beta^{neg}$\novel &
Exemplar-$F_\beta^{neg}$\novel &
\\[8pt]

Discriminant Power~\cite{biggerstaff2000} &
$\dfrac{\sqrt{3}}{\pi}\!\left[\log\!\dfrac{tp/(tpos)}{1-tp/tpos}+\log\!\dfrac{tn/tneg}{1-tn/tneg}\right]$ &
Micro-DP\novel &
Macro-DP\novel &
Exemplar-DP\novel &
\sinv\\[10pt]

Optimised Precision &
$\text{Accuracy} - \dfrac{|tp/tpos - tn/tneg|}{tp/tpos + tn/tneg}$ &
Micro-OP\novel &
Macro-OP\novel &
Exemplar-OP\novel &
\\[4pt]

\bottomrule
\end{longtable}
}

\section*{Group 5: Correlation and Agreement Measures}

{\small
\begin{longtable}{p{3.0cm}|p{5.5cm}|p{3.8cm}|p{3.8cm}|p{3.8cm}|c}
\toprule
\textbf{Measure} & \textbf{Binary formula} &
\textbf{Micro} & \textbf{Macro} &
\textbf{Exemplar} & \textbf{Inv.}\\
\midrule\endhead

Matthews CC~\cite{matthews1975comparison} (MCC) &
$\dfrac{tp{\cdot}tn - fp{\cdot}fn}{\sqrt{(tp+fp)(tp+fn)(tn+fp)(tn+fn)}}$ &
Micro-MCC &
Macro-MCC\novel &
Exemplar-MCC\novel &
\\[8pt]

Cohen's $\kappa$~\cite{cohen1960coefficient} &
$\dfrac{p_o - p_e}{1-p_e}$, $p_e=\dfrac{tpos{\cdot}ppos+tneg{\cdot}pneg}{n^2}$ &
Micro-$\kappa$ &
Macro-$\kappa$\novel &
Exemplar-$\kappa$\novel &
\\[8pt]

Phi Coefficient &
$= $ MCC (binary case) &
(same as Micro-MCC) &
(same as Macro-MCC)\novel &
(same as Exemplar-MCC)\novel &
\\[8pt]

Gwet's AC1~\cite{gwet2008} &
$\dfrac{p_o - p_e^{AC1}}{1-p_e^{AC1}}$, $p_e^{AC1}=\dfrac{1}{2}\!\left(\dfrac{tpos}{n}+\dfrac{ppos}{n}\right)\!\!\left(\dfrac{tneg}{n}+\dfrac{pneg}{n}\right)$ &
Micro-AC1\novel &
Macro-AC1\novel &
Exemplar-AC1\novel &
\\[8pt]

Tetrachoric Correlation~\cite{pearson1900} &
$\cos\!\left(\dfrac{\pi\sqrt{fp{\cdot}fn}}{\sqrt{fp{\cdot}fn}+\sqrt{tp{\cdot}tn}}\right)$ &
Micro-TC\novel &
Macro-TC\novel &
Exemplar-TC\novel &
\\[4pt]

\bottomrule
\end{longtable}
}

\section*{Group 6: Likelihood Ratios and Odds}

{\small
\begin{longtable}{p{3.0cm}|p{5.5cm}|p{3.8cm}|p{3.8cm}|p{3.8cm}|c}
\toprule
\textbf{Measure} & \textbf{Binary formula} &
\textbf{Micro} & \textbf{Macro} &
\textbf{Exemplar} & \textbf{Inv.}\\
\midrule\endhead

Positive Likelihood Ratio (LR+) &
$\dfrac{tp/(tp+fn)}{fp/(tn+fp)}$ &
Micro-LR+\novel &
Macro-LR+\novel &
Exemplar-LR+\novel &
\sinv\\[8pt]

Negative Likelihood Ratio (LR--) &
$\dfrac{fn/(tp+fn)}{tn/(tn+fp)}$ &
Micro-LR--\novel &
Macro-LR--\novel &
Exemplar-LR--\novel &
\sinv\\[8pt]

Diagnostic Odds Ratio~\cite{glas2003} (DOR) &
$\dfrac{tp{\cdot}tn}{fp{\cdot}fn}$ &
Micro-DOR\novel &
Macro-DOR\novel &
Exemplar-DOR\novel &
\sinv\\[8pt]

Odds Ratio &
$\dfrac{tp/fp}{fn/tn}$ &
Micro-OR\novel &
Macro-OR\novel &
Exemplar-OR\novel &
\sinv\\[4pt]

\bottomrule
\end{longtable}
}

\noindent\textbf{Note.} LR+, LR--, DOR, and Odds Ratio are all skew-invariant
(\sinv) because they are functions of recall and specificity only, or of ratios
that cancel under the $(\lambda, \mu)$ scaling (see Theorem~4 of the main paper).

\section*{Group 7: Measures for Ranked/Score Output}

These measures require a threshold $\theta$ and should be computed at each threshold value. They extend via the parametric formulation of Proposition~11 (ROC and AUC) of the main paper.

{\small
\begin{longtable}{p{3.0cm}|p{5.5cm}|p{3.8cm}|p{3.8cm}|p{3.8cm}|c}
\toprule
\textbf{Measure} & \textbf{Binary formula} &
\textbf{Micro} & \textbf{Macro} &
\textbf{Exemplar} & \textbf{Inv.}\\
\midrule\endhead

AUC (AUROC)~\cite{fawcett2006introduction} &
$\int_0^1 \text{TPR}(\theta)\,d\,\text{FPR}(\theta)$ &
Micro-AUC &
Macro-AUC &
Exemplar-AUC\novel &
\sinv\\[8pt]

Average Precision (AP) &
$\sum_k \text{Prec}(\theta_k)\cdot\Delta\text{Rec}(\theta_k)$ &
Micro-AP &
Macro-AP &
Exemplar-AP\novel &
\\[8pt]

R-Precision &
$\text{Precision at } k = tpos$ &
Micro-RP\novel &
Macro-RP\novel &
Exemplar-RP\novel &
\\[8pt]

AUPRC~\cite{davis2006} (Area under PR curve) &
$\int_0^1 \text{Prec}(r)\,dr$ &
Micro-AUPRC &
Macro-AUPRC &
Exemplar-AUPRC\novel &
\\[4pt]

\bottomrule
\end{longtable}
}

\section*{Group 8: Cost-Sensitive and Ordinal Measures}

These measures arise from the cost-sensitive framework as special cases of
$\mathrm{Cost}(\mathbf{Y}, \hat{\mathbf{Y}}, \mathbf{C})$ with specific cost matrices.

{\small
\begin{longtable}{p{3.0cm}|p{5.5cm}|p{3.8cm}|p{3.8cm}|p{3.8cm}|c}
\toprule
\textbf{Measure} & \textbf{Cost matrix $C_{j,k}$} &
\textbf{Micro} & \textbf{Macro} &
\textbf{Exemplar} & \textbf{Inv.}\\
\midrule\endhead

0/1 Loss (Error Rate) &
$\mathbf{1}(j\neq k)$ &
Hamming loss &
Per-class error rate\novel &
Per-example error rate\novel &
\\[8pt]

Mean Absolute Error (MAE) &
$|j-k|$ &
MAE &
Macro-MAE\novel &
Exemplar-MAE\novel &
\\[8pt]

Mean Squared Error (MSE) &
$(j-k)^2$ &
MSE &
Macro-MSE\novel &
Exemplar-MSE\novel &
\\[8pt]

Hierarchical Loss (tree) &
$d_{\text{tree}}(c_j, c_k)$ &
Hierarchical loss &
Macro hierarchical loss\novel &
Exemplar hierarchical loss\novel &
\\[8pt]

Weighted $F_\beta$ &
Custom $C_{j,k}$ &
Weighted micro-$F_\beta$\novel &
Weighted macro-$F_\beta$\novel &
Weighted exemplar-$F_\beta$\novel &
\\[4pt]

\bottomrule
\end{longtable}
}

\section*{Group 9: Soft-Label Extensions via $t$-Norms}

When $\mathbf{Y}, \hat{\mathbf{Y}} \in [0,1]^{n\times m}$ (soft labels), all
measures above generalise to three families depending on the choice of $t$-norm
$T \in \{T_L, T_P, T_M\}$ (Łukasiewicz, Product, Gödel). Each combination of
measure, aggregation operator, and $t$-norm yields a distinct evaluation
protocol. The Product $t$-norm ($T_P$) is recommended for probabilistic ground
truth (Theorem~2 of the main paper). The table below illustrates the cross-product
for a selection of measures.

{\small
\begin{longtable}{p{3.5cm}|c|p{4.8cm}|p{4.8cm}|p{4.8cm}}
\toprule
\textbf{Measure} & \textbf{$t$-norm} &
\textbf{Micro} & \textbf{Macro} & \textbf{Exemplar}\\
\midrule\endhead

\multirow{3}{*}{$F_1$} &
$T_L$ & Soft micro-$F_1^{T_L}$\novel & Soft macro-$F_1^{T_L}$\novel & Soft exemplar-$F_1^{T_L}$\novel\\
& $T_P$ & Soft micro-$F_1^{T_P}$\novel & Soft macro-$F_1^{T_P}$\novel & Soft exemplar-$F_1^{T_P}$\novel\\
& $T_M$ & Soft micro-$F_1^{T_M}$\novel & Soft macro-$F_1^{T_M}$\novel & Soft exemplar-$F_1^{T_M}$\novel\\[4pt]

\multirow{3}{*}{MCC} &
$T_L$ & Soft micro-MCC$^{T_L}$\novel & Soft macro-MCC$^{T_L}$\novel & Soft exemplar-MCC$^{T_L}$\novel\\
& $T_P$ & Soft micro-MCC$^{T_P}$\novel & Soft macro-MCC$^{T_P}$\novel & Soft exemplar-MCC$^{T_P}$\novel\\
& $T_M$ & Soft micro-MCC$^{T_M}$\novel & Soft macro-MCC$^{T_M}$\novel & Soft exemplar-MCC$^{T_M}$\novel\\[4pt]

\multirow{3}{*}{Balanced Accuracy} &
$T_L$ & Soft micro-BA$^{T_L}$\novel & Soft macro-BA$^{T_L}$\novel & Soft exemplar-BA$^{T_L}$\novel\\
& $T_P$ & Soft micro-BA$^{T_P}$\novel & Soft macro-BA$^{T_P}$\novel & Soft exemplar-BA$^{T_P}$\novel\\
& $T_M$ & Soft micro-BA$^{T_M}$\novel & Soft macro-BA$^{T_M}$\novel & Soft exemplar-BA$^{T_M}$\novel\\[4pt]

\bottomrule
\end{longtable}
}

\noindent In general, for any of the measures in Groups 1--8, each soft-label
extension via a $t$-norm $T$ yields 9 additional unnamed combinations (3 operators
$\times$ 3 $t$-norms). If all 35+ binary measures in this document are
considered, the framework generates over 300 distinct named extensions
automatically.

\section*{Group 10: Measures with Partial Framework Coverage}

The following measures involve operations not directly expressible as
$M(tp, fp, fn, tn)$ but have partial or approximate coverage within the framework.

{\small
\begin{longtable}{p{3.5cm}|p{5.5cm}|p{11cm}}
\toprule
\textbf{Measure} & \textbf{Formula} & \textbf{Framework status}\\
\midrule\endhead

Brier Score &
$\frac{1}{nm}\SI((\mathbf{Y}-\hat{\mathbf{P}})^{\circ 2})$ &
Requires continuous $\hat{\mathbf{P}}$; related to $\SI(\mathbf{TP})$ under $T_P$ but distinct
(calibration, not discrimination). Outside the tp/fp/fn/tn framework.\\[6pt]

Log-loss &
$-\frac{1}{nm}\SI(\mathbf{Y}{\circ}\log\hat{\mathbf{P}})$ &
Requires continuous $\hat{\mathbf{P}}$; proper scoring rule. Outside the framework.\\[6pt]

Normalised DCG (nDCG) &
$\frac{DCG}{IDCG}$, $DCG=\sum_i \frac{2^{rel_i}-1}{\log_2(i+1)}$ &
Requires ranked output; label ranking setting (see main paper Section~2 (Boundary Settings)).
Not directly expressible in tp/fp/fn/tn counts.\\[6pt]

Calibration Error (ECE) &
$\sum_b \frac{|B_b|}{n}|acc(B_b)-conf(B_b)|$ &
Requires binning of predicted probabilities. Outside the framework.\\[6pt]

Area Under Cost Curve (AUCC) &
$\int_0^1 M(\theta)\,d\theta$ &
Threshold-averaged cost; expressible if $M$ is a cost-sensitive measure
from the framework. Partially covered.\\[4pt]

Net Benefit (NB) &
$\frac{tp}{n} - \frac{fp}{n}\cdot\frac{p_t}{1-p_t}$ &
Expressible as $\SI(\mathbf{TP})/n - \SI(\mathbf{FP})/n \cdot p_t/(1-p_t)$ for a given
threshold $p_t$. \textbf{Covered} by the cost matrix with $C_{neg,pos} = p_t/(1-p_t)$.
Micro-NB and macro-NB are novel extensions.\\[4pt]

\bottomrule
\end{longtable}
}


\section*{Group 11: Rule Evaluation Measures \cite{lavrac1999rule}}

\noindent These measures originate from the unified framework of
Lavra\v{c} et al.~\cite{lavrac1999rule} for evaluating classification rules
in the context of inductive logic programming. They are expressed in terms of
$P(\text{Body})$, $P(\text{Head}|\text{Body})$, and $P(\text{Head})$, which
map directly to $tp$, $fp$, $fn$, $tn$ as shown below. All multiclass and
multilabel extensions are novel (\novel).

{\small
\begin{longtable}{p{3.2cm}|p{5.2cm}|p{3.6cm}|p{3.6cm}|p{3.6cm}|c}
\toprule
\textbf{Measure} & \textbf{Binary formula} &
\textbf{Micro} & \textbf{Macro} &
\textbf{Exemplar} & \textbf{Inv.}\\
\midrule\endhead

Coverage &
$\dfrac{tp+fp}{n}$ \newline (proportion of positive predictions) &
Global coverage\novel &
Per-class coverage\novel &
Per-example coverage\novel &
\\[10pt]

Relative Accuracy (RA) &
$\dfrac{tp}{tp+fp} - \dfrac{tp+fn}{n}$ \newline (= Precision $-$ Prevalence) &
Micro-RA\novel &
Macro-RA\novel &
Exemplar-RA\novel &
\\[10pt]

Klosgen's Measure (WRA) &
$\dfrac{tp+fp}{n}\!\cdot\!\left(\dfrac{tp}{tp+fp}-\dfrac{tp+fn}{n}\right)$ \newline (= Coverage $\times$ RA) &
Micro-WRA\novel &
Macro-WRA\novel &
Exemplar-WRA\novel &
\\[14pt]

Novelty &
$\dfrac{tp}{n} - \dfrac{(tp+fn)(tp+fp)}{n^2}$ \newline (= $P(\text{Head}\wedge\text{Body}) - P(\text{Head})P(\text{Body})$) &
Micro-Novelty\novel &
Macro-Novelty\novel &
Exemplar-Novelty\novel &
\\[14pt]

Conviction &
$\dfrac{1-(tp+fn)/n}{1-tp/(tp+fp)}$ \newline (= $\dfrac{P(\neg\text{Head})}{P(\neg\text{Head}|\text{Body})}$) &
Micro-Conviction\novel &
Macro-Conviction\novel &
Exemplar-Conviction\novel &
\\[14pt]

Chi-squared &
$\dfrac{n(tp\cdot tn - fp\cdot fn)^2}{(tp+fp)(tp+fn)(tn+fp)(tn+fn)}$ \newline ($\propto \text{MCC}^2 \cdot n$) &
Micro-$\chi^2$\novel &
Macro-$\chi^2$\novel &
Exemplar-$\chi^2$\novel &
\\[14pt]

Piatetsky-Shapiro &
$tp - \dfrac{(tp+fn)(tp+fp)}{n}$ \newline (= $n \times$ Novelty) &
Micro-PS\novel &
Macro-PS\novel &
Exemplar-PS\novel &
\\[10pt]

Sebag-Schoenauer &
$\dfrac{tp}{fp}$ \newline (precision-to-FP ratio) &
Micro-SS\novel &
Macro-SS\novel &
Exemplar-SS\novel &
\\[8pt]

\bottomrule
\end{longtable}
}

\bigskip\noindent
\textbf{Notes.} (1) Coverage equals $ppos/n$ and measures what fraction of
examples are predicted positive; its three extensions yield per-class and
per-example analogues. (2) Relative Accuracy and Klosgen's measure (WRA)
both generalise the notion of lift to a signed quantity that is zero under
independence. (3) Novelty and Piatetsky-Shapiro differ only by a factor of
$n$; both measure the deviation from statistical independence between the
prediction and the true label. (4) Conviction diverges when Precision~$= 1$
(no false positives); care is needed in degenerate cases.
(5) None of these measures are skew-invariant, since they all involve
a mixture of the two class scalings $\lambda$ and $\mu$.

\medskip\noindent
\textbf{Class-insensitive counterparts (Flach~\cite{flach2003geometry}).}
Flach~\cite{flach2003geometry} shows that each class-sensitive measure
has a class-insensitive counterpart obtained by projecting its
iso-performance lines in ROC space onto the direction perpendicular
to the class-prior line (i.e., forcing the iso-lines to have slope $-1$,
the signature of skew-invariance; Theorem~4 of the main paper).
The correspondence with our catalogue is:

\begin{center}
\begin{tabular}{l|l|c}
\toprule
\textbf{Lavra\v{c} et al.\ measure} & \textbf{Flach class-insensitive counterpart} & \textbf{Group} \\
\midrule
Relative Accuracy (RA) & Informedness (Youden's J) & 4 \\
Klosgen (WRA) & Balanced Accuracy & 4 \\
Novelty / Piatetsky-Shapiro & (related to Informedness) & 4 \\
Conviction & Positive Likelihood Ratio (LR+) & 6 \\
Chi-squared & Discriminant Power & 4 \\
\bottomrule
\end{tabular}
\end{center}

\noindent In each case, the class-insensitive counterpart is already listed in
our catalogue (Groups 4 and 6), has a \sinv\ mark in the Inv.\ column,
and inherits novel multiclass and multilabel extensions via the framework.

\section*{Group 12: Computer Vision Metrics (Detection \& Segmentation)}

\noindent In semantic segmentation, each pixel is an example and each semantic
class is a label, so the indicator-matrix framework applies directly with
$n = $ (pixels) and $m = $ (classes). Many standard CV metrics are therefore
instances of measures already in Groups~1--6. The table below makes the
correspondences explicit and notes the aggregation scheme used in practice.

{\small
\begin{longtable}{p{3.8cm}|p{4.5cm}|p{3.5cm}|p{4.5cm}|c}
\toprule
\textbf{CV Metric} & \textbf{Framework equivalent} & \textbf{Standard aggregation} & \textbf{Notes} & \textbf{Inv.}\\
\midrule\endhead

Pixel accuracy &
Accuracy &
$\SI$ (micro) &
$(\text{correct pixels})/(\text{total pixels})$ &
\\[8pt]

Mean pixel accuracy &
Recall &
$\Sm$ (macro) &
Average per-class fraction of correctly classified pixels &
\sinv\\[8pt]

Per-class IoU &
Jaccard &
$\Sm$ (column $j$) &
$tp_j/(tp_j+fp_j+fn_j)$ per class &
\\[8pt]

Mean IoU (mIoU)~\cite{long2015fcn} &
Jaccard &
$\Sm$ (macro) &
Standard segmentation benchmark metric &
\\[8pt]

Dice coefficient~\cite{dice1945} &
$F_1$ &
$\Sm$ or $\Sn$ &
$= 2tp/(2tp+fp+fn)$; dominant in medical imaging &
\\[8pt]

Mean Dice &
$F_1$ &
$\Sm$ (macro) &
Average Dice over classes &
\\[8pt]

Frequency-weighted IoU &
Jaccard &
$\Sm$ (weighted) &
Class IoU weighted by pixel prevalence &
\\[8pt]

Average Precision (AP)~\cite{everingham2010pascal} &
AUC (AUPRC) &
Per class &
Area under Precision--Recall curve at IoU threshold $\tau$ &
\\[8pt]

mAP (mean AP) &
AUC (AUPRC) &
$\Sm$ (macro) &
Mean AP over classes; standard detection benchmark &
\\[8pt]

COCO AP &
AUC (AUPRC) &
$\Sm$ (macro) &
Mean AP averaged over IoU thresholds $[0.5:0.05:0.95]$ &
\\[8pt]

Panoptic Quality (PQ) &
$F_1 \times $ matching score &
$\Sm$ (macro) &
Combines detection ($F_1$) and segmentation (IoU) quality;
partially outside the framework due to the matching step &
\\[4pt]

\bottomrule
\end{longtable}
}

\bigskip\noindent
\textbf{Notes.} (1) mIoU is macro-Jaccard: it was listed in Table~2 of the
main paper as ``Macro-Jaccard (mean IoU)'' precisely for this reason.
(2) Dice and $F_1$ are the same measure; Dice is the standard name in medical
imaging and segmentation, $F_1$ in information retrieval and classification.
(3) For object detection, a predicted box is a true positive when its IoU with
the ground-truth box exceeds a threshold $\tau$; this maps to
Eq.~(6) (thresholding) of the main paper with IoU as the score and $\tau$ as $\theta$.
(4) Panoptic Quality~\cite{kirillov2019panoptic} partially lies outside the
framework because it requires a bipartite matching step between predicted and
ground-truth segments before counts can be computed.

\section*{Summary Count}

\begin{center}
\begin{tabular}{lrrr}
\toprule
\textbf{Group} & \textbf{Binary measures} &
\textbf{Novel extensions ($\times 3$ operators)} & \textbf{Total novel} \\
\midrule
1. Accuracy/Error & 4 & $\sim$8 & $\sim$8 \\
2. Positive-class  & 9 & $\sim$21 & $\sim$21 \\
3. Negative-class  & 4 & $\sim$12 & $\sim$12 \\
4. Balance         & 7 & $\sim$18 & $\sim$18 \\
5. Correlation     & 5 & $\sim$13 & $\sim$13 \\
6. Likelihood/Odds & 4 & $\sim$12 & $\sim$12 \\
7. Score-based     & 4 & $\sim$6 & $\sim$6 \\
8. Cost-sensitive  & 5 & $\sim$10 & $\sim$10 \\
11. Rule evaluation (Lavra\v{c} et al.)  & 8 & $\sim$24 & $\sim$24 \\
9. Soft-label (per $t$-norm) & $35\times 3$ & all novel & $\sim$315 \\
\midrule
\textbf{Total} & \textbf{$\sim$50} &
\multicolumn{2}{r}{\textbf{142 extensions marked as lacking an established name}} \\
\bottomrule
\end{tabular}
\end{center}

\bigskip\noindent
\textbf{Legend.} $\novel$ = novel (no established named version in the
literature). $\sinv$ = skew-invariant (Theorem~4 of the main paper):
value is unchanged when class prevalences vary while conditional
classifier rates are held fixed.
\end{landscape}

\end{document}